\documentclass{article}
\usepackage[margin=1in]{geometry}
\usepackage{amsmath,amsthm,amssymb,mathrsfs}
\usepackage{graphicx}
\usepackage[pdfencoding=auto, psdextra]{hyperref}
\usepackage{cleveref}
\usepackage[figurename=Figure]{caption}
\usepackage{booktabs}
\usepackage{enumitem}
\usepackage{tikz,tikz-cd}
\usepackage{pgfplots}
\pgfplotsset{compat=1.18}
\usetikzlibrary{positioning, arrows.meta, decorations.pathreplacing}
\usepackage{pifont}
\usepackage{algorithm}
\usepackage{algpseudocode}
\usepackage{xcolor}
\usepackage{tabularx,array}

\theoremstyle{plain}
\newtheorem{theorem}{Theorem}
\newtheorem{lemma}{Lemma}
\newtheorem{prop}{Proposition}
\newtheorem{corollary}{Corollary}
\theoremstyle{definition}
\newtheorem{definition}{Definition}
\newtheorem{conjecture}{Conjecture}
\theoremstyle{remark}
\newtheorem{remark}{Remark}

\newcommand{\mathpdf}[2]{\texorpdfstring{$#1$}{#2}}

\title{Sprecher Networks: \\ A Parameter-Efficient Kolmogorov-Arnold Architecture}
\author{
  Christian Hägg\thanks{Department of Mathematics, Stockholm University, Stockholm, Sweden. Email: \texttt{hagg@math.su.se}} \and
  Kathlén Kohn\thanks{Department of Mathematics, KTH Royal Institute of Technology, Stockholm, Sweden. Email: \texttt{kathlen@kth.se}} \and
  Giovanni Luca Marchetti\thanks{Department of Mathematics, KTH Royal Institute of Technology, Stockholm, Sweden. Email: \texttt{glma@kth.se}} \and
  Boris Shapiro\thanks{Department of Mathematics, Stockholm University, Stockholm, Sweden. Email: \texttt{shapiro@math.su.se}}
}
\date{\today}

\begin{document}

\maketitle

\begin{abstract}
We introduce \emph{Sprecher Networks} (SNs), a family of trainable architectures derived from David Sprecher's 1965 constructive form of the Kolmogorov--Arnold representation. Each SN block implements a ``sum of shifted univariate functions'' using only two shared learnable splines per block, a monotone inner spline $\phi$ and a general outer spline $\Phi$, together with a learnable shift parameter $\eta$ and a mixing vector $\lambda$ shared across all output dimensions. Stacking these blocks yields deep, compositional models; for vector-valued outputs we append an additional non-summed output block.

We also propose an optional lateral mixing operator enabling intra-block communication between output channels with only $O(d_{\mathrm{out}})$ additional parameters. Owing to the vector (not matrix) mixing weights and spline sharing, SNs scale linearly in width, approximately $O\!\left(\sum_{\ell}(d_{\ell-1}+d_{\ell}+G)\right)$ parameters for $G$ spline knots, versus $O\!\left(\sum_{\ell} d_{\ell-1}d_{\ell}\right)$ for dense MLPs and $O\!\left(G\sum_{\ell} d_{\ell-1}d_{\ell}\right)$ for edge-spline KANs. This linear width-scaling is particularly attractive for extremely wide, shallow models, where low depth can translate into low inference latency. Finally, we describe a sequential forward implementation that avoids materializing the $d_{\mathrm{in}}\times d_{\mathrm{out}}$ shifted-input tensor, reducing peak forward-intermediate memory from quadratic to linear in layer width (treating batch size as constant), which is relevant in memory-constrained settings such as on-device/edge inference; we demonstrate deployability via fixed-point real-time digit classification on a resource-constrained embedded device with only 4\,MB RAM. We provide empirical demonstrations on supervised regression, Fashion-MNIST classification (including stable training at 25 hidden layers with residual connections and normalization), and a Poisson PINN, together with controlled comparisons to MLP and KAN baselines.
\end{abstract}

\section{Introduction and historical background}
Approximation of continuous functions by sums of univariate functions has been a recurring theme in mathematical analysis and neural networks. The Kolmogorov--Arnold Representation Theorem \cite{kolmogorov, arnold} established that any multivariate continuous function $f : [0,1]^n \to \mathbb{R}$ can be represented as a finite composition of continuous functions of a single variable and the addition operation. Specifically, Kolmogorov (1957) showed that such functions can be represented as a finite sum involving univariate functions applied to sums of other univariate functions of the inputs.

\vspace{1mm}
\textbf{David Sprecher's 1965 construction.} In his 1965 landmark paper \cite{sprecher1965}, David Sprecher provided a constructive proof and a specific formula realizing the Kolmogorov--Arnold representation. He showed that any continuous function $f:[0,1]^n \to \mathbb{R}$ could be represented as:
\begin{equation}\label{eq:sprecher_original}
f(\mathbf{x})=\sum_{q=0}^{2n}\Phi\Biggl(\,\sum_{p=1}^{n}\lambda_p\,\phi\bigl(x_p+\eta\,q\bigr)+q\Biggr)
\end{equation}
for a single \emph{monotone} inner function $\phi$, a continuous outer function $\Phi$, a constant shift parameter $\eta > 0$, and constants $\lambda_p$. This construction simplified the representation by using only one inner function $\phi$, relying on shifts of the input coordinates ($x_p + \eta q$) and an outer summation index shift ($+q$) to achieve universality. The key insight of \emph{shifting input coordinates} and summing evaluations under inner and outer univariate maps is central to Sprecher's specific result.

\vspace{1mm}
\textbf{From a shallow theorem to a deep architecture.} Sprecher's 1965 result is remarkable because it guarantees universal approximation with a single hidden layer (i.e., a shallow network). This mirrors the history of Multi-Layer Perceptrons, where the Universal Approximation Theorem also guaranteed the sufficiency of a single hidden layer. However, the entire deep learning revolution was built on the empirical discovery that composing multiple layers, while not theoretically necessary for universality, provides vast practical benefits in terms of efficiency and learnability.

This history motivates our central research question: can the components of Sprecher's shallow, highly-structured formula be used as a new kind of building block in a \emph{deep}, compositional architecture? We propose to investigate this by composing what we term \emph{Sprecher blocks}, where the vector output of one block becomes the input to the next. It is crucial to emphasize that this deep, compositional structure is our own architectural proposal, inspired by the paradigm of deep learning, and is not part of Sprecher's original construction or proof of universality. The goal of this paper is not to generalize Sprecher's theorem, but to empirically evaluate whether this theorem-inspired design is a viable and efficient alternative to existing deep learning models when extended into a deep framework.

\vspace{1mm}
\textbf{Modern context.} Understanding how novel architectures relate to established ones is crucial. Standard Multi-Layer Perceptrons (MLPs) \cite{haykin1994neural} employ fixed nonlinear activation functions at nodes and learnable linear weights on edges, justified by the Universal Approximation Theorem \cite{cybenko1989approximation, hornik1989multilayer}. Extensions include networks with \emph{learnable activations on nodes}, sometimes called Adaptive-MLPs or Learnable Activation Networks (LANs) \cite{goyal2019learning, zhang2022neural, liu2024kan}, which retain linear edge weights but make the node non-linearity trainable. Recent work has revitalized interest in leveraging Kolmogorov-Arnold representations for modern deep learning. Notably, Kolmogorov-Arnold Networks (KANs) \cite{liu2024kan} were introduced, proposing an architecture with learnable activation functions (splines) placed on the \emph{edges} of the network graph, replacing traditional linear weights and fixed node activations. This enables KANs to encode structure that MLPs must discover implicitly, leading to benefits such as better scalability, and improved performance on physics-based tasks. Due to their success, KANs have been extended in various forms, and adapted to a variety of tasks -- see \cite{somvanshi2025survey} for a survey. To name a few, alternative learned activations have been explored \cite{li2024kolmogorov, bozorgasl2405wav}, and several architectural extensions have been developed, e.g. convolutional \cite{bodner2024convolutional}, recurrent \cite{genet2024temporal}, and attention-based \cite{yang2024kolmogorov}. Sprecher Networks (SNs), as we detail below, propose a distinct approach, derived directly from Sprecher's 1965 formula. SNs employ function blocks containing shared learnable splines ($\phi, \Phi$), learnable mixing weights ($\lambda$), explicit structural shifts ($\eta, q$), and optionally, lateral mixing connections for intra-block communication. This structure offers a different alternative within the landscape of function approximation networks. Code is available at \url{https://github.com/Zelaron/Sprecher-Network}.

\subsection{Scalability comparison with related architectures}\label{sec:scalability}

Before detailing the SN architecture, we present a key empirical result that motivates our design choices. Recent works have proposed Kolmogorov-Arnold-inspired architectures claiming ``Sprecher-inspired'' efficiency, including GS-KAN \cite{eliasson2025gskan} and SaKAN \cite{sakan2025}. While both invoke Sprecher's 1965 construction as theoretical motivation (and reduce the number of distinct learned univariate functions), they retain dense $O(d_{\mathrm{in}} \times d_{\mathrm{out}})$ weight matrices: GS-KAN uses per-edge learnable weights $\lambda_{p,q}$ connecting all input-output pairs, and SaKAN includes dense residual weight matrices $u_{ij}$ for linear residuals. In contrast, SNs use only \emph{vector} weights $\lambda_i$ (shared across outputs), achieving genuinely $O(d)$ parameter scaling in layer width.

\paragraph{Memory scalability benchmark.}
To demonstrate this difference empirically, we conduct a memory scalability stress test (here \emph{MLP} denotes a stack of dense linear layers with bias): starting at width 512 with depth 3 on 64-dimensional input, we double the width until only one architecture survives without running out of memory (OOM). For each width, we run a single training step (forward, mean-squared error, backward, Adam update) and record the peak \emph{additional} memory used (relative to a cleaned baseline), taking the maximum of accelerator-allocated memory and process resident set size (RSS). To isolate connectivity scaling rather than spline resolution, we use a small fixed univariate-function parameterization (i.e., we replace the usual spline parameterizations used elsewhere in the paper for this stress test): the KAN baseline uses 5 learnable coefficients per edge, while our GS-KAN and SaKAN baselines (and the SN) use a 1-parameter piecewise-linear nonlinearity (PReLU; $x\mapsto \max(0,x) + a\,\min(0,x)$ with learnable slope $a$) for each required univariate function. For SN, we additionally use the memory-efficient evaluation strategy described in Section~\ref{sec:memory_efficient}, which iterates over $q$ (optionally in small chunks) to avoid materializing the full $(B\times d_{\mathrm{in}}\times d_{\mathrm{out}})$ shifted-input tensor. Table~\ref{tab:scalability} shows results on a machine with 8GB of system RAM (exact OOM breakpoints are hardware- and implementation-dependent).

\begin{table}[ht]
\centering
\small
\caption{Scalability benchmark: parameter count and peak \emph{additional memory} (MB) during a single Adam training step (forward+backward+update) as layer width increases. Configuration: \texttt{batch\_size=32}, \texttt{input\_dim=64}, \texttt{depth=3} hidden layers, \texttt{output\_dim=1} (with an explicit final output layer; for SN, this is a final Sprecher block with $d_{\mathrm{out}}=1$); Apple M-series GPU (Metal/MPS), float32. At width 16384, MLP and SaKAN run out of memory; KAN and GS-KAN were skipped after earlier OOMs; SN uses 49,228 parameters and 7.4~MB peak additional memory.}
\label{tab:scalability}
\begin{tabular}{r l r r l}
\toprule
Width & Model & Parameters & Peak $\Delta$Mem (MB) & Status \\
\midrule
512   & MLP       & 559,105       & 8.8     & OK \\
512   & KAN       & 2,787,840     & 42.7    & OK \\
512   & GS-KAN    & 559,109       & 138.3   & OK \\
512   & SaKAN     & 559,109       & 9.0     & OK \\
512   & SN        & 1,612         & 0.5     & OK \\
\midrule
1024  & MLP       & 2,166,785     & 33.6    & OK \\
1024  & KAN       & 10,818,560    & 184.1   & OK \\
1024  & GS-KAN    & 2,166,789     & 536.7   & OK \\
1024  & SaKAN     & 2,166,789     & 34.0    & OK \\
1024  & SN        & 3,148         & 1.1     & OK \\
\midrule
2048  & MLP       & 8,527,873     & 131.1   & OK \\
2048  & KAN       & 42,608,640    & 650.6   & OK \\
2048  & GS-KAN    & 8,527,877     & 2,113.3 & OK \\
2048  & SaKAN     & 8,527,877     & 132.1   & OK \\
2048  & SN        & 6,220         & 2.0     & OK \\
\midrule
4096  & MLP       & 33,832,961    & 518.2   & OK \\
4096  & KAN       & 169,103,360   & 2,581.6 & OK \\
4096  & GS-KAN    & 33,832,965    & ---     & OOM \\
4096  & SaKAN     & 33,832,965    & 520.1   & OK \\
4096  & SN        & 12,364        & 4.2     & OK \\
\midrule
8192  & MLP       & 134,774,785   & 2,060.4 & OK \\
8192  & KAN       & 673,751,040   & ---     & OOM \\
8192  & GS-KAN    & 134,774,789   & ---     & SKIP(after OOM) \\
8192  & SaKAN     & 134,774,789   & 2,064.3 & OK \\
8192  & SN        & 24,652        & 8.3     & OK \\
\midrule
16384 & MLP       & 537,985,025   & ---     & OOM \\
16384 & KAN       & 2,689,679,360 & ---     & SKIP(after OOM) \\
16384 & GS-KAN    & 537,985,029   & ---     & SKIP(after OOM) \\
16384 & SaKAN     & 537,985,029   & ---     & OOM \\
\textbf{16384} & \textbf{SN} & \textbf{49,228} & \textbf{7.4} & \textbf{OK} \\
\bottomrule
\end{tabular}
\end{table}

At width 4096, GS-KAN runs out of memory and is therefore not re-run at larger widths (reported as \texttt{SKIP(after OOM)}). In a straightforward fully vectorized evaluation, GS-KAN can OOM early despite MLP-like parameter counts because evaluating the shifted activations $\psi(x_p+\epsilon_q)$ for all $(p,q)$ at once materializes a tensor of shape $B \times N_{\mathrm{in}} \times N_{\mathrm{out}}$ (here $B=32$). At width 8192, standard KAN runs out of memory and is likewise skipped at width 16384. At width 16384, both MLP and SaKAN run out of memory, leaving \textbf{SN as the only architecture that completes the step} (49,228 parameters, 7.4~MB peak additional memory). For reference, the computed parameter counts at this scale are 537,985,025 (MLP), 537,985,029 (SaKAN/GS-KAN), and 2,689,679,360 (KAN).

\paragraph{Capacity utilization.}
A natural question is whether the extreme parameter efficiency comes at the cost of learning capacity. To verify that the wide SN actually \emph{uses} its capacity rather than merely allocating it, we train the width-16384 SN (using the same lightweight configuration as in Table~\ref{tab:scalability}, including the memory-efficient sequential evaluation from Section~\ref{sec:memory_efficient}) on a 64-dimensional regression task adapted from \cite{liu2024kan}: $f(\mathbf{x}) = \exp\bigl(\frac{1}{64}\sum_{i=1}^{64} \sin^2(\frac{\pi x_i}{2})\bigr)$ for $\mathbf{x} \in [0,1]^{64}$. We optimize mean-squared error on a fixed batch of 32 random samples for 400 epochs with Adam (learning rate $0.001$). Figure~\ref{fig:scalability_training} shows the best-so-far \emph{training} loss (the minimum loss observed up to each epoch), which is robust to occasional loss spikes in this high-dimensional setting. The best loss decreases from $3.3867e+00$ to $6.7672e-02$, indicating that the model remains trainable and can substantially reduce loss even at extreme parameter efficiency.

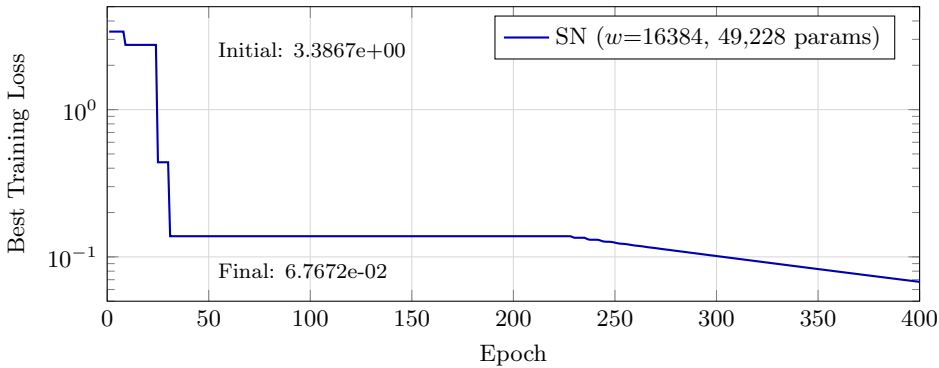
\begin{figure}[ht]
\centering
\begin{tikzpicture}
\begin{axis}[
    width=0.75\textwidth,
    height=5.5cm,
    xlabel={Epoch},
    ylabel={Best Training Loss},
    ymode=log,
    xmin=0, xmax=400,
    ymin=0.05, ymax=5,
    grid=major,
    grid style={gray!30},
    legend pos=north east,
    legend style={font=\small},
    tick label style={font=\small},
    label style={font=\small},
]
\addplot[
    color=blue!70!black,
    thick,
    mark=none,
] coordinates {
    (1, 3.3867e+00) (8, 3.3867e+00) (9, 2.7513e+00) (24, 2.7513e+00) 
    (25, 4.3903e-01) (30, 4.3903e-01) (31, 1.3813e-01) (228, 1.3813e-01) 
    (229, 1.3655e-01) (230, 1.3484e-01) (235, 1.3484e-01) (236, 1.3291e-01) 
    (237, 1.3118e-01) (238, 1.3071e-01) (241, 1.3071e-01) (242, 1.3033e-01) 
    (243, 1.2888e-01) (244, 1.2760e-01) (245, 1.2687e-01) (246, 1.2667e-01) 
    (247, 1.2665e-01) (248, 1.2643e-01) (249, 1.2582e-01) (250, 1.2492e-01) 
    (251, 1.2400e-01) (252, 1.2330e-01) (253, 1.2287e-01) (254, 1.2260e-01) 
    (255, 1.2229e-01) (256, 1.2180e-01) (257, 1.2115e-01) (258, 1.2046e-01) 
    (259, 1.1983e-01) (260, 1.1934e-01) (261, 1.1896e-01) (262, 1.1857e-01) 
    (263, 1.1812e-01) (264, 1.1758e-01) (265, 1.1700e-01) (266, 1.1645e-01) 
    (267, 1.1596e-01) (268, 1.1552e-01) (269, 1.1510e-01) (270, 1.1466e-01) 
    (271, 1.1417e-01) (272, 1.1366e-01) (273, 1.1316e-01) (274, 1.1268e-01) 
    (275, 1.1223e-01) (276, 1.1180e-01) (277, 1.1136e-01) (278, 1.1090e-01) 
    (279, 1.1042e-01) (280, 1.0995e-01) (281, 1.0950e-01) (282, 1.0905e-01) 
    (283, 1.0862e-01) (284, 1.0818e-01) (285, 1.0774e-01) (286, 1.0729e-01) 
    (287, 1.0684e-01) (288, 1.0640e-01) (289, 1.0597e-01) (290, 1.0554e-01) 
    (291, 1.0511e-01) (292, 1.0468e-01) (293, 1.0425e-01) (294, 1.0382e-01) 
    (295, 1.0339e-01) (296, 1.0297e-01) (297, 1.0255e-01) (298, 1.0214e-01) 
    (299, 1.0172e-01) (300, 1.0131e-01) (301, 1.0089e-01) (302, 1.0048e-01) 
    (303, 1.0007e-01) (304, 9.9660e-02) (305, 9.9255e-02) (306, 9.8850e-02) 
    (307, 9.8446e-02) (308, 9.8043e-02) (309, 9.7642e-02) (310, 9.7244e-02) 
    (311, 9.6849e-02) (312, 9.6454e-02) (313, 9.6061e-02) (314, 9.5670e-02) 
    (315, 9.5280e-02) (316, 9.4891e-02) (317, 9.4505e-02) (318, 9.4121e-02) 
    (319, 9.3739e-02) (320, 9.3358e-02) (321, 9.2978e-02) (322, 9.2600e-02) 
    (323, 9.2224e-02) (324, 9.1849e-02) (325, 9.1477e-02) (326, 9.1106e-02) 
    (327, 9.0736e-02) (328, 9.0368e-02) (329, 9.0002e-02) (330, 8.9637e-02) 
    (331, 8.9274e-02) (332, 8.8912e-02) (333, 8.8552e-02) (334, 8.8193e-02) 
    (335, 8.7836e-02) (336, 8.7480e-02) (337, 8.7125e-02) (338, 8.6773e-02) 
    (339, 8.6422e-02) (340, 8.6072e-02) (341, 8.5723e-02) (342, 8.5377e-02) 
    (343, 8.5031e-02) (344, 8.4688e-02) (345, 8.4346e-02) (346, 8.4005e-02) 
    (347, 8.3666e-02) (348, 8.3329e-02) (349, 8.2993e-02) (350, 8.2658e-02) 
    (351, 8.2325e-02) (352, 8.1993e-02) (353, 8.1663e-02) (354, 8.1335e-02) 
    (355, 8.1008e-02) (356, 8.0682e-02) (357, 8.0357e-02) (358, 8.0034e-02) 
    (359, 7.9713e-02) (360, 7.9393e-02) (361, 7.9074e-02) (362, 7.8757e-02) 
    (363, 7.8441e-02) (364, 7.8126e-02) (365, 7.7813e-02) (366, 7.7501e-02) 
    (367, 7.7191e-02) (368, 7.6881e-02) (369, 7.6574e-02) (370, 7.6267e-02) 
    (371, 7.5962e-02) (372, 7.5658e-02) (373, 7.5355e-02) (374, 7.5054e-02) 
    (375, 7.4753e-02) (376, 7.4455e-02) (377, 7.4157e-02) (378, 7.3861e-02) 
    (379, 7.3566e-02) (380, 7.3273e-02) (381, 7.2980e-02) (382, 7.2689e-02) 
    (383, 7.2400e-02) (384, 7.2112e-02) (385, 7.1825e-02) (386, 7.1540e-02) 
    (387, 7.1256e-02) (388, 7.0973e-02) (389, 7.0691e-02) (390, 7.0410e-02) 
    (391, 7.0131e-02) (392, 6.9853e-02) (393, 6.9576e-02) (394, 6.9301e-02) 
    (395, 6.9026e-02) (396, 6.8753e-02) (397, 6.8481e-02) (398, 6.8210e-02) 
    (399, 6.7940e-02) (400, 6.7672e-02)
};
\addlegendentry{SN ($w{=}16384$, 49,228 params)}

\node[anchor=west, font=\footnotesize] at (axis cs:50,0.08) {Final: 6.7672e-02};
\node[anchor=west, font=\footnotesize] at (axis cs:50,2.5) {Initial: 3.3867e+00};

\end{axis}
\end{tikzpicture}
\caption{Training curve for SN at width 16384 (the only model that completes the memory benchmark at this scale) on a 64-dimensional regression task. The best-so-far training loss decreases from 3.3867e+00 to 6.7672e-02 over 400 epochs, indicating that the model remains trainable and can substantially reduce loss despite its extreme parameter efficiency.}
\label{fig:scalability_training}
\end{figure}

\paragraph{Why competitors fail to scale.}
The results in Table~\ref{tab:scalability} reflect two distinct scaling bottlenecks: (i) \emph{quadratic parameter growth} from dense $N_{\mathrm{in}}\times N_{\mathrm{out}}$ matrices (and their optimizer state), and (ii) \emph{quadratic edge-activation working memory} when edgewise nonlinearities are evaluated for all $(p,q)$ at once. GS-KAN \cite{eliasson2025gskan} parameterizes edge functions as $\lambda_{p,q} \cdot \psi(x_p + \epsilon_q)$, where $\lambda_{p,q}$ is a learnable weight for \emph{each} input-output pair $(p,q)$---this retains a dense $d_{\mathrm{in}}\times d_{\mathrm{out}}$ matrix and yields $O(d_{\mathrm{in}} \times d_{\mathrm{out}})$ parameters per layer despite sharing the basis function $\psi$. In a straightforward fully vectorized evaluation, evaluating $\psi(x_p+\epsilon_q)$ for all $(p,q)$ at once materializes a $B \times N_{\mathrm{in}} \times N_{\mathrm{out}}$ activation tensor (here $B=32$), explaining the early OOM of GS-KAN at width 4096 even though its parameter count nearly matches an MLP at the same width. This activation-memory bottleneck is implementation-dependent and can be alleviated with sequential/chunked evaluation, but the dense $\lambda_{p,q}$ matrix still imposes $O(N_{\mathrm{in}} \times N_{\mathrm{out}})$ parameter (and optimizer-state) scaling. SaKAN \cite{sakan2025} similarly reduces the number of distinct learned univariate functions but retains a dense linear residual matrix $u_{ij} \in \mathbb{R}^{n_{\mathrm{out}} \times n_{\mathrm{in}}}$ in each layer (see their Eq.~8), so its width-scaling remains dominated by $O(N_{\mathrm{in}} \times N_{\mathrm{out}})$ parameters (even when techniques such as gradient-free spline evaluation reduce spline-gradient memory). Standard KANs \cite{liu2024kan} place learned functions on every edge, giving $O(G\,N_{\mathrm{in}}\times N_{\mathrm{out}})$ parameters.

In contrast, SNs follow Sprecher's original construction more faithfully: the mixing weights $\lambda_i$ are \emph{shared across all output dimensions}, collapsing the weight tensor from a full matrix to a single vector and yielding only $O(N_{\mathrm{in}})$ mixing-weight parameters per block (plus the shared spline parameters, independent of $N_{\mathrm{out}}$; optional lateral mixing adds at most linear $O(N_{\mathrm{out}})$ parameters, and cyclic/node residual connections add at most linear $O(\max\{N_{\mathrm{in}},N_{\mathrm{out}}\})$ parameters, whereas a dense linear residual projection would instead contribute $O(N_{\mathrm{in}}N_{\mathrm{out}})$ parameters). Combined with the memory-efficient sequential evaluation described in Section~\ref{sec:memory_efficient} (and, when needed, checkpointing when training), this yields linear parameter scaling and small working-memory growth in practice. This linear scaling is particularly advantageous for \emph{shallow but very wide} networks: while MLPs and KANs quickly exhaust memory as layer width increases (Table~\ref{tab:scalability}), SNs can accommodate extremely wide layers with minimal overhead. Such architectures may benefit low-latency settings where fewer sequential blocks reduce depth-induced latency and the per-block computation is highly parallelizable, though the arithmetic cost of a block still scales as $O(d_{\mathrm{in}}d_{\mathrm{out}})$. They are also natural when the target function is well-approximated by a single compositional step with many output channels. While this weight sharing imposes a structural constraint, Sprecher's theorem guarantees it suffices for universal approximation in single-layer networks, and our experiments suggest it remains effective in deep compositions over the depths explored in Section~\ref{sec:experiments}.

\medskip
\noindent Having established the scalability advantages of SNs, we now detail the architecture (Section~\ref{sec:motivation}), provide theoretical analysis (Section~\ref{sec:universality}), and present additional empirical demonstrations (Section~\ref{sec:experiments}).

\section{Motivation and overview of Sprecher Networks}\label{sec:motivation}
While MLPs are the workhorse of deep learning, architectures inspired by Kolmogorov--Arnold superposition (KAS) representations offer potential benefits, particularly in interpretability and potentially parameter efficiency for certain function classes. KANs explore one direction by placing learnable functions on edges. Our \emph{Sprecher Networks} (SNs) explore a different direction, aiming to directly implement Sprecher's constructive formula within a trainable framework and extend it to deeper architectures.

SNs are built upon the following principles, directly reflecting Sprecher's formula:
\begin{itemize}
    \item Each functional block (mapping between layers) is organized around a shared \emph{monotone} spline $\phi(\cdot)$ and a shared \emph{general} spline $\Phi(\cdot)$, both learnable.
    \item Each block incorporates a learnable scalar shift $\eta$ applied to inputs based on the output index $q$.
    \item Each block includes learnable mixing weights $\lambda_{i}$ (a vector, not a matrix), shared across all output dimensions.
    \item The structure explicitly includes the additive shift $\alpha q$ inside the outer spline $\Phi$, where $\alpha$ is a scaling factor (typically $\alpha = 1$) that maintains consistency with Sprecher's formulation.
    \item Optionally, blocks can include lateral mixing connections that allow output dimensions to exchange information before the outer spline transformation, enhancing expressivity with minimal parameter overhead.
\end{itemize}
Our architecture generalizes this classical single-layer shift-and-sum construction to a multi-layer network by composing these functional units, which we term \emph{Sprecher blocks}. The mapping from one hidden layer representation to the next is realized by such a block. Unlike MLPs with fixed node activations, LANs with learnable node activations, or KANs with learnable edge activations, SNs concentrate their learnable non-linearity into the two shared splines per block, applied in a specific structure involving shifts and learnable linear weights. This imposes a strong inductive bias, trading the flexibility of independent weights/splines for extreme parameter sharing. Diversity in the transformation arises from the mixing weights ($\lambda$), the index-dependent shifts ($q$), and when enabled, the lateral mixing connections.

Concretely, each Sprecher block computes a shared-spline shift-and-sum map parameterized by $(\phi^{(\ell)},\Phi^{(\ell)},\lambda^{(\ell)},\eta^{(\ell)})$ and, when enabled, lateral-mixing parameters $(\tau^{(\ell)},\omega^{(\ell)})$ (see Figure~\ref{fig:block_flow} and the formal operator definition below).
Lateral mixing (if enabled) modifies the pre-activation $s_q^{(\ell)}$ (defined below) \emph{before} $\Phi^{(\ell)}$, whereas residual connections (if enabled) are applied \emph{after} $\Phi^{(\ell)}$ and do not introduce lateral coupling across $q$.
For scalar outputs (summation-head variant), the outputs of the final Sprecher block are aggregated (via summation); for vector outputs, no final summation is applied.

In Sprecher's original work, one layer (block) with $d_{\mathrm{out}} = 2n+1$ outputs (where $n=d_{\mathrm{in}}$) was sufficient for universality. Our approach stacks $L$ Sprecher blocks to create a deep network progression:
$$ d_0 \to d_1 \to \cdots \to d_{L-1} \to d_L, $$
where $d_0=d_{\mathrm{in}}$ is the input dimension, and $d_L$ is the dimension of the final hidden representation before potential aggregation or final mapping. This multi-block composition provides a deeper analog of the KAS construction, aiming for potentially enhanced expressive power or efficiency for complex compositional functions. Although the overall SN family is universal already for $L=1$ (Theorem~\ref{thm:ua_single_layer}), we do not yet have a universality characterization under fixed depth/width constraints for $L>1$ blocks (nor for our vector-output extension); we explore this empirically (see Section~\ref{sec:universality}).

\begin{definition}[Network notation]
Throughout this paper, we denote Sprecher Network architectures using arrow notation of the form $d_{\mathrm{in}}\to[d_1,d_2,\ldots,d_L]\to d_{\mathrm{out}}$, where $d_{\mathrm{in}}$ is the input dimension, $[d_1,d_2,\ldots,d_L]$ represents the hidden layer dimensions (widths), and $d_{\mathrm{out}}$ is the final output dimension of the network. For scalar output ($d_{\mathrm{out}}=1$), we typically aggregate the final block's outputs by summation; alternatively, scalar output can be obtained via the vector-output formulation by including an additional non-summed output block mapping from $d_L$ to $1$ (we state explicitly when this variant is used). For vector output ($d_{\mathrm{out}}>1$), an additional non-summed block maps from $d_L$ to $d_{\mathrm{out}}$. For example, $2\to[5,3,8]\to1$ describes a network with 2-dimensional input, three hidden layers of widths 5, 3, and 8 respectively, and a scalar output (implying the final block's outputs of dimension 8 are summed). $2\to[5,3]\to4$ describes a network with 2-dimensional input, two hidden layers of widths 5 and 3, and a 4-dimensional vector output (implying an additional output block maps from dimension 3 to 4 without summation). When input or output dimensions are clear from context, we may use the abbreviated notation $[d_1,d_2,\ldots,d_L]$ to focus on the hidden layer structure.
\end{definition}

\section{Core architectural details}
In our architecture, the fundamental building unit is the \emph{Sprecher block}. The network is composed of a sequence of Sprecher blocks, each performing a shift-and-sum transformation inspired by Sprecher's original construction.

\subsection{Sprecher block structure}
A Sprecher block transforms an input vector $\mathbf{x} \in \mathbb{R}^{d_{\mathrm{in}}}$ to an output vector $\mathbf{h} \in \mathbb{R}^{d_{\mathrm{out}}}$. The data flow through a single block is illustrated in Figure~\ref{fig:block_flow}. This transformation is implemented using the following shared, learnable components specific to that block:
\begin{itemize}
    \item \textbf{Monotone inner spline $\phi$:} a shared non-decreasing univariate spline (applied coordinate-wise to shifted inputs).
    \item \textbf{Outer spline $\Phi$:} a shared univariate spline applied after summation (no monotonicity constraint).
    \item \textbf{Mixing vector $\lambda\in\mathbb{R}^{d_{\mathrm{in}}}$:} shared across all $d_{\mathrm{out}}$ outputs (vector weights, not a matrix).
    \item \textbf{Shift $\eta\in\mathbb{R}$:} controls the structured input shift $x_i\mapsto x_i+\eta q$.
    \item \textbf{Optional lateral mixing:} $O(d_{\mathrm{out}})$ parameters that couple neighboring channels before applying $\Phi$.
\end{itemize}
For details on spline parameterizations, enforcing monotonicity of $\phi$, and (optional) adaptive spline domains, see Sections~\ref{sec:implementation} and~\ref{sec:theoretical_domains}.

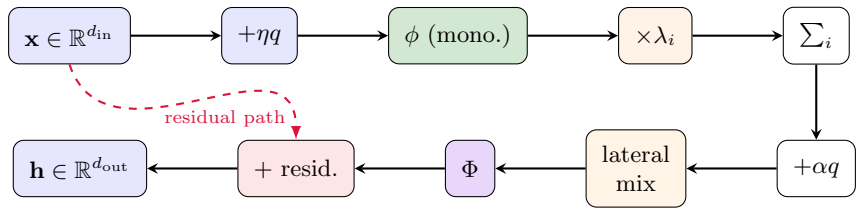
\begin{figure}[ht]
    \centering
    \begin{tikzpicture}[
        >=latex, 
        % Horizontal spacing 1.2cm (wide), Vertical 1.0cm (slightly taller for label space)
        node distance=1.0cm and 1.2cm, 
        font=\small
    ]
        % Define custom colors matching the presentation
        \definecolor{sprecherblue}{RGB}{70,130,180}
        \definecolor{sprecherred}{RGB}{220,20,60}
        \definecolor{sprechergreen}{RGB}{34,139,34}
        \definecolor{sprecherorange}{RGB}{255,140,0}
        \definecolor{sprecherpurple}{RGB}{147,51,234}

        % Define node styles
        \tikzstyle{op}=[
            draw, 
            rounded corners, 
            inner sep=6pt, 
            align=center, 
            minimum height=0.75cm
        ]
        \tikzstyle{flowarrow}=[->, thick, >=stealth, rounded corners]

        % --- Top Row (Forward Pass) ---
        \node[op, fill=blue!10] (input) {$\mathbf{x} \in \mathbb{R}^{d_{\mathrm{in}}}$};
        \node[op, fill=blue!10, right=of input] (shift) {$+\eta q$};
        \node[op, fill=sprechergreen!20, right=of shift] (phi) {$\phi$ (mono.)};
        \node[op, fill=orange!10, right=of phi] (lambda) {$\times \lambda_i$};
        \node[op, right=of lambda] (sum) {$\sum_i$};

        % --- Bottom Row (Reverse Direction) ---
        % Position 'plusq' directly below 'sum'
        \node[op, below=of sum] (plusq) {$+ \alpha q$};
        \node[op, fill=sprecherorange!10, left=of plusq] (mix) {lateral\\mix};
        \node[op, fill=sprecherpurple!20, left=of mix] (Phi) {$\Phi$};
        \node[op, fill=red!10, left=of Phi] (res) {+ resid.};
        \node[op, fill=blue!10, left=of res] (output) {$\mathbf{h} \in \mathbb{R}^{d_{\mathrm{out}}}$};

        % --- Main Flow Connections ---
        \draw[flowarrow] (input) -- (shift);
        \draw[flowarrow] (shift) -- (phi);
        \draw[flowarrow] (phi) -- (lambda);
        \draw[flowarrow] (lambda) -- (sum);
        \draw[flowarrow] (sum) -- (plusq);
        \draw[flowarrow] (plusq) -- (mix);
        \draw[flowarrow] (mix) -- (Phi);
        \draw[flowarrow] (Phi) -- (res);
        \draw[flowarrow] (res) -- (output);

        % --- Residual Connection (Fixed) ---
        % Removed fill=white. Positioned text 'right' of the curve to avoid overlap.
        \draw[->, dashed, sprecherred, thick] (input.south) 
            .. controls +(0.5,-1.5) and +(0,1.5) .. 
            node[pos=0.5, right, xshift=-16pt, yshift=-6pt, font=\scriptsize, text=sprecherred] {residual path}
            (res.north);

    \end{tikzpicture}
    \caption{Data flow through a single Sprecher block. Each input $x_i$ is shifted by $\eta q$ (where $q$ indexes outputs), passed through the shared monotonic spline $\phi$, weighted by $\lambda_i$, and summed. The pre-activation $s_q = \sum_i \lambda_i \phi(x_i + \eta q) + \alpha q$ undergoes optional lateral mixing before being transformed by the shared general spline $\Phi$. Residual connections (dashed) provide direct gradient paths.}
    \label{fig:block_flow}
\end{figure}

Concretely, we first define a single Sprecher block as an operator. Let $B^{(\ell)}: \mathbb{R}^{d_{\ell-1}} \to \mathbb{R}^{d_\ell}$ denote the $\ell$-th Sprecher block with parameters $\phi^{(\ell)}, \Phi^{(\ell)}, \eta^{(\ell)}, \lambda^{(\ell)}$, a fixed spacing constant $\alpha$, and optionally $\tau^{(\ell)}, \omega^{(\ell)}$. Given an input vector $\mathbf{x} = (x_1, \dots, x_{d_{\mathrm{in}}}) \in \mathbb{R}^{d_{\mathrm{in}}}$, the block computes its output vector $\mathbf{h} \in \mathbb{R}^{d_{\mathrm{out}}}$ component-wise as:
\begin{equation}\label{eq:SN}
[B^{(\ell)}(\mathbf{x})]_q = \Phi^{(\ell)}\Biggl(\,s_q^{(\ell)} + \tau^{(\ell)} \sum_{j \in \mathcal{N}(q)} \omega_{q,j}^{(\ell)} s_j^{(\ell)}\Biggr),
\end{equation}

where $s_q^{(\ell)} = \sum_{i=1}^{d_{\mathrm{in}}} \lambda_i^{(\ell)} \phi^{(\ell)}(x_i + \eta^{(\ell)} q) + \alpha q$ represents the pre-mixing activation, $\alpha$ is the fixed channel-spacing constant introduced in Section~\ref{sec:motivation}, and $\mathcal{N}(q)$ denotes the neighborhood structure:
\begin{itemize}
    \item \textbf{No mixing:} $\mathcal{N}(q) = \emptyset$ (reduces to original formulation)
    \item \textbf{Cyclic:} $\mathcal{N}(q) = \{(q+1) \bmod d_{\mathrm{out}}\}$ with a single weight $\omega_{q,j}^{(\ell)} = \omega_q^{(\ell)}$
    \item \textbf{Bidirectional:} $\mathcal{N}(q) = \{(q-1) \bmod d_{\mathrm{out}}, (q+1) \bmod d_{\mathrm{out}}\}$ with separate weights $\omega_{q,(q-1)\bmod d_{\mathrm{out}}}^{(\ell)}$ and $\omega_{q,(q+1)\bmod d_{\mathrm{out}}}^{(\ell)}$
\end{itemize}

While $\alpha = 1$ maintains theoretical fidelity, alternative values may be explored to improve optimization dynamics in deeper networks. Note that $q$ serves dual roles here: it indexes output channels, and it also acts as the fixed one-dimensional coordinate used by the additive shift terms. Although we write $q = 0, \ldots, d_{\mathrm{out}}-1$ for notational convenience, in practice one may instead use any fixed grid of $d_{\mathrm{out}}$ coordinates (e.g., a uniform grid on $[-1,1]$) to keep shifts well-scaled. (Note: we use 0-based indexing for output channels $q$ to align with the additive shift $+\alpha q$; raw inputs use 1-based indices $i=1,\ldots,d_{\mathrm{in}}$ following Sprecher's original convention, while intermediate layer outputs $\mathbf{h}^{(\ell)}_r$ inherit 0-based indexing from their role as output channels of the preceding block.)

In a network with multiple layers, each Sprecher block (indexed by $\ell=1, \dots, L$ or $L+1$) uses its own independent set of shared parameters $(\phi^{(\ell)}, \Phi^{(\ell)}, \eta^{(\ell)}, \lambda^{(\ell)})$ and optionally $(\tau^{(\ell)}, \omega^{(\ell)})$. The block operation implements a specific form of transformation: each input coordinate $x_i$ is first shifted by an amount depending on the output index $q$ and the shared shift parameter $\eta^{(\ell)}$, then passed through the shared monotonic spline $\phi^{(\ell)}$. The results are linearly combined using the learnable mixing weights $\lambda^{(\ell)}_{i}$, shifted again by the output index scaled by $\alpha$. When lateral mixing is enabled, these pre-activation values undergo weighted mixing with neighboring outputs. Finally, the result is passed through the shared general spline $\Phi^{(\ell)}$. Figure~\ref{fig:spline_sharing} illustrates this spline-sharing structure within a single block, highlighting the key architectural difference from KANs. Stacking these blocks creates a deep, compositional representation.

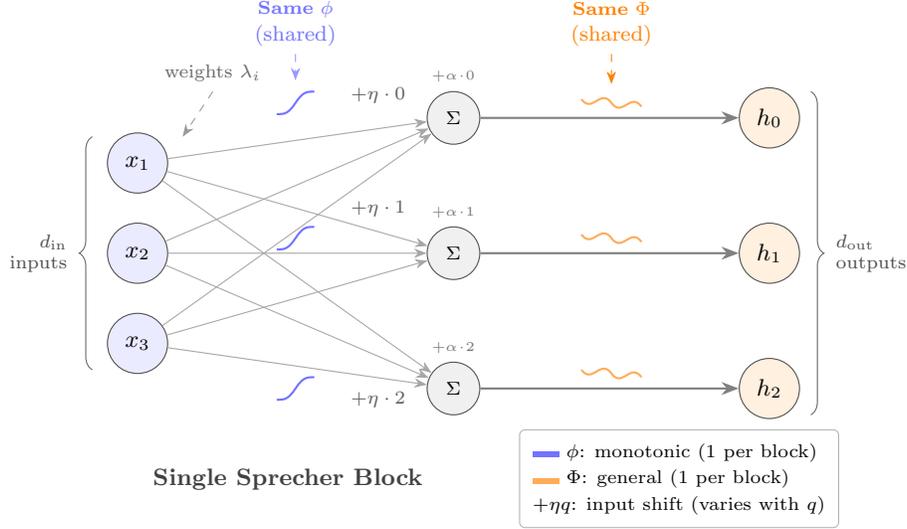
\begin{figure}[ht]
\centering
\begin{tikzpicture}[
    >=Stealth,
    node distance=0.6cm,
    inputnode/.style={circle, draw=black!70, fill=blue!8, minimum size=8mm, inner sep=1pt, font=\small},
    outputnode/.style={circle, draw=black!70, fill=orange!12, minimum size=8mm, inner sep=1pt, font=\small},
    sumnode/.style={circle, draw=black!70, fill=gray!12, minimum size=7mm, inner sep=0pt, font=\scriptsize},
    phicolor/.style={draw=blue!55, thick},
    Phicolor/.style={draw=orange!65, thick},
]

% === Input nodes ===
\node[inputnode] (x1) at (0, 1.8) {$x_1$};
\node[inputnode] (x2) at (0, 0.6) {$x_2$};
\node[inputnode] (x3) at (0, -0.6) {$x_3$};

% === Summation nodes (centered horizontally) ===
\node[sumnode] (s0) at (4.2, 2.4) {$\Sigma$};
\node[sumnode] (s1) at (4.2, 0.6) {$\Sigma$};
\node[sumnode] (s2) at (4.2, -1.2) {$\Sigma$};

% === Output nodes ===
\node[outputnode] (h0) at (8.4, 2.4) {$h_0$};
\node[outputnode] (h1) at (8.4, 0.6) {$h_1$};
\node[outputnode] (h2) at (8.4, -1.2) {$h_2$};

% === Edges from inputs to sum nodes ===
% Row q=0
\draw[->, black!35] (x1) -- (s0);
\draw[->, black!35] (x2) -- (s0);
\draw[->, black!35] (x3) -- (s0);
% Row q=1
\draw[->, black!35] (x1) -- (s1);
\draw[->, black!35] (x2) -- (s1);
\draw[->, black!35] (x3) -- (s1);
% Row q=2
\draw[->, black!35] (x1) -- (s2);
\draw[->, black!35] (x2) -- (s2);
\draw[->, black!35] (x3) -- (s2);

% === Edges from sum nodes to outputs ===
\draw[->, black!50, thick] (s0) -- (h0);
\draw[->, black!50, thick] (s1) -- (h1);
\draw[->, black!50, thick] (s2) -- (h2);

% === φ spline symbols (MONOTONIC - sigmoid/S-curve shape, going UP) ===
% Centered between inputs and sums (x ≈ 2.1)
\draw[phicolor] (1.85, 2.45) 
    .. controls (1.95, 2.45) and (2.0, 2.45) .. (2.1, 2.6)
    .. controls (2.2, 2.75) and (2.25, 2.75) .. (2.35, 2.75);
\draw[phicolor] (1.85, 0.65) 
    .. controls (1.95, 0.65) and (2.0, 0.65) .. (2.1, 0.8)
    .. controls (2.2, 0.95) and (2.25, 0.95) .. (2.35, 0.95);
% Shifted down to avoid edge overlap
\draw[phicolor] (1.85, -1.35) 
    .. controls (1.95, -1.35) and (2.0, -1.35) .. (2.1, -1.2)
    .. controls (2.2, -1.05) and (2.25, -1.05) .. (2.35, -1.05);

% === Φ spline symbols (general wavy shape) ===
% Centered between sums and outputs (x ≈ 6.3)
\draw[Phicolor] (5.9, 2.65) 
    .. controls (6.05, 2.45) and (6.15, 2.8) .. (6.3, 2.6)
    .. controls (6.45, 2.4) and (6.55, 2.75) .. (6.7, 2.55);
\draw[Phicolor] (5.9, 0.85) 
    .. controls (6.05, 0.65) and (6.15, 1.0) .. (6.3, 0.8)
    .. controls (6.45, 0.6) and (6.55, 0.95) .. (6.7, 0.75);
\draw[Phicolor] (5.9, -0.95) 
    .. controls (6.05, -1.15) and (6.15, -0.8) .. (6.3, -1.0)
    .. controls (6.45, -1.2) and (6.55, -0.85) .. (6.7, -1.05);

% === Shift labels (+η·q) - plain text, positioned between φ and Σ ===
\node[font=\scriptsize, text=black!65] at (3.2, 2.7) {$+\eta \cdot 0$};
\node[font=\scriptsize, text=black!65] at (3.2, 1.2) {$+\eta \cdot 1$};
\node[font=\scriptsize, text=black!65] at (3.2, -1.35) {$+\eta \cdot 2$};

% === Alpha·q labels (above Σ nodes) ===
\node[font=\tiny, text=black!55] at (4.2, 2.95) {$+\alpha\!\cdot\!0$};
\node[font=\tiny, text=black!55] at (4.2, 1.15) {$+\alpha\!\cdot\!1$};
\node[font=\tiny, text=black!55] at (4.2, -0.65) {$+\alpha\!\cdot\!2$};

% === Annotations for shared splines (more vertical spacing) ===
\node[font=\scriptsize, text=blue!55, align=center] at (2.1, 3.85) {\textbf{Same} $\phi$};
\node[font=\footnotesize, text=blue!55] at (2.1, 3.5) {(shared)};
\draw[->, blue!40, dashed] (2.1, 3.25) -- (2.1, 2.9);

% Changed text=orange!60 to orange!100 (or just orange) for better contrast
\node[font=\scriptsize, text=orange, align=center] at (6.3, 3.85) {\textbf{Same} $\Phi$};
\node[font=\footnotesize, text=orange] at (6.3, 3.5) {(shared)};
\draw[->, orange, dashed] (6.3, 3.25) -- (6.3, 2.85);

% === Lambda annotation (single label showing weights are shared) ===
\node[font=\scriptsize, text=black!60] at (1.0, 3.0) {weights $\lambda_i$};
\draw[->, black!40, dashed] (1.0, 2.75) -- (0.6, 2.15);

% === Braces ===
\draw[decorate, decoration={brace, amplitude=5pt}, black!60] 
    (-0.6, -0.95) -- (-0.6, 2.15) node[midway, left=6pt, font=\scriptsize, text=black!70, align=right] {$d_{\mathrm{in}}$\\inputs};

\draw[decorate, decoration={brace, amplitude=5pt, mirror}, black!60]
    (8.95, -1.55) -- (8.95, 2.75) node[midway, right=6pt, font=\scriptsize, text=black!70, align=left] {$d_{\mathrm{out}}$\\outputs};

% === Block title ===
\node[font=\small\bfseries, text=black!75] at (2.0, -2.4) {Single Sprecher Block};

% === Legend box (more spacing from nodes) ===
\node[draw=black!30, rounded corners=2pt, fill=white, inner sep=5pt, font=\scriptsize, align=left] 
    at (7.2, -2.4) {
    \textcolor{blue!55}{\rule{10pt}{2pt}} $\phi$: monotonic (1 per block)\\[2pt]
    \textcolor{orange!65}{\rule{10pt}{2pt}} $\Phi$: general (1 per block)\\[2pt]
    $+\eta q$: input shift (varies with $q$)
};

\end{tikzpicture}
\caption{Internal structure of a Sprecher block showing spline sharing. Unlike KANs where each edge has a unique learnable spline, a Sprecher block uses only \textbf{two shared splines}: one monotonic $\phi$ applied to all shifted inputs, and one general $\Phi$ applied to all weighted sums. Conceptually the computation is still dense (each output depends on every input through the weighted sum), but no parameters are tied to individual edges: all nonlinearity is shared via the two splines and the shared mixing vector $\lambda$. The mixing weights $\lambda_i$ are shared across all output dimensions. Diversity across outputs arises from the index-dependent shifts $+\eta q$ (applied before $\phi$) and $+\alpha q$ (added to each sum). This extreme parameter sharing yields $O(d_{\mathrm{in}} + G)$ base parameters per block (plus $O(d_{\mathrm{out}})$ when lateral mixing is enabled) versus $O(d_{\mathrm{in}} \cdot d_{\mathrm{out}} \cdot G)$ for KANs.}
\label{fig:spline_sharing}
\end{figure}

\begin{remark}[Computational considerations]
While the block operation as written suggests computing all $d_{\mathrm{out}}$ outputs simultaneously, implementations may compute them sequentially to reduce memory usage from $O(B \cdot d_{\mathrm{in}} \cdot d_{\mathrm{out}})$ to $O(B \cdot \max(d_{\mathrm{in}}, d_{\mathrm{out}}))$ where $B$ is the batch size. This sequential computation produces mathematically identical results to the parallel formulation. This is particularly valuable when exploring architectures with wide layers, where memory constraints often limit feasible network configurations before parameter count becomes prohibitive. Section~\ref{sec:memory_efficient} details this sequential computation strategy.
\end{remark}

\subsection{Optional enhancements}\label{sec:optional_enhancements}
Several optional components can enhance the basic Sprecher block:

\begin{itemize}
    \item \textbf{Residual connections (node-wise / cyclic):} When enabled, residual connections are implemented using a dimension-adaptive modulo-based cyclic assignment that maintains architectural coherence. We refer to this $O(\max(d_{\mathrm{in}}, d_{\mathrm{out}}))$-parameter construction as a \emph{node-wise} residual (in contrast to a dense linear residual projection). Like convolutional layers that achieve parameter efficiency through local connectivity patterns, cyclic residuals use periodic patterns to reduce parameters from $O(d_{\mathrm{in}} \times d_{\mathrm{out}})$ to $O(\max(d_{\mathrm{in}}, d_{\mathrm{out}}))$. This method adapts based on dimensional relationships between adjacent layers:
    \begin{itemize}
        \item \textbf{Identity} ($d_{\mathrm{in}} = d_{\mathrm{out}}$): Uses a single learnable weight $w_{\text{res}}$
        \item \textbf{Broadcasting} ($d_{\mathrm{in}} < d_{\mathrm{out}}$): Each output dimension cyclically selects an input via $1+(q \bmod d_{\mathrm{in}})$ and applies a learnable scale
        \item \textbf{Pooling} ($d_{\mathrm{in}} > d_{\mathrm{out}}$): Input dimensions are cyclically assigned to outputs via $(i-1) \bmod d_{\mathrm{out}}$ with learnable weights
    \end{itemize}
    Specifically, the block output with residual connection becomes:
    $$[B^{(\ell)}_{\text{res}}(\mathbf{x})]_q = [B^{(\ell)}(\mathbf{x})]_q + \begin{cases}
    w_{\text{res}} \cdot x_{q+1} & \text{if } d_{\mathrm{in}} = d_{\mathrm{out}}\\
    w_q^{\text{bcast}} \cdot x_{1+(q \bmod d_{\mathrm{in}})} & \text{if } d_{\mathrm{in}} < d_{\mathrm{out}}\\
    \sum_{i: (i-1) \bmod d_{\mathrm{out}} = q} w_i^{\text{pool}} \cdot x_i & \text{if } d_{\mathrm{in}} > d_{\mathrm{out}}
    \end{cases}$$
    where $w^{\text{bcast}} \in \mathbb{R}^{d_{\mathrm{out}}}$ and $w^{\text{pool}} \in \mathbb{R}^{d_{\mathrm{in}}}$ are learnable weight vectors. Empirically, these constrained connections often perform comparably to full projection residuals while using substantially fewer residual parameters.
    
    \item \textbf{Lateral mixing connections (distinct from cyclic residuals):} Inspired by the success of attention mechanisms and lateral connections in vision models, we introduce an optional intra-block communication mechanism. Before applying the outer spline $\Phi^{(\ell)}$, each output dimension can incorporate weighted contributions from neighboring outputs:
    $$\tilde{s}_q^{(\ell)} = s_q^{(\ell)} + \tau^{(\ell)} \cdot \begin{cases}
    \omega_q^{(\ell)} \cdot s_{(q+1) \bmod d_{\mathrm{out}}}^{(\ell)} & \text{(cyclic)} \\
    \omega_{q,\text{fwd}}^{(\ell)} \cdot s_{(q+1) \bmod d_{\mathrm{out}}}^{(\ell)} + \omega_{q,\text{bwd}}^{(\ell)} \cdot s_{(q-1) \bmod d_{\mathrm{out}}}^{(\ell)} & \text{(bidirectional)}
    \end{cases}$$
    This mechanism allows the network to learn correlations between output dimensions while maintaining the parameter efficiency of the architecture, adding only $O(d_{\mathrm{out}})$ parameters per block. Empirically, we find this particularly beneficial for vector-valued outputs and deeper networks.
    
    \item \textbf{Normalization:} Batch normalization can be applied to block outputs to improve training stability, particularly in deeper networks. See Section \ref{sec:normalization} for details.
    
    \item \textbf{Output affine head:} The network may include a learnable affine head applied to the final network output: $f_{\text{final}}(\mathbf{x}) = s_{\text{out}}\, f(\mathbf{x}) + b_{\text{out}}$, where $s_{\text{out}}, b_{\text{out}}\in\mathbb{R}$ (2 parameters) are shared across all output dimensions. For supervised regression tasks we initialize $b_{\text{out}}$ to the mean of the training targets (averaged over all outputs when $m>1$) and set $s_{\text{out}}=0.1$; for classification and PINN experiments we keep the default initialization $s_{\text{out}}=1$ and $b_{\text{out}}=0$.

\end{itemize}

\begin{remark}[Lateral mixing as structured attention]
The lateral mixing mechanism can be viewed as a highly constrained form of self-attention where each output dimension ``attends'' only to its immediate neighbors in a cyclic topology (see Figure~\ref{fig:lateral_mixing}). Unlike full attention which requires $O(d_{\mathrm{out}}^2)$ parameters, our approach maintains linear scaling while still enabling cross-dimensional information flow. This design choice reflects our philosophy of extreme parameter sharing: just as the Sprecher structure shares splines across all connections within a block, lateral mixing shares the communication pattern across all samples while learning only the mixing weights.
\end{remark}

\begin{figure}[ht]
    \centering
    \definecolor{sprecherblue}{RGB}{70,130,180}
    \definecolor{sprechergreen}{RGB}{34,139,34}
    \begin{tikzpicture}[scale=1.1]
        % Cyclic mixing diagram
        \begin{scope}[xshift=0cm]
            \node[circle, draw, fill=sprecherblue!30, minimum size=12mm, font=\small] (n0) at (0:1.8) {$s_0$};
            \node[circle, draw, fill=sprecherblue!30, minimum size=12mm, font=\small] (n1) at (72:1.8) {$s_1$};
            \node[circle, draw, fill=sprecherblue!30, minimum size=12mm, font=\small] (n2) at (144:1.8) {$s_2$};
            \node[circle, draw, fill=sprecherblue!30, minimum size=12mm, font=\small] (n3) at (216:1.8) {$s_3$};
            \node[circle, draw, fill=sprecherblue!30, minimum size=12mm, font=\small] (n4) at (288:1.8) {$s_4$};
            \draw[->, thick, sprechergreen, bend left=18] (n1) to (n0);
            \node[font=\scriptsize, text=sprechergreen] at (36:2) {$\omega_0$};
            \draw[->, thick, sprechergreen, bend left=18] (n2) to (n1);
            \node[font=\scriptsize, text=sprechergreen] at (108:2) {$\omega_1$};
            \draw[->, thick, sprechergreen, bend left=18] (n3) to (n2);
            \node[font=\scriptsize, text=sprechergreen] at (180:2) {$\omega_2$};
            \draw[->, thick, sprechergreen, bend left=18] (n4) to (n3);
            \node[font=\scriptsize, text=sprechergreen] at (252:2) {$\omega_3$};
            \draw[->, thick, sprechergreen, bend left=18] (n0) to (n4);
            \node[font=\scriptsize, text=sprechergreen] at (324:2) {$\omega_4$};
            \node[font=\small\bfseries] at (0,-3.0) {Cyclic lateral mixing};
            \node[font=\footnotesize, align=center] at (0,-3.6) {$\tilde{s}_q = s_q + \tau \omega_q s_{(q+1) \bmod d_{\mathrm{out}}}$};
        \end{scope}
        
        % Equation box
        \begin{scope}[xshift=6.5cm, yshift=-0.5cm]
            \node[draw, rounded corners, fill=gray!5, inner sep=10pt, align=left, font=\small] {
                \textbf{Pre-activation:}\\[3pt]
                $s_q = \sum_{i=1}^{d_{\mathrm{in}}} \lambda_i \phi(x_i + \eta q) + \alpha q$\\[8pt]
                \textbf{After lateral mixing:}\\[3pt]
                $\tilde{s}_q = s_q + \tau \sum_{j \in \mathcal{N}(q)} \omega_{q,j} s_j$\\[8pt]
                \textbf{Output:}\\[3pt]
                $h_q = \Phi(\tilde{s}_q)$
            };
        \end{scope}
    \end{tikzpicture}
    \caption{Lateral mixing enables cross-dimensional communication before the outer spline $\Phi$.
    In the cyclic variant (shown), each pre-activation $s_q$ receives a scaled contribution from its neighbor $s_{(q+1) \bmod d_{\mathrm{out}}}$, parameterized by scale $\tau$ and per-output weights $\omega_q$.
    This adds only $O(d_{\mathrm{out}})$ parameters while breaking the symmetries inherent in the shared-weight structure.
    The bidirectional variant additionally includes contributions from $s_{(q-1) \bmod d_{\mathrm{out}}}$.}
    \label{fig:lateral_mixing}
\end{figure}

\begin{remark}[Lateral mixing resolves the shared-weight optimization plateau]
Wide single-layer Sprecher Networks exhibit a particularly severe optimization challenge beyond the typical difficulties of shallow networks. Due to the shared weight vector $\lambda$ (rather than a weight matrix), all output dimensions process identical linear combinations of inputs, 
differing only through the deterministic index-dependent shifts $x_i\mapsto x_i+\eta q$ (inside $\phi$) and the additive term $+\alpha q$ (in the pre-activation), leading to nearly identical outputs early in training. This symmetry can cause outputs to collapse into similar representations, reducing effective capacity and slowing convergence. 

Cyclic residual connections help by injecting input information directly into each output, 
but in very wide layers this may still be insufficient to break output symmetry and avoid plateaus.
As an illustrative example, a $2\to[120]\to1$ network trained on the 2D regression benchmark (\textsc{Toy-2D-Complex}) $f(x_1,x_2)=\exp(\sin(11x_1))+3x_2+4\sin(8x_2)$ over $[0,1]^2$ (with a fixed $32\times 32$ training grid) reaches an MSE of $8.55$ after $40{,}000$ epochs with cyclic residuals (broadcast weights; $120$ residual parameters). Replacing the cyclic residual with a dense projection residual ($2\times120=240$ residual parameters) provides no improvement (MSE: $8.55$), while adding cyclic lateral mixing (a scalar $\tau$ plus $120$ mixing weights) adds only $121$ parameters and drops the MSE to $2.12$. This demonstrates that the problem is not simply the general difficulty of optimizing wide shallow networks (which affects MLPs too), but 
specifically the need to break symmetries created by the shared-weight constraint. The lateral mixing mechanism provides the minimal cross-dimensional communication needed to differentiate the outputs beyond mere shifting, enabling successful optimization where traditional approaches fail.
\end{remark}

\begin{remark}[Cyclic residuals as dimensional folding]
The cyclic assignment pattern can be understood geometrically as a form of dimensional folding. When $d_{\mathrm{in}} > d_{\mathrm{out}}$, the modulo operation $(i-1) \bmod d_{\mathrm{out}}$ effectively ``wraps'' the higher-dimensional input space around the lower-dimensional output space, analogous to wrapping a string around a cylinder. This creates a regular pattern where input dimensions are distributed uniformly across outputs, ensuring complete coverage and direct gradient paths. The cyclic pattern maximally separates consecutive inputs; for instance, with $d_{\mathrm{in}}=7, d_{\mathrm{out}}=3$, consecutive inputs $\{1,2,3,4,5,6,7\}$ map to outputs $\{0,1,2,0,1,2,0\}$ respectively, preventing the local clustering that might occur with contiguous block assignments.
\end{remark}

\begin{table}[ht]
\centering
\caption{Cyclic residuals vs.\ linear residual projections on synthetic regression tasks. We report the best MSE on a fixed 1{,}024-point training set (measured in evaluation mode) achieved within the stated number of epochs. Toy-2D-Complex refers to the scalar 2D regression benchmark defined above; Toy-2D (vector) refers to the vector-valued 2D benchmark in Fig.~\ref{fig:twovarsprechervector_revised}; Toy-$4\to5$ is a $[0,1]^4\to\mathbb{R}^5$ regression task with outputs $f_1=\sin(2\pi x_1)\cos(\pi x_2)$, $f_2=\exp(-2(x_1^2+x_2^2))$, $f_3=x_3^3-x_4^2+\tfrac{1}{2}\sin(5x_3)$, $f_4=\sigma(3(x_1+x_2-x_3+x_4))$, $f_5=\tfrac{1}{2}\sin(4\pi x_1 x_4)+\tfrac{1}{2}\cos(3\pi x_2 x_3)$ (where $\sigma$ denotes sigmoid). ``Linear'' uses a learned scalar when $d_{\mathrm{in}}=d_{\mathrm{out}}$ and a learned matrix $W\in\mathbb{R}^{d_{\mathrm{in}}\times d_{\mathrm{out}}}$ otherwise; ``cyclic'' uses the pooling/broadcast construction when $d_{\mathrm{in}}\neq d_{\mathrm{out}}$ (and a learned scalar when $d_{\mathrm{in}}=d_{\mathrm{out}}$). The last columns count residual-path parameters only. For tasks with $d_{\mathrm{out}}>1$, MSE is computed as the mean of $(\hat y-y)^2$ over all output dimensions and samples.}
\label{tab:cyclic-residual-ablation}
\small
\resizebox{\textwidth}{!}{%
\begin{tabular}{@{}llrrrrr@{}}
\toprule
Task & Architecture & Epochs & MSE (linear) & MSE (cyclic) & \#resid params (lin./cyc.) & Ratio \\
\midrule
Toy-2D-Complex & $2\rightarrow[10,10,10]\rightarrow1$ & 10{,}000 & $4.33\times10^{-2}$ & $\mathbf{1.79\times10^{-3}}$ & $22/12$ & $1.8\times$ \\
Toy-2D-Complex & $2\rightarrow[10,11,12,13,14,15,16,17]\rightarrow1$ & 50{,}000 & $\mathbf{9.77\times10^{-6}}$ & $3.88\times10^{-5}$ & $1322/108$ & $12.2\times$ \\
Toy-2D-Complex & $2\rightarrow[10,11,10,11,10,11,10,11,10,11,10,11,10,11,10,11]\rightarrow1$ & 50{,}000 & $\mathbf{1.52\times10^{-5}}$ & $3.89\times10^{-5}$ & $1670/175$ & $9.5\times$ \\
Toy-2D (vector) & $2\rightarrow[50,50,50]\rightarrow2$ & 10{,}000 & $7.50\times10^{-6}$ & $\mathbf{2.46\times10^{-6}}$ & $202/102$ & $2.0\times$ \\
Toy-2D (vector) & $2\rightarrow[10,11,12,13,14,15,16,17]\rightarrow2$ & 5{,}000 & $\mathbf{5.47\times10^{-3}}$ & $9.09\times10^{-2}$ & $1356/125$ & $10.8\times$ \\
Toy-$4\to5$ & $4\rightarrow[30,40]\rightarrow5$ & 10{,}000 & $\mathbf{8.01\times10^{-3}}$ & $1.60\times10^{-2}$ & $1520/110$ & $13.8\times$ \\
Toy-$4\to5$ & $4\rightarrow[10,11,10,11,10,11,10,11,10,11,10,11,10,11,10,11]\rightarrow5$ & 10{,}000 & $\mathbf{3.24\times10^{-4}}$ & $5.51\times10^{-3}$ & $1745/186$ & $9.4\times$ \\
\bottomrule
\end{tabular}%
}
\end{table}

\begin{table}[t]
\centering
\caption{Ablation on the 2D vector-valued benchmark defined in Fig.~\ref{fig:twovarsprechervector_revised} (``Toy-2D-Vector''). Mean test RMSE $\pm$ std over 10 seeds (5{,}000 epochs per run). Lower is better. Here, ``+ domain tracking'' updates spline domains without coefficient resampling, whereas ``+ resampling'' enables resampling.}

\label{tab:toy2d_addons_ablation}
\small
\begin{tabular}{lc}
\toprule
Setting & Test RMSE \\
\midrule
Base SN (no residuals/mixing/domain tracking) & $0.325 \pm 0.015$ \\
+ cyclic residuals & $0.226 \pm 0.065$ \\
+ bidirectional mixing & $0.211 \pm 0.055$ \\
+ domain tracking & $0.046 \pm 0.050$ \\
+ resampling (full SN) & $\mathbf{0.029 \pm 0.020}$ \\
\bottomrule
\end{tabular}
\end{table}

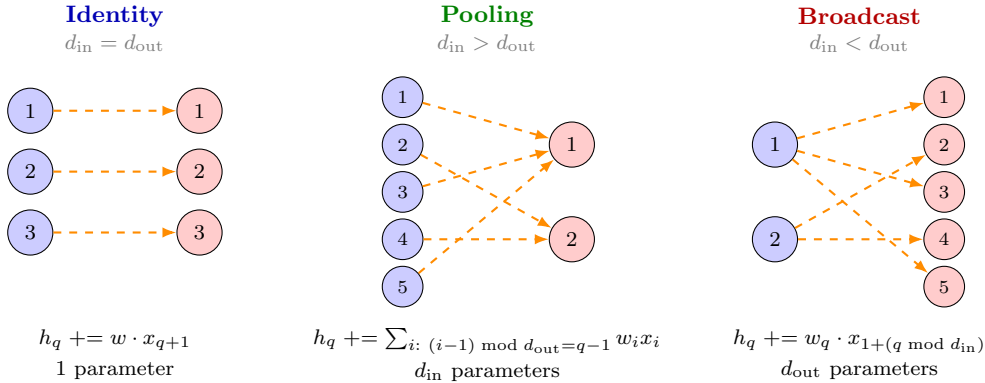
\begin{figure}[ht]
    \centering
    \definecolor{sprecherorange}{RGB}{255,140,0}
    % Added >={Latex[scale=0.8]} to globally shrink arrowheads
    \begin{tikzpicture}[scale=0.9, >={Latex[scale=0.8]}]
        % Identity case
        \begin{scope}[xshift=0cm]
            \node[font=\small\bfseries, color=blue!70!black] at (1.25,0.5) {Identity};
            \node[font=\footnotesize, color=gray] at (1.25,0.1) {$d_{\mathrm{in}} = d_{\mathrm{out}}$};
            
            \foreach \i in {1,2,3} {
                \node[circle, draw, fill=blue!20, minimum size=6mm, font=\footnotesize] (x\i) at (0, -\i*0.9) {\i};
                \node[circle, draw, fill=red!20, minimum size=6mm, font=\footnotesize] (y\i) at (2.5, -\i*0.9) {\i};
                \draw[->, dashed, sprecherorange, thick] (x\i) -- (y\i);
            }
            
            \node[font=\footnotesize, align=center] at (1.25,-4.5) {$h_q \mathrel{+}= w \cdot x_{q+1}$\\[2pt]\footnotesize 1 parameter};
        \end{scope}
        
        % Pooling case
        \begin{scope}[xshift=5.5cm]
            \node[font=\small\bfseries, color=green!50!black] at (1.25,0.5) {Pooling};
            \node[font=\footnotesize, color=gray] at (1.25,0.1) {$d_{\mathrm{in}} > d_{\mathrm{out}}$};
            
            \foreach \i in {1,...,5} {
                \node[circle, draw, fill=blue!20, minimum size=5mm, font=\scriptsize] (px\i) at (0, -\i*0.7) {\i};
            }
            \node[circle, draw, fill=red!20, minimum size=6mm, font=\footnotesize] (py1) at (2.5, -1.4) {1};
            \node[circle, draw, fill=red!20, minimum size=6mm, font=\footnotesize] (py2) at (2.5, -2.8) {2};

            % Cyclic assignment
            \draw[->, dashed, sprecherorange, thick] (px1) -- (py1);
            \draw[->, dashed, sprecherorange, thick] (px2) -- (py2);
            \draw[->, dashed, sprecherorange, thick] (px3) -- (py1);
            \draw[->, dashed, sprecherorange, thick] (px4) -- (py2);
            \draw[->, dashed, sprecherorange, thick] (px5) -- (py1);

            \node[font=\footnotesize, align=center] at (1.25,-4.5) {$h_q \mathrel{+}= \sum_{i:\ (i-1)\bmod d_{\mathrm{out}} = q-1} w_i x_i$\\[2pt]\footnotesize $d_{\mathrm{in}}$ parameters};
        \end{scope}
        
        % Broadcast case
        \begin{scope}[xshift=11cm]
            \node[font=\small\bfseries, color=red!70!black] at (1.25,0.5) {Broadcast};
            \node[font=\footnotesize, color=gray] at (1.25,0.1) {$d_{\mathrm{in}} < d_{\mathrm{out}}$};
            
            \node[circle, draw, fill=blue!20, minimum size=6mm, font=\footnotesize] (bx1) at (0, -1.4) {1};
            \node[circle, draw, fill=blue!20, minimum size=6mm, font=\footnotesize] (bx2) at (0, -2.8) {2};

            \foreach \j in {1,...,5} {
                \node[circle, draw, fill=red!20, minimum size=5mm, font=\scriptsize] (by\j) at (2.5, -\j*0.7) {\j};
            }
            
            % Cyclic source
            \draw[->, dashed, sprecherorange, thick] (bx1) -- (by1);
            \draw[->, dashed, sprecherorange, thick] (bx2) -- (by2);
            \draw[->, dashed, sprecherorange, thick] (bx1) -- (by3);
            \draw[->, dashed, sprecherorange, thick] (bx2) -- (by4);
            \draw[->, dashed, sprecherorange, thick] (bx1) -- (by5);

            \node[font=\footnotesize, align=center] at (1.25,-4.5) {$h_q \mathrel{+}= w_q \cdot x_{1+(q \bmod d_{\mathrm{in}})}$\\[2pt]\footnotesize $d_{\mathrm{out}}$ parameters};
        \end{scope}
    \end{tikzpicture}
    \caption{Dimension-adaptive cyclic residual connections (indices shown 1-based for readability). \textbf{Left:} When dimensions match, a single scalar weight applies element-wise. \textbf{Center:} When $d_{\mathrm{in}} > d_{\mathrm{out}}$ (pooling), multiple inputs are cyclically assigned to each output via modular indexing. \textbf{Right:} When $d_{\mathrm{in}} < d_{\mathrm{out}}$ (broadcast), inputs are cyclically reused across outputs. All cases maintain $O(\max(d_{\mathrm{in}}, d_{\mathrm{out}}))$ parameters, preserving the linear scaling of the architecture.}
    \label{fig:cyclic_residuals}
\end{figure}

\begin{remark}[Structured sparsity and architectural coherence]
The cyclic residual design can be viewed as implementing a structured sparse projection where each output connects to approximately $d_{\mathrm{in}}/d_{\mathrm{out}}$ inputs (pooling case) or each input connects to approximately $d_{\mathrm{out}}/d_{\mathrm{in}}$ outputs (broadcasting case). This represents another form of weight sharing that complements the main Sprecher architecture: while Sprecher blocks share splines across output dimensions with diversity through shifts, residual connections share structure across connections with diversity through cyclic assignment. The unexpected effectiveness of this severe constraint suggests that for gradient flow and information propagation, the specific connection pattern may matter less than ensuring uniform coverage and balanced paths, a hypothesis that warrants further theoretical investigation.
\end{remark}

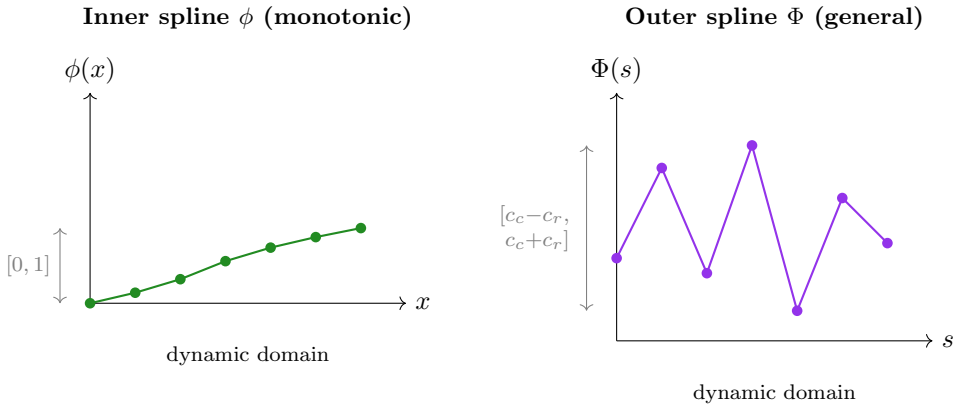
\begin{figure}[ht]
    \centering
    \definecolor{sprechergreen}{RGB}{34,139,34}
    \definecolor{sprecherpurple}{RGB}{147,51,234}
    \begin{tikzpicture}
        % === Inner spline phi (Left) ===
        \begin{scope}[xshift=0cm]
            % Subfigure Title (Moved up to 3.8 for better clearance)
            \node[font=\small\bfseries] at (2.1, 3.8) {Inner spline $\phi$ (monotonic)};

            % Axes
            \draw[->] (0,0) -- (4.2,0) node[right] {$x$};
            % Label 'above' the arrow tip
            \draw[->] (0,0) -- (0,2.8) node[above] {$\phi(x)$};
            
            % Draw monotonic spline
            \draw[thick, sprechergreen] 
                (0,0) -- (0.6,0.14) -- (1.2,0.32) -- (1.8,0.56) -- 
                (2.4,0.74) -- (3.0,0.88) -- (3.6,1.0);
            
            % Mark knots
            \foreach \x/\y in {0/0, 0.6/0.14, 1.2/0.32, 1.8/0.56, 2.4/0.74, 3.0/0.88, 3.6/1.0} {
                \fill[sprechergreen] (\x,\y) circle (2pt);
            }
            
            % Codomain annotation
            \draw[<->, gray] (-0.4,0) -- (-0.4,1.0);
            \node[left, gray, font=\footnotesize] at (-0.4,0.5) {$[0,1]$};
            
            % Label: Fixed offset from x-axis
            \node[font=\footnotesize] at (2.1,-0.7) {dynamic domain};
        \end{scope}
        
        % === Outer spline Phi (Right) ===
        \begin{scope}[xshift=7cm]
            % Subfigure Title (Moved up to 3.8)
            \node[font=\small\bfseries] at (2.1, 3.8) {Outer spline $\Phi$ (general)};

            % Axes (x-axis shifted down to y=-0.5)
            \draw[->] (0,-0.5) -- (4.2,-0.5) node[right] {$s$};
            % Label 'above'
            \draw[->] (0,-0.5) -- (0,2.8) node[above] {$\Phi(s)$};
            
            % Draw general (non-monotonic) spline
            \draw[thick, sprecherpurple] 
                (0,0.6) -- (0.6,1.8) -- (1.2,0.4) -- (1.8,2.1) -- 
                (2.4,-0.1) -- (3.0,1.4) -- (3.6,0.8);
            
            % Mark knots
            \foreach \x/\y in {0/0.6, 0.6/1.8, 1.2/0.4, 1.8/2.1, 2.4/-0.1, 3.0/1.4, 3.6/0.8} {
                \fill[sprecherpurple] (\x,\y) circle (2pt);
            }
            
            % Codomain annotation
            \draw[<->, gray] (-0.4,-0.1) -- (-0.4,2.1);
            \node[left, gray, font=\footnotesize, align=right] at (-0.5,1.0) {$[c_c{-}c_r,$\\$c_c{+}c_r]$};
            
            % Label: Fixed offset from x-axis (y=-0.5 -> pos -1.2)
            \node[font=\footnotesize] at (2.1,-1.2) {dynamic domain};
        \end{scope}
    \end{tikzpicture}
    \caption{The dual spline system in Sprecher Networks. Left: The inner spline $\phi$ is monotonic (non-decreasing) with fixed codomain $[0, 1]$, parameterized via cumulative sums of softplus-transformed increments; it is strictly increasing on its spline domain and uses constant extension outside that domain ($0$ for inputs below the leftmost knot, $1$ above the rightmost knot). Right: The outer spline $\Phi$ is a general (non-monotonic) univariate spline (piecewise-linear or cubic PCHIP) with optional learnable codomain parameters $(c_c, c_r)$ defining center and radius. It uses linear extrapolation outside its domain. Both domains can be updated during training as $\lambda$ and $\eta$ evolve; unless explicitly stated otherwise, we perform dynamic domain updates during training (the PINN setting in Sec.~\ref{sec:pinn_poisson} is an exception: spline domains are kept fixed (no domain updates), as described there).}
    \label{fig:dual_splines}
\end{figure}

\subsection{Normalization considerations}\label{sec:normalization}

The unique structure of Sprecher Networks requires careful consideration when incorporating normalization techniques. While standard neural networks typically normalize individual neuron activations, the shared-spline architecture of SNs suggests applying normalization only at block boundaries: we normalize the block output vector $\mathbf{h}^{(\ell)} \in \mathbb{R}^{d_\ell}$ (featurewise, as in standard BatchNorm) rather than normalizing the internal scalar pre-activations within a block.

Batch normalization can be placed either \emph{before} or \emph{after} a Sprecher block. In both placements it acts on the vector features $\mathbf{h}$, not on the internal scalars $s^{(\ell)}_q$, so the additive $+\alpha q$ and any lateral mixing remain entirely inside $\Phi^{(\ell)}$.
\begin{itemize}
    \item \textbf{After block (default).} Apply BN to the block output $\mathbf{h}^{(\ell)} \in \mathbb{R}^{d_\ell}$.
    \item \textbf{Before block.} Apply BN to the input $\mathbf{h}^{(\ell-1)} \in \mathbb{R}^{d_{\ell-1}}$ before it enters the block.
\end{itemize}

The transformation takes the familiar form:
$$\tilde{\mathbf{h}} = \gamma \odot \frac{\mathbf{h} - \mu}{\sqrt{\sigma^2 + \epsilon}} + \beta$$
where $\gamma, \beta \in \mathbb{R}^{d}$ are learnable affine parameters, and the statistics $\mu$ and $\sigma^2$ are computed per-dimension across the batch.

This approach maintains the parameter efficiency central to SNs: each normalization layer adds only $2d$ parameters, preserving the $O(LN)$ scaling. Moreover, by treating the block output as a unified representation rather than normalizing individual components, the method respects the architectural philosophy that all output dimensions arise from shared transformations through $\phi^{(\ell)}$ and $\Phi^{(\ell)}$.

We often find it beneficial to skip normalization for the first block (default setting), allowing the network to directly process the input features while still stabilizing deeper layers. This also preserves the canonical assumption that raw inputs lie in $[0,1]^n$; if normalization is applied before the first block, bound propagation simply starts from the correspondingly transformed input interval.

When computing theoretical domains (Section~\ref{sec:theoretical_domains}), the bounds must account for normalization, since subsequent $\phi^{(\ell)}$ splines receive \emph{post-normalization} activations:
\begin{itemize}
    \item \textbf{Batch-statistics mode (training, and also the ``batch-stat'' evaluation protocol used in several of our experiments):} Conservative bounds assume standardization to approximately $[-4, 4]$ (covering ~99.99\% of a standard normal distribution); this is a heuristic bound, with affine transformation applied afterward.
    \item \textbf{Running-statistics mode (standard BatchNorm evaluation):} Uses the stored per-channel running statistics (and affine parameters, if enabled) to compute tighter bounds.
\end{itemize}
This means spline-domain updates track the post-normalization ranges, so normalization is explicitly accounted for when setting subsequent spline domains.

In theoretical domain propagation, the interval passed to the next block is always taken \emph{after} applying the normalization map (whether normalization is placed before or after a block), so subsequent $\phi$ splines do not ``miss'' values due to normalization.

Our empirical findings suggest that the combination of normalization with residual connections and lateral mixing proves particularly effective for deeper SNs, enabling successful training of networks with many blocks while maintaining the characteristic parameter efficiency of the architecture.

\subsection{Layer composition and final mapping}
Let $L$ be the number of hidden layers specified by the architecture $[d_1, \dots, d_L]$. In our framework, a ``hidden layer'' corresponds to the vector output of a Sprecher block. The mapping from the representation at layer $\ell-1$ to layer $\ell$ is implemented by the $\ell$-th Sprecher block.

The $\ell$-th block map is given in Eq.~\eqref{eq:SN} (introduced above). Throughout, we write $\tilde{\mathbf{h}}^{(\ell)}$ for the pre-normalization output of block $\ell$ (after $\Phi^{(\ell)}$ and any residual term), and $\mathbf{h}^{(\ell)}$ for the vector passed to the next block after any optional normalization (otherwise $\mathbf{h}^{(\ell)}=\tilde{\mathbf{h}}^{(\ell)}$). More explicitly, the (unmixed) pre-activation for output $q$ is
$$s_q^{(\ell)} := \sum_{i=1}^{d_{\ell-1}} \lambda^{(\ell)}_{i}\,\phi^{(\ell)}\Bigl(\mathbf{h}^{(\ell-1)}_i+\eta^{(\ell)}\,q\Bigr) + \alpha q,$$
and the lateral mixing adds a weighted combination of neighboring $s_j^{(\ell)}$ values before applying $\Phi^{(\ell)}$. The term $R^{(\ell)}(\mathbf{h}^{(\ell-1)})_q$ denotes an (optional) residual connection (e.g., the node-wise/cyclic residual of Fig.~\ref{fig:cyclic_residuals} or a standard learned linear projection); set $R^{(\ell)}\equiv 0$ when residuals are disabled. Recall that $\tilde{\mathbf{h}}^{(\ell)}$ denotes the pre-normalization block output and $\mathbf{h}^{(\ell)}$ the post-normalization vector (Section~\ref{sec:normalization}).

\begin{remark}[On the nature of composition]
Note that in this $L>1$ composition (Eq. \ref{eq:SN}), the argument to the inner spline $\phi^{(\ell)}$ is $\mathbf{h}^{(\ell-1)}_i$, the output of the previous layer, not the original input coordinate $x_i$. This is the fundamental departure from Sprecher's construction and is the defining feature of our proposed deep architecture. The motivation for this compositional structure comes not from Sprecher's work, but from the empirical success of the deep learning paradigm. Each layer processes the complex, transformed output of the layer before it, enabling the network to learn hierarchical representations.
\end{remark}

The composition of these blocks and the final output generation depend on the desired final output dimension $m=d_{\mathrm{out}}$:

\paragraph{(a) Scalar output ($m=1$):}
One can construct a scalar-output SN using exactly $L$ Sprecher blocks. The output of the final block, $\mathbf{h}^{(L)} \in \mathbb{R}^{d_L}$, is aggregated by summation to yield the scalar output:
$$ f(\mathbf{x}) = \sum_{q=0}^{d_L-1} \mathbf{h}^{(L)}_q. $$
If we define the operator for the $\ell$-th block (including residuals when enabled) as $T^{(\ell)}: \mathbb{R}^{d_{\ell-1}} \to \mathbb{R}^{d_\ell}$, where
\begin{equation}\label{eq:SN_full}
\Bigl(T^{(\ell)}(\mathbf{z})\Bigr)_q = \Phi^{(\ell)}\Biggl(\,\sum_{i} \lambda^{(\ell)}_{i}\,\phi^{(\ell)}\Bigl(z_i+\eta^{(\ell)}\,q\Bigr) + \alpha q + \tau^{(\ell)} \sum_{j \in \mathcal{N}(q)} \omega_{q,j}^{(\ell)} s_j^{(\ell)}\Biggr) + R^{(\ell)}(\mathbf{z})_q,
\end{equation}
then the overall function is
$$ f(\mathbf{x}) = \sum_{q=0}^{d_L-1} \Bigl(T^{(L)} \circ T^{(L-1)} \circ \cdots \circ T^{(1)}\Bigr)(\mathbf{x})_q. $$
(The summation index $i$ ranges over the input coordinates to block $\ell$: for the first block, $i=1,\ldots,d_0$ following Sprecher's 1-based convention for raw inputs; for subsequent blocks, $i$ ranges over the preceding block's outputs using 0-based indexing as noted in Section~\ref{sec:motivation}.)
This summed scalar-output form uses $L$ blocks and $2L$ shared spline functions in total (one pair $(\phi^{(\ell)}, \Phi^{(\ell)})$ per block). Unless explicitly stated otherwise, our scalar-output experiments use this summed form. In Section~\ref{sec:scalability} we instead use the $m=1$ vector-valued construction below, appending an additional non-summed output block with $d^{(L+1)}=1$, to align depth counting with the MLP baselines.

\paragraph{(b) Vector-valued output ($m>1$):}
When the target function $f$ maps to $\mathbb{R}^m$ with $m>1$, the network first constructs the $L$ hidden layers as above, yielding a final hidden representation $\mathbf{h}^{(L)} \in \mathbb{R}^{d_L}$. A final output block (block $L+1$) is then appended to map this representation $\mathbf{h}^{(L)}$ to the final output space $\mathbb{R}^m$. This $(L+1)$-th block operates \emph{without} a final summation over its output index. It computes the final output vector $\mathbf{y} \in \mathbb{R}^m$ as:
\begin{equation*}
\begin{aligned}
 y_q &= \Bigl(T^{(L+1)}(\mathbf{h}^{(L)})\Bigr)_q  \\
 &= \Phi^{(L+1)}\Biggl(\,\sum_{r=0}^{d_L-1} \lambda^{(L+1)}_{r}\,\phi^{(L+1)}\Bigl(\mathbf{h}^{(L)}_r+\eta^{(L+1)}\,q\Bigr)+\alpha q + \tau^{(L+1)} \sum_{j \in \mathcal{N}(q)} \omega_{q,j}^{(L+1)} s_j^{(L+1)}\Biggr) + R^{(L+1)}(\mathbf{h}^{(L)})_q, 
\end{aligned}
\end{equation*}
for $q=0,\dots,m-1$. The network output function is then:
\begin{equation}\label{eq:defsn}
f(\mathbf{x}) = \mathbf{y} = \Bigl(T^{(L+1)} \circ T^{(L)} \circ \cdots \circ T^{(1)}\Bigr)(\mathbf{x}) \in \mathbb{R}^m. 
\end{equation}
In this configuration, the network uses $L+1$ blocks and involves $2(L+1)$ shared spline functions. The extra block serves as a trainable output mapping layer, transforming the final hidden representation $\mathbf{h}^{(L)}$ into the desired $m$-dimensional output vector.

\medskip
\noindent\textbf{Summary.} For an architecture with $L$ hidden layers, the sum-aggregated scalar-output SN uses $L$ blocks and $2L$ shared splines. A vector-output SN (with $m>1$) uses $L+1$ blocks and $2(L+1)$ shared splines. This structure provides a natural extension of Sprecher's original scalar formula to the vector-valued setting. In our experiments we apply the learned affine head described in Section~\ref{sec:optional_enhancements} to the final output.
\medskip

We illustrate the vector-output case ($m>1$) for a network architecture $d_0\to[d_1,d_2,d_3]\to m$ (i.e., $L=3$ hidden layers). Let $X^{(0)}$ be the input $\mathbf{x}$.
\begin{center}
\begin{tikzpicture}[
    node distance=0.8cm and 2.2cm, % Vertical and horizontal node distances
    block/.style={font=\small}, % Style for block labels
    dim/.style={font=\footnotesize} % Style for dimension labels
  ]
    % Define nodes
    \node (X0)                                      {$X^{(0)}$};
    \node (L0) [below=0.07cm of X0, dim]             {$(\in \mathbb{R}^{d_0})$};
    \node (X1) [right=of X0]                        {$X^{(1)}$};
    \node (L1) [below=0.07cm of X1, dim]             {$(\in \mathbb{R}^{d_1})$};
    \node (X2) [right=of X1]                        {$X^{(2)}$};
    \node (L2) [below=0.07cm of X2, dim]             {$(\in \mathbb{R}^{d_2})$};
    \node (X3) [right=of X2]                        {$X^{(3)}$};
    \node (L3) [right=0.07cm of X3, dim]             {$(\in \mathbb{R}^{d_3})$};
    \node (OB) [below=of X3]                        {Output block 4 $(T^{(4)})$};
    \node (Y)  [below=of OB]                        {$\mathbf{y} \in \mathbb{R}^m$ (Network output)};

    % Draw arrows
    \draw[-{Stealth[length=2mm]}] (X0) -- node[above, block] {Block 1} (X1);
    \draw[-{Stealth[length=2mm]}] (X1) -- node[above, block] {Block 2} (X2);
    \draw[-{Stealth[length=2mm]}] (X2) -- node[above, block] {Block 3} (X3);
    \draw[-{Stealth[length=2mm]}] (X3) -- (OB); % Vertical arrow
    \draw[-{Stealth[length=2mm]}] (OB) -- (Y);  % Vertical arrow
\end{tikzpicture}
\end{center}
Here, $X^{(\ell)} = \mathbf{h}^{(\ell)}$ denotes the output vector of the $\ell$-th Sprecher block (we use both notations interchangeably for clarity in different contexts). Each block $T^{(\ell)}$ internally uses its own pair of shared splines $(\phi^{(\ell)}, \Phi^{(\ell)})$, mixing weights $\lambda^{(\ell)}$, shift $\eta^{(\ell)}$, and optionally lateral mixing parameters $(\tau^{(\ell)}, \omega^{(\ell)})$. The final output block $T^{(4)}$ maps the representation $X^{(3)}$ to the final $m$-dimensional output $\mathbf{y}$ without subsequent summation.

\subsection{Illustrative expansions (sum-aggregated scalar output)}
To further clarify the compositional structure for the sum-aggregated scalar-output case ($m=1$), we write out the full expansions for networks with $L=1, 2, 3$ hidden layers. For conciseness, we omit optional add-ons such as cyclic residual terms and normalization; when enabled, these appear additively (e.g., each block output includes an extra $R^{(\ell)}(\cdot)_q$ term as in \eqref{eq:SN_full}).

\subsubsection{Single hidden layer (\mathpdf{L=1}{L=1})}\label{sec:single_layer}
For a network with architecture $d_{\mathrm{in}}\to[d_1]\to1$ (i.e., $d_0=d_{\mathrm{in}}$), the network computes:
$$f(\mathbf{x}) = \sum_{q=0}^{d_1-1} \mathbf{h}^{(1)}_q = \sum_{q=0}^{d_1-1} \Phi^{(1)}\Biggl(\sum_{i=1}^{d_0} \lambda^{(1)}_{i}\,\phi^{(1)}\Bigl(x_i+\eta^{(1)}\,q\Bigr)+\alpha q + \tau^{(1)} \sum_{j \in \mathcal{N}(q)} \omega_{q,j}^{(1)} s_j^{(1)}\Biggr).$$
When lateral mixing is disabled ($\tau^{(1)} = 0$ or $\mathcal{N}(q) = \emptyset$), this reduces to Sprecher's 1965 construction if we choose $d_1=2d_0+1$, identify $\phi^{(1)}=\phi$, $\Phi^{(1)}=\Phi$, $\lambda^{(1)}_{i} = \lambda_i$, and set $\alpha = 1$.

\subsubsection{Two hidden layers (\mathpdf{L=2}{L=2})}
Let the architecture be $d_0\to[d_1, d_2]\to1$. The intermediate output $\mathbf{h}^{(1)} \in \mathbb{R}^{d_1}$ is computed as:
$$\mathbf{h}^{(1)}_r=\Phi^{(1)}\Bigl(\sum_{i=1}^{d_0}\lambda^{(1)}_{i}\,\phi^{(1)}\Bigl(x_i+\eta^{(1)}\,r\Bigr)+\alpha r + \tau^{(1)} \sum_{j \in \mathcal{N}(r)} \omega_{r,j}^{(1)} s_j^{(1)}\Bigr),\quad r=0,\dots,d_1-1.$$
The second block computes $\mathbf{h}^{(2)} \in \mathbb{R}^{d_2}$ using $\mathbf{h}^{(1)}$ as input:
$$\mathbf{h}^{(2)}_q=\Phi^{(2)}\Bigl(\sum_{r=0}^{d_1-1}\lambda^{(2)}_{r}\,\phi^{(2)}\Bigl(\mathbf{h}^{(1)}_r+\eta^{(2)}\,q\Bigr)+\alpha q + \tau^{(2)} \sum_{j \in \mathcal{N}(q)} \omega_{q,j}^{(2)} s_j^{(2)}\Bigr),\quad q=0,\dots,d_2-1.$$
The final network output is the sum over the components of $\mathbf{h}^{(2)}$: $f(\mathbf{x})=\sum_{q=0}^{d_2-1}\mathbf{h}^{(2)}_q$.

For the base architecture without optional enhancements, the fully expanded form reveals the nested compositional structure:
\begin{equation}\label{eq:two-layer-expanded}
f(\mathbf{x})=\sum_{q=0}^{d_2-1}\Phi^{(2)}\Biggl(\sum_{r=0}^{d_1-1}\lambda^{(2)}_{r}\,\phi^{(2)}\Biggl(\underbrace{\Phi^{(1)}\Biggl(\sum_{i=1}^{d_0}\lambda^{(1)}_{i}\,\phi^{(1)}\Bigl(x_i+\eta^{(1)}\,r\Bigr)+\alpha r\Biggr)}_{\text{Output of first block}}+\eta^{(2)}\,q\Biggr)+\alpha q\Biggr).
\end{equation}
This nested structure, where the output of $\Phi^{(1)}$ becomes the input to $\phi^{(2)}$, fundamentally departs from Sprecher's original construction and represents our empirically-motivated deep architecture.

\begin{remark}[Lateral mixing in expanded form]
When lateral mixing is enabled, the expanded forms become more complex. For instance, in the two-layer case, each $s_r^{(1)}$ term undergoes mixing before the $\Phi^{(1)}$ transformation, and similarly for $s_q^{(2)}$ before $\Phi^{(2)}$. We omit the full expansion for clarity, but note that lateral mixing represents an additional source of inter-dimensional coupling beyond the shifts, potentially enhancing the network's ability to capture correlations between output dimensions.
\end{remark}

\subsubsection{Three hidden layers (\mathpdf{L=3}{L=3})}
Let the architecture be $d_0\to [d_1, d_2, d_3]\to1$. The recursive definition involves:
$$\begin{aligned}
\mathbf{h}^{(1)}_r &= \Phi^{(1)}\Bigl(\sum_{i=1}^{d_0}\lambda^{(1)}_{i}\,\phi^{(1)}\Bigl(x_i+\eta^{(1)}\,r\Bigr)+\alpha r + \tau^{(1)} \sum_{j \in \mathcal{N}(r)} \omega_{r,j}^{(1)} s_j^{(1)}\Bigr),\quad r=0,\dots,d_1-1,\\[1mm]
\mathbf{h}^{(2)}_p &= \Phi^{(2)}\Bigl(\sum_{r=0}^{d_1-1}\lambda^{(2)}_{r}\,\phi^{(2)}\Bigl(\mathbf{h}^{(1)}_r+\eta^{(2)}\,p\Bigr)+\alpha p + \tau^{(2)} \sum_{j \in \mathcal{N}(p)} \omega_{p,j}^{(2)} s_j^{(2)}\Bigr),\quad p=0,\dots,d_2-1,\\[1mm]
\mathbf{h}^{(3)}_q &= \Phi^{(3)}\Bigl(\sum_{p=0}^{d_2-1}\lambda^{(3)}_{p}\,\phi^{(3)}\Bigl(\mathbf{h}^{(2)}_p+\eta^{(3)}\,q\Bigr)+\alpha q + \tau^{(3)} \sum_{j \in \mathcal{N}(q)} \omega_{q,j}^{(3)} s_j^{(3)}\Bigr),\quad q=0,\dots,d_3-1.
\end{aligned}$$
The network output is $f(\mathbf{x})=\sum_{q=0}^{d_3-1}\mathbf{h}^{(3)}_q$. 

\begin{remark}
These expansions highlight the compositional nature where the output of one Sprecher block, which is a vector of transformed values, serves as the input to the next. Each transformation layer involves its own pair of shared splines, learnable parameters, and optionally lateral mixing connections.
\end{remark}

\begin{remark}[Necessity of internal shifts]\label{rem:shiftneeded}
It is tempting to simplify the nested structures, for instance by removing the inner shift terms like $\eta^{(2)}q$ inside $\phi^{(2)}$. One might hypothesize that the outer splines $\Phi^{(\ell)}$ could absorb this shifting effect (yielding a single composite spline per Sprecher block). However, experiments suggest that these internal shifts $\eta^{(\ell)}q$ applied to the inputs of the $\phi^{(\ell)}$ splines are crucial for the effective functioning of deeper Sprecher Networks. Removing them significantly degrades performance. The precise theoretical reason for their necessity in the multi-layer case, beyond their presence in Sprecher's original single-layer formula, warrants further investigation.
\end{remark}

\section{Comparison with related architectures}
To position Sprecher Networks accurately, we compare their core architectural features with Multi-Layer Perceptrons (MLPs), networks with learnable node activations (LANs/Adaptive-MLPs), and Kolmogorov-Arnold Networks (KANs).

\begin{table}[ht]
\centering
\caption{Architectural comparison of neural network families.}
\label{tab:arch_comparison}
\small % Use smaller font size if needed
\begin{tabular}{@{}lllll@{}}
\toprule
Feature                   & MLP                 & LAN / Adaptive-MLP & KAN                   & Sprecher Network (SN) \\ \midrule
\textbf{Learnable}        & Linear Weights      & Linear Weights     & Edge Splines          & Block Splines ($\phi, \Phi$) \\
\textbf{Components}       & (on edges)          & + Node Activations &                       & + Mixing Weights ($\lambda$) \\
                          &                     &                    &                       & + Shift Parameter ($\eta$) \\
                          &                     &                    &                       & + Lateral Mixing ($\tau, \omega$) \\
\textbf{Fixed components}   & Node Activations    & ---                & Node Summation        & Node Summation (implicit) \\
                          &                     &                    &                       & + Fixed Shifts ($+\alpha q$) \\
\textbf{Location of}      & Nodes               & Nodes              & Edges                 & Blocks \\
\textbf{Non-linearity}    & (Fixed)             & (Learnable)        & (Learnable)           & (Shared, Learnable) \\
\textbf{Node operation}     & Apply $\sigma(\cdot)$ & Apply $\sigma_{\text{learn}}(\cdot)$ & $\sum (\text{inputs})$ & Implicit in Block Formula \\
\textbf{Parameter sharing}  & None (typically)    & Activations (optional) & None (typically)      & Splines ($\phi, \Phi$) per block \\
\textbf{Intra-layer mixing} & None & None & None & Lateral (cyclic/bidir.) \\
                           &      &      &      & $O(N)$ params \\
\textbf{Residual design}    & Matrix projection   & Matrix projection  & None (typically)      & Cyclic \\
                          & $O(N^2)$ per skip   & $O(N^2)$ per skip  &                       & $O(N)$ per skip \\
\textbf{Theoretical basis}  & UAT                 & UAT                & KAT (inspired)        & Sprecher (direct) \\
\textbf{Param scaling}    & $O(L N^2)$          & $O(L N^2 + L N G)$ & $O(L N^2 G)$          & $O(L N + L G)$ \\
                          &                     & (Approx.)          &                       & (Approx.) \\
\textbf{Memory complexity} & $O(L N^2)$          & $O(L N^2 + L N G)$ & $O(L N^2 G)$          & $O(L N^2)$ parallel \\
                          &                     &                    &                       & $\mathbf{O(L N + L G)}$ sequential \\
\bottomrule
\end{tabular}
\vspace{2mm}
\parbox{\textwidth}{\footnotesize \textit{Notes:} $L$=depth, $N$=average width, $G$=spline grid size/complexity. UAT=Universal Approx.\ Theorem, KAT=Kolmogorov-Arnold Theorem. LAN details often follow KAN Appendix B \cite{liu2024kan}. Forward working memory refers to peak \emph{forward-pass} intermediate storage plus parameters (excluding optimizer state), treating batch size as a constant; during training, reverse-mode autodiff may store additional activations unless checkpointing is used. The ``sequential'' mode for SNs processes output dimensions iteratively rather than in parallel, maintaining mathematical equivalence while reducing memory usage. All architectures may optionally include normalization layers; SNs apply normalization to entire block outputs rather than individual activations. For SNs, the $O(LN+LG)$ parameter scaling assumes the cyclic/node residual; using a full linear residual projection adds an additional $O(\sum_{\ell} d_{\ell-1}d_{\ell})$ parameters.
\textit{The parameter scaling notation uses $N$ to denote a typical or average layer width for simplicity, following \cite{liu2024kan}. For architectures with varying widths $d_\ell$, the $LN^2$ terms should be understood as $\sum_{\ell} d_{\ell-1}d_{\ell}$ (MLP, LAN), the $LN^2G$ term for KAN as $(\sum_{\ell} d_{\ell-1}d_{\ell})G$, and the $LN$ term for SN as $\sum_{\ell} (d_{\ell-1}+d_{\ell})$ (since SN uses weight vectors, not matrices), where the sum is over the relevant blocks/layers, for precise counts.}}
\end{table}

Table \ref{tab:arch_comparison} summarizes the key distinctions between these architectures. MLPs learn edge weights with fixed node activations, LANs add learnable node activations to this structure, KANs move learnability entirely to edge splines while eliminating linear weights, and SNs concentrate learnability in shared block-level splines, block-level shifts, mixing weights, and optionally lateral mixing connections. The critical difference for SNs is their use of weight vectors rather than matrices, which fundamentally reduces the dependence on width from quadratic to linear. This architectural choice can be understood through an analogy with convolutional neural networks: just as CNNs achieve parameter efficiency and improved generalization by sharing weights across spatial locations, SNs share weights across output dimensions within each block. In CNNs, spatial shifts provide the necessary diversity despite weight sharing; in SNs, the shifts $\eta q$ and the additive term $+\alpha q$, along with lateral mixing when enabled, play this diversifying role. This perspective reframes our architectural constraint not as a limitation, but as a principled form of weight sharing motivated by Sprecher's theorem, suggesting that SNs might be viewed as a ``convolutional'' approach to function approximation networks. While this weight sharing is theoretically justified for single-layer networks, its effectiveness in deep compositions remains an empirical finding that warrants further theoretical investigation. This combination of choices leads to SNs' distinctive parameter scaling of $O(L N + L G)$ compared to KANs' $O(L N^2 G)$ and, for fixed spline resolution $G$, parameter-memory scaling of $O(LN)$.\footnote{The $O(LN)$ term includes contributions from input weights ($\lambda$), normalization parameters when used, lateral mixing parameters, and (when used) \emph{cyclic} residual connections. Cyclic residuals contribute at most $\max(d_{\mathrm{in}}, d_{\mathrm{out}})$ parameters per block; a learned linear residual projection instead adds $d_{\mathrm{in}}d_{\mathrm{out}}$ parameters when $d_{\mathrm{in}}\neq d_{\mathrm{out}}$ (and only a single scalar gate when $d_{\mathrm{in}}=d_{\mathrm{out}}$), so it is quadratic in width only on dimension-changing blocks.}

Here, we provide a precise comparison between LANs and SNs. While the following proposition shows that SNs can be expressed as special cases of LANs with specific structural constraints, it is important to note that Sprecher's construction guarantees that this constrained form retains full expressivity in the single-layer case. This suggests that the extreme parameter sharing and structural constraints in SNs may serve as a beneficial inductive bias rather than a limitation.

\begin{definition}
A LAN is an MLP with learnable activation. More precisely, the model is defined as: 
$$
f(\mathbf{x}) = A^{(L)} \circ \sigma^{(L-1)} \circ A^{(L-1)} \circ \sigma^{(L-2)} \circ \cdots \circ \sigma^{(1)} \circ A^{(1)} (\mathbf{x}),
$$
where $A^{(k)} \colon \mathbb{R}^{d_{k-1}} \rightarrow \mathbb{R}^{d_k}$ is an affine map, and $\sigma^{(k)} \colon \mathbb{R} \rightarrow \mathbb{R}  $ is the activation function (applied coordinate-wise). In an MLP, the trainable parameters are the weights $W^{(k)}$ and biases $b^{(k)}$ of $A^{(k)}(\mathbf{x}) = W^{(k)}\mathbf{x} + b^{(k)}$ for $k = 1, \ldots, L$. In a LAN, $\sigma$ contains additional trainable parameters, e.g., the coefficients of a spline.
\end{definition}

\begin{remark}[Note on weight structure]
The following proposition uses matrix weights $\lambda_{i,q}^{(\ell)}$ to establish the connection between SNs and LANs. However, the architecture we propose and analyze throughout this paper uses vector weights $\lambda_i^{(\ell)}$ (following Sprecher's original formulation), where the same weights are shared across all output dimensions. This vector version represents an even more constrained special case of the LAN formulation shown below. The practical SN architecture with vector weights corresponds to the constraint that these matrix weights are constant across the output index, i.e., $\lambda_{i,q}^{(\ell)} = \lambda_i^{(\ell)}$ for all $q$. This constraint is fundamental to maintaining the ``true Sprecher'' structure and achieving the characteristic $O(LN)$ parameter scaling. Additionally, when lateral mixing is enabled, the structure becomes more complex, requiring additional terms in the LAN representation to capture the intra-block communication.
\end{remark}

\begin{prop} \label{prop:SNareLAN}
    A matrix-weighted variant of a Sprecher Network (i.e., with per-output mixing weights $\lambda^{(\ell)}_{i,q}$ in each block) and with lateral mixing disabled is a LAN, where:
    \begin{itemize}
        \item in odd layers $k = 2\ell - 1$, the weight matrix $W^{(k)} \in \mathbb{R}^{d_\ell d_{\ell-1} \times d_{\ell-1}}$ is fixed to  $[I | \cdots | I]^\top$, where $I$ is the $d_{\ell-1} \times d_{\ell-1}$ identity matrix, the bias vector has only one learnable parameter $\eta^{(\ell)}$ and is structured as $b^{(k)} = \eta^{(\ell)}(0,\ldots,0, 1, \ldots, 1, \ldots, d_{\ell}-1, \ldots, d_{\ell}-1)^\top \in \mathbb{R}^{d_\ell d_{\ell-1}}$, and the activation is $\sigma^{(k)} = \phi^{(\ell)}$, 
        \item in even layers $k = 2\ell $, the learnable weight matrix $W^{(k)} \in \mathbb{R}^{d_\ell \times d_\ell d_{\ell-1}}$ is structured as
        \begin{align*}
            \begin{bmatrix}
                \lambda_{1,0}^{(\ell)} & \cdots & \lambda_{d_{\ell-1},0}^{(\ell)} & 0 &&& \cdots &&&0 \\
               0  &\cdots & 0 & \lambda_{1,1}^{(\ell)} & \cdots & \lambda_{d_{\ell-1},1}^{(\ell)} & 0 && \cdots & 0 \\
               &&&&&& \ddots \\
               0 &&&\cdots &&&0&\lambda_{1,d_\ell-1}^{(\ell)} & \cdots & \lambda_{d_{\ell-1},d_\ell-1}^{(\ell)}
            \end{bmatrix},
        \end{align*} 
        the bias is fixed to $b^{(k)} = \alpha(0,\ldots, d_\ell-1)^\top \in \mathbb{R}^{d_\ell}$, and the activation is $\sigma^{(k)} = \Phi^{(\ell)}$. 
    \end{itemize}
\end{prop}
\begin{proof}
    Follows immediately by inspecting \eqref{eq:SN} with $\tau^{(\ell)} = 0$ and $R^{(\ell)}\equiv 0$ (no lateral mixing or residual path).
\end{proof}

\begin{remark}[Understanding the LAN representation]
The representation of SNs as LANs in Proposition~\ref{prop:SNareLAN} uses an expanded 
intermediate state space. Each Sprecher block is decomposed into two LAN layers:
\begin{itemize}
\item The first layer expands the input $\mathbf{h}^{(\ell-1)} \in \mathbb{R}^{d_{\ell-1}}$ 
to $\mathbb{R}^{d_\ell d_{\ell-1}}$ by creating $d_\ell$ shifted copies, where the 
$(q \cdot d_{\ell-1} + i)$-th component contains $\phi^{(\ell)}(h^{(\ell-1)}_i + \eta^{(\ell)} q)$.
\item The second layer applies the (tied) mixing weights $\lambda^{(\ell)}_{i}$ (shared across all $q$) to select and sum 
the appropriate components for each output $q$, adds the shift $\alpha q$, and applies $\Phi^{(\ell)}$.
\end{itemize}
This construction shows that while SNs can be expressed within the LAN framework, they 
represent a highly structured special case with specific weight patterns and an expanded 
intermediate dimension of $O(d_{\ell-1} d_\ell)$ between each pair of SN layers. The inclusion of lateral mixing would further complicate this representation, requiring additional structure to capture the intra-block communication.
\end{remark}

While this proposition shows that SNs are special cases of LANs with specific structural constraints, Sprecher's theorem guarantees that \emph{in the single-layer setting} this highly structured shift-and-sum form is still universal on $[0,1]^n$. In Sprecher's original construction, $\eta$ can be chosen as a universal constant rather than a learnable parameter. Thus, universality can already hold in the presence of strongly structured, index-dependent bias patterns; Proposition \ref{prop:SNareLAN} makes this connection explicit by exhibiting one concrete structured realization.

\begin{figure}[h!]
    \centering  
\begin{tikzcd}[ampersand replacement=\&]
    \& \textnormal{LAN} \\
    \textnormal{MLP} \&\& \textnormal{SN}
    \arrow[hook, from=2-1, to=1-2]
    \arrow[hook', from=2-3, to=1-2]
\end{tikzcd}
    \caption{Diagram illustrating the dependencies between the models, in terms of learnable parameters. MLPs are LANs with fixed activation function, while SNs are LANs with a particular parameter structure (Proposition \ref{prop:SNareLAN}).}
    \label{fig:model_dependencies}
\end{figure}
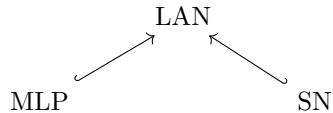

\begin{remark}[Domain considerations]
The above proposition assumes that the spline functions $\phi^{(\ell)}$ can handle inputs outside their original domain $[0,1]$, which may arise due to the shifts $\eta^{(\ell)}q$. While Sprecher's theoretical construction extends $\phi$ appropriately, practical implementations typically compute exact theoretical bounds for all spline domains based on the network structure, and use these bounds to update spline knot ranges; any remaining out-of-range queries are handled by the splines' explicit extension/extrapolation rules (constant extension for $\phi^{(\ell)}$ and linear extrapolation for $\Phi^{(\ell)}$).
\end{remark}

\section{Parameter counting and efficiency as a trade-off}
A key consequence of adhering to the vector-based weighting scheme inspired by Sprecher's formula is a dramatic reduction in parameters compared to standard architectures. This represents a strong architectural constraint that may serve as either a beneficial inductive bias or a limitation, depending on the target function class. The specific design of SNs, particularly the sharing of splines and use of weight vectors rather than matrices, leads to a distinctive parameter scaling that warrants careful analysis.

Let's assume a network of depth $L$ (meaning $L$ hidden layers, thus $L$ blocks for scalar output or $L+1$ for vector output), with an average layer width $N$. We denote the input dimension as $N_{in}$ when it differs significantly from the hidden layer widths. Let $G$ be the number of knot locations (equivalently, the number of learnable coefficients) used to parameterize each one-dimensional spline; a piecewise-linear spline then has $G-1$ intervals. For simplicity, we approximate the number of parameters per spline as $O(G)$.

The parameter counts for different architectures scale as follows. MLPs primarily consist of linear weight matrices, leading to a total parameter count dominated by these weights, scaling as $O(L N^2)$. LANs (Adaptive-MLPs) have both linear weights ($O(L N^2)$) and learnable activations; if each of the $N$ nodes per layer has a learnable spline, this adds $O(L N G)$ parameters, for a total of $O(L N^2 + L N G)$. KANs replace linear weights with learnable edge splines, with $O(N^2)$ edges between layers. If each edge has a spline with $O(G)$ parameters, the total count per layer is $O(N^2 G)$, leading to an overall scaling of $O(L N^2 G)$.

Sprecher Networks have a fundamentally different structure. While a block can be viewed as densely connecting inputs to outputs computationally (as in Figure~\ref{fig:block_flow}), SNs do \emph{not} assign learnable parameters to individual edges: all input--output interactions share the same $\phi^{(\ell)}$ (and the same $\Phi^{(\ell)}$ after aggregation), with diversity introduced via the index shift and shared vector weights. Each block $\ell$ contains mixing weights $\lambda^{(\ell)}$ with $O(d_{\ell-1})$ parameters, where $d_{\ell-1}$ is the input dimension to that block (crucially, these are vectors, not matrices), shared splines $\phi^{(\ell)}, \Phi^{(\ell)}$ with $2 \times O(G) = O(G)$ parameters per block independent of $N$, and a shift parameter $\eta^{(\ell)}$ with $O(1)$ parameter per block. When lateral mixing is enabled, each block additionally contains a lateral scale parameter $\tau^{(\ell)}$ with $O(1)$ parameter and mixing weights $\omega^{(\ell)}$ with $O(d_\ell)$ parameters (specifically $d_\ell$ for cyclic, $2d_\ell$ for bidirectional). When normalization is used, each block may have an associated normalization layer adding $O(N)$ parameters (specifically $2N$ for the affine transformation). When residual connections are used, cyclic residuals add at most $O(N)$ parameters per block (specifically $\max(d_{\ell-1}, d_\ell)$ for block $\ell$, or just 1 when dimensions match); a learned linear residual projection instead adds $O(d_{\ell-1}d_\ell)$ parameters when $d_{\ell-1}\neq d_\ell$ (and just 1 when $d_{\ell-1}=d_\ell$). Summing over $L$ (or $L+1$) blocks, the total parameter count scales approximately as $O(LN + LG + L)$, where the linear dependence on $N$ now includes mixing weights, optional normalization parameters, lateral mixing weights, and residual connections.

Beyond parameter efficiency, SNs also admit memory-efficient computation strategies. While standard vectorized implementations materialize a $(B \cdot d_{\mathrm{in}} \cdot d_{\mathrm{out}})$ tensor of shifted inputs, a sequential evaluation strategy avoids this materialization and reduces the largest temporary forward tensor to $O(B \cdot \max(d_{\mathrm{in}}, d_{\mathrm{out}}))$ (Section~\ref{sec:memory_efficient}). In wide settings where $d_{\mathrm{in}}\approx d_{\mathrm{out}}\approx N$, this yields a peak temporary forward footprint that scales linearly in $N$ rather than quadratically. (As usual, achieving this scaling during training may require checkpointing/recomputation, since the backward pass can retain additional intermediates.)

This scaling reveals the crucial trade-off: SNs achieve a reduction from $O(LN^2)$ to $O(LN)$ in parameter count (and corresponding parameter memory) by using weight vectors rather than matrices; with sequential evaluation (Section~\ref{sec:memory_efficient}) they also avoid explicitly materializing the $B\times d_{\mathrm{in}}\times d_{\mathrm{out}}$ shifted-input tensor in the forward pass. Additionally, in KANs, the spline complexity $G$ multiplies the $N^2$ term ($O(L N^2 G)$), while in SNs, due to spline sharing within blocks, it appears as an additive term ($O(L G)$). This suggests potential for significant parameter and memory savings, particularly when high spline resolution ($G$) is required for accuracy or when the layer width ($N$) is moderate to large.

This extreme parameter sharing represents a fundamental architectural bet: that the structure imposed by Sprecher's theorem, using only weight vectors with output diversity through shifts and lateral mixing, provides a beneficial inductive bias that outweighs the reduction in parameter flexibility. The empirical observation that training remains feasible (albeit sometimes requiring more iterations) with vector weights suggests this bet may be justified for certain function classes. Moreover, viewing this constraint through the lens of weight sharing in CNNs provides a new perspective: both architectures sacrifice parameter flexibility for a structured representation that has proven effective in practice, though for SNs the theoretical justification comes from Sprecher's theorem rather than domain-specific intuitions about spatial invariance. The addition of lateral mixing provides a middle ground, allowing some cross-dimensional communication while maintaining the overall parameter efficiency. Whether this constraint serves as beneficial regularization or harmful limitation likely depends on the specific problem domain and the alignment between the target function's structure and the inductive bias imposed by the SN architecture.

\section{Theoretical aspects and open questions}\label{sec:universality}

\subsection{Relation to Sprecher (1965) and universality}
As shown in Section \ref{sec:single_layer}, a single-layer ($L=1$) Sprecher Network (SN) reduces exactly to Sprecher's constructive representation \eqref{eq:sprecher_original} when lateral mixing and residuals are disabled (i.e., $\tau=0$ and no residual connection), with the additive shift implemented by $\alpha q$ (set $\alpha=1$ to match \eqref{eq:sprecher_original}). In his 1965 paper, Sprecher proved that for any continuous function $f:[0,1]^n\to\mathbb{R}$ there exist a \emph{single} nondecreasing univariate function $\phi$, a continuous univariate function $\Phi$, real constants $\lambda_1,\ldots,\lambda_n$ and $\eta$, and an integer upper index $2n$ such that the representation \eqref{eq:sprecher_original} holds \cite{sprecher1965}. This immediately implies:

\begin{theorem}[Universality of single-layer SNs]\label{thm:ua_single_layer}
For any $n\ge 1$ and any $f\in C([0,1]^n)$, there exists a single-layer SN with $\tau=0$ and no residual connection that represents $f$ exactly in the form \eqref{eq:sprecher_original} (taking $d_{\mathrm{out}}=2n+1$ and $\alpha=1$).
\end{theorem}

Thus, single-layer SNs inherit (indeed, contain) the family given by Sprecher's representation. In practice, however, we do not fit arbitrary $C^2$ univariate functions $\phi$ and $\Phi$ exactly; we parameterize them as splines. This motivates the following finite-parameter approximation analysis, which also clarifies how optional \emph{lateral mixing} (controlled by $\tau$ and $\omega$) and cyclic residuals affect constants but not rates.

\medskip
\noindent\textbf{Notation for a single block.}
A (possibly mixed) Sprecher block $T:\mathbb{R}^{d_{\mathrm{in}}}\to\mathbb{R}^{d_{\mathrm{out}}}$ acts componentwise as
$$[T(\mathbf{z})]_q \;=\; \Phi\!\Biggl(\,\underbrace{\sum_{i=1}^{d_{\mathrm{in}}}\lambda_i\,\phi(z_i+\eta q)+\alpha q}_{=:s_q}\;+\;\tau\sum_{j\in\mathcal{N}(q)}\omega_{qj}\,s_j\Biggr),
\quad q=0,\ldots,d_{\mathrm{out}}-1,$$
where $\lambda\in\mathbb{R}^{d_{\mathrm{in}}}$, $\eta,\alpha,\tau\in\mathbb{R}$, $\mathcal{N}(q)$ specifies the (optional) neighborhood used for lateral mixing, and $\omega=(\omega_{qj})$ are the mixing weights. When a cyclic residual connection is present, the block output is $T_{\mathrm{res}}(\mathbf{z})=T(\mathbf{z})+R(\mathbf{z})$. The approximation analysis below applies to $T$; when the residual map $R$ is shared between the target block and its spline approximation, it cancels in the blockwise difference. However, Lipschitz/composition constants for the full block include an additive contribution from $R$ (e.g.\ $L_{T_{\mathrm{res}}}\le L_T+L_R$).

\begin{lemma}[Single block approximation with piecewise-linear splines]\label{lem:single_block}
Fix $B_{\mathrm{in}}>0$ and assume $\mathbf{z}\in\mathcal{B}_{\mathrm{in}}:=[-B_{\mathrm{in}},B_{\mathrm{in}}]^{d_{\mathrm{in}}}$. Let
$$I_\phi := \bigl[-B_{\mathrm{in}}-|\eta|(d_{\mathrm{out}}-1),\; B_{\mathrm{in}}+|\eta|(d_{\mathrm{out}}-1)\bigr],
\quad
M_\phi := \|\phi\|_{L^\infty(I_\phi)}.$$
Define
$$B_\omega:=\sup_{q}\sum_{j\in\mathcal{N}(q)}|\omega_{qj}|,
\qquad
R_{\mathrm{mix}}:=1+|\tau|\,B_\omega,
\qquad
M_s := \|\lambda\|_{1}\,M_\phi + |\alpha|\,(d_{\mathrm{out}}-1),$$
and set $I_\Phi := [-R_{\mathrm{mix}}M_s,\; R_{\mathrm{mix}}M_s]$. Assume $\phi,\Phi\in C^2$ on neighborhoods of $I_\phi$ and $I_\Phi$, respectively. Let $\hat\phi$ and $\hat\Phi$ be the (shape-preserving) piecewise-linear interpolants of $\phi$ on a uniform grid of $G_\phi\ge 2$ knots over $I_\phi$ and of $\Phi$ on a uniform grid of $G_\Phi\ge 2$ knots over $I_\Phi$; write
$$h_\phi:=\frac{|I_\phi|}{G_\phi-1},\quad h_\Phi:=\frac{|I_\Phi|}{G_\Phi-1},\quad
M_{\phi''}:=\|\phi''\|_{L^\infty(I_\phi)},\quad M_{\Phi''}:=\|\Phi''\|_{L^\infty(I_\Phi)}.$$
Define $\hat T$ by replacing $(\phi,\Phi)$ with $(\hat\phi,\hat\Phi)$ in $T$ (keeping the same $\lambda,\eta,\alpha,\tau,\omega,\mathcal{N}$). Then, for all $\mathbf{z}\in\mathcal{B}_{\mathrm{in}}$,
\begin{equation}\label{eq:single_block_sup}
\|T(\mathbf{z})-\hat T(\mathbf{z})\|_\infty
\;\le\;
L_\Phi\,R_{\mathrm{mix}}\,\|\lambda\|_{1}\,\delta_\phi\;+\;\delta_\Phi,
\qquad
\delta_\phi:=\frac{M_{\phi''}}{8}h_\phi^2,\ \ \delta_\Phi:=\frac{M_{\Phi''}}{8}h_\Phi^2,
\end{equation}
where $L_\Phi:=\|\Phi'\|_{L^\infty(I_\Phi)}$. Equivalently,
$$\|T(\mathbf{z})-\hat T(\mathbf{z})\|_\infty
\;\le\;
K_T\cdot \max\{h_\phi^2,h_\Phi^2\}
\quad\text{with}\quad
K_T:=\frac{1}{8}\Bigl(L_\Phi\,R_{\mathrm{mix}}\,\|\lambda\|_{1}\,M_{\phi''}+M_{\Phi''}\Bigr).$$
The same bounds hold for a block with a (shared) cyclic residual $T_{\mathrm{res}}=T+R$, since $R$ cancels when $T$ and $\hat T$ share the same residual path.
\end{lemma}

\begin{proof}
On each subinterval of the uniform grid for $I_\phi$ (resp.\ $I_\Phi$) of length $h_\phi$ (resp.\ $h_\Phi$), the classical bound for piecewise-linear interpolation of a $C^2$ function gives
$$\|\phi-\hat\phi\|_{L^\infty(I_\phi)}\le \frac{M_{\phi''}}{8}h_\phi^2=:\delta_\phi,
\qquad
\|\Phi-\hat\Phi\|_{L^\infty(I_\Phi)}\le \frac{M_{\Phi''}}{8}h_\Phi^2=:\delta_\Phi.$$
For fixed $\mathbf{z}$ and $q$, write $s_q=\sum_i\lambda_i\,\phi(z_i+\eta q)+\alpha q$ and $\hat s_q$ the analogous quantity with $\hat\phi$. Then
$$|s_q-\hat s_q|\;\le\;\sum_{i=1}^{d_{\mathrm{in}}}|\lambda_i|\,\|\phi-\hat\phi\|_{L^\infty(I_\phi)}
\;\le\;\|\lambda\|_{1}\,\delta_\phi.$$
Define the post-mixing scalars $\tilde s_q:=s_q+\tau\sum_{j\in\mathcal{N}(q)}\omega_{qj}s_j$ and $\hat{\tilde s}_q:=\hat s_q+\tau\sum_{j\in\mathcal{N}(q)}\omega_{qj}\hat s_j$. Using the triangle inequality and the definition of $B_\omega$,
$$|\tilde s_q-\hat{\tilde s}_q|
\;\le\;
|s_q-\hat s_q| + |\tau| \sum_{j\in\mathcal{N}(q)}|\omega_{qj}|\,|s_j-\hat s_j|
\;\le\; R_{\mathrm{mix}}\cdot \sup_{j}|s_j-\hat s_j|
\;\le\; R_{\mathrm{mix}}\,\|\lambda\|_{1}\,\delta_\phi.$$
Next, since a linear interpolant lies in the convex hull of the function values at the grid endpoints, we have
$$\|\hat\phi\|_{L^\infty(I_\phi)} \;\le\; \|\phi\|_{L^\infty(I_\phi)}=M_\phi.$$
Hence $|s_j|\le M_s$ and $|\hat s_j|\le M_s$, which implies $|\tilde s_q|\le R_{\mathrm{mix}}M_s$ and $|\hat{\tilde s}_q|\le R_{\mathrm{mix}}M_s$. Therefore $\tilde s_q,\hat{\tilde s}_q\in I_\Phi$, and
$$|T(\mathbf{z})_q-\hat T(\mathbf{z})_q|
=\bigl|\Phi(\tilde s_q)-\hat\Phi(\hat{\tilde s}_q)\bigr|
\le \underbrace{|\Phi(\tilde s_q)-\Phi(\hat{\tilde s}_q)|}_{\le L_\Phi |\tilde s_q-\hat{\tilde s}_q|}
+ \underbrace{|\Phi(\hat{\tilde s}_q)-\hat\Phi(\hat{\tilde s}_q)|}_{\le \delta_\Phi}.$$
Combining the bounds and taking the maximum over $q$ yields \eqref{eq:single_block_sup}. If a residual path $R$ is present and shared between $T$ and $\hat T$, then $T_{\mathrm{res}}-\hat T_{\mathrm{res}}=(T-\hat T)+(R-R)=(T-\hat T)$, so the same inequalities apply.
\end{proof}

\begin{lemma}[Lipschitz constant of a Sprecher block]\label{lem:block_lip}
Assume $\phi$ and $\Phi$ are Lipschitz on $I_\phi$ and $I_\Phi$ with Lipschitz constants $L_\phi$ and $L_\Phi$, respectively. Then $T$ is $L_T$-Lipschitz on $\mathcal{B}_{\mathrm{in}}$ in the $\ell_\infty$ norm with
$$L_T \;\le\; L_\Phi\,R_{\mathrm{mix}}\,\|\lambda\|_1\,L_\phi,$$
where $R_{\mathrm{mix}}=1+|\tau|\,B_\omega$ as in Lemma~\ref{lem:single_block}. If a residual path $R$ (shared between the exact and approximated networks) is present with Lipschitz constant $L_R$ on $\mathcal{B}_{\mathrm{in}}$, then $T_{\mathrm{res}}=T+R$ is $(L_T+L_R)$-Lipschitz.
\end{lemma}

\begin{proof}
For $\mathbf{z},\mathbf{z}'\in\mathcal{B}_{\mathrm{in}}$ and each $q$,
$$\bigl|s_q(\mathbf{z})-s_q(\mathbf{z}')\bigr|
\le \sum_i |\lambda_i|\,\bigl|\phi(z_i+\eta q)-\phi(z_i'+\eta q)\bigr|
\le \|\lambda\|_1\,L_\phi\,\|\mathbf{z}-\mathbf{z}'\|_\infty.$$
Therefore,
$$\bigl|\tilde s_q(\mathbf{z})-\tilde s_q(\mathbf{z}')\bigr|
\le \Bigl(1+|\tau|\,\sum_{j\in\mathcal{N}(q)}|\omega_{qj}|\Bigr)\,\|\lambda\|_1\,L_\phi\,\|\mathbf{z}-\mathbf{z}'\|_\infty
\le R_{\mathrm{mix}}\|\lambda\|_1 L_\phi\,\|\mathbf{z}-\mathbf{z}'\|_\infty,$$
and hence
$$|T(\mathbf{z})_q-T(\mathbf{z}')_q|
\le L_\Phi\,|\tilde s_q(\mathbf{z})-\tilde s_q(\mathbf{z}')|
\le L_\Phi R_{\mathrm{mix}}\|\lambda\|_1 L_\phi\,\|\mathbf{z}-\mathbf{z}'\|_\infty.$$
Taking the maximum over $q$ gives the claimed bound for $T$, and adding $L_R$ handles $T_{\mathrm{res}}$ by the triangle inequality.
\end{proof}

\begin{lemma}[Error composition]\label{lem:error_composition}
Consider an $L$-block SN (with optional lateral mixing and cyclic residuals at each block). Let $T^{(\ell)}:\mathbb{R}^{d_{\ell-1}}\to\mathbb{R}^{d_{\ell}}$ denote the $\ell$-th block map on the exact network and $\hat T^{(\ell)}$ its spline-approximated counterpart (same $\lambda,\eta,\alpha,\tau,\omega$ and (when present) the residual parameters, but $(\phi,\Phi)$ replaced by $(\hat\phi,\hat\Phi)$). Assume there exist bounded sets $\mathcal{B}_{\ell}\subset\mathbb{R}^{d_{\ell}}$ such that for all inputs $\mathbf{x}\in[0,1]^n$,
$$\mathbf{h}^{(\ell)}(\mathbf{x}) := T^{(\ell)}\!\circ\cdots\circ T^{(1)}(\mathbf{x}) \in \mathcal{B}_{\ell}
\quad\text{and}\quad
\hat{\mathbf{h}}^{(\ell)}(\mathbf{x}) := \hat T^{(\ell)}\!\circ\cdots\circ \hat T^{(1)}(\mathbf{x}) \in \mathcal{B}_{\ell},$$
and suppose each $T^{(\ell)}$ is $L_{T^{(\ell)}}$-Lipschitz on $\mathcal{B}_{\ell-1}$. If $\varepsilon_\ell:=\sup_{\mathbf{z}\in\mathcal{B}_{\ell-1}}\|T^{(\ell)}(\mathbf{z})-\hat T^{(\ell)}(\mathbf{z})\|_\infty$, then for $E_\ell:=\sup_{\mathbf{x}\in[0,1]^n}\|\mathbf{h}^{(\ell)}(\mathbf{x})-\hat{\mathbf{h}}^{(\ell)}(\mathbf{x})\|_\infty$ we have
\begin{equation}\label{eq:error_composition_result}
E_\ell \;\le\; \sum_{j=1}^{\ell}\left(\prod_{m=j+1}^{\ell} L_{T^{(m)}}\right)\varepsilon_j,
\end{equation}
with the empty product equal to $1$.
\end{lemma}

\begin{proof}
The case $\ell=1$ is $E_1\le \varepsilon_1$. Assume the claim for $\ell-1$. Then, for any $\mathbf{x}$,
\begin{align*}
\|\mathbf{h}^{(\ell)}(\mathbf{x})-\hat{\mathbf{h}}^{(\ell)}(\mathbf{x})\|_\infty
&= \|T^{(\ell)}(\mathbf{h}^{(\ell-1)}(\mathbf{x})) - \hat T^{(\ell)}(\hat{\mathbf{h}}^{(\ell-1)}(\mathbf{x}))\|_\infty\\
&\le \underbrace{\|T^{(\ell)}(\mathbf{h}^{(\ell-1)}(\mathbf{x})) - T^{(\ell)}(\hat{\mathbf{h}}^{(\ell-1)}(\mathbf{x}))\|_\infty}_{\le L_{T^{(\ell)}}\,\|\mathbf{h}^{(\ell-1)}(\mathbf{x})-\hat{\mathbf{h}}^{(\ell-1)}(\mathbf{x})\|_\infty}
+ \underbrace{\|T^{(\ell)}-\hat T^{(\ell)}\|_{\infty,\mathcal{B}_{\ell-1}}}_{\le \varepsilon_\ell}\\
&\le L_{T^{(\ell)}} E_{\ell-1} + \varepsilon_\ell,
\end{align*}
and the induction hypothesis gives \eqref{eq:error_composition_result}.
\end{proof}

\begin{corollary}[Global spline approximation rate for SNs with piecewise-linear splines]\label{cor:sprecher_rate}
Let $f:[0,1]^n\to\mathbb{R}$ be realized by an ideal $L$-block SN (possibly with lateral mixing and cyclic residuals) with scalar output. For each block $\ell=1,\dots,L$, let $B_{\mathrm{in}}^{(\ell)}:=\sup\{\|\mathbf{z}\|_\infty:\mathbf{z}\in\mathcal{B}_{\ell-1}\}$. Assume:
\begin{enumerate}[label=\textup{(\roman*)}]
\item The \emph{ideal constituent} functions $\phi^{(\ell)}$ and $\Phi^{(\ell)}$ are in $C^2$ on neighborhoods of the intervals
$$I_\phi^{(\ell)} := \bigl[-B_{\mathrm{in}}^{(\ell)}-|\eta^{(\ell)}|(d_{\ell}-1),\; B_{\mathrm{in}}^{(\ell)}+|\eta^{(\ell)}|(d_{\ell}-1)\bigr],
\quad
I_\Phi^{(\ell)} := \bigl[-R_{\mathrm{mix}}^{(\ell)}M_s^{(\ell)},\; R_{\mathrm{mix}}^{(\ell)}M_s^{(\ell)}\bigr],$$
with $M_s^{(\ell)}:=\|\lambda^{(\ell)}\|_{1}\|\phi^{(\ell)}\|_{L^\infty(I_\phi^{(\ell)})}+|\alpha|(d_\ell-1)$ and $R_{\mathrm{mix}}^{(\ell)}:=1+|\tau^{(\ell)}|\,B_\omega^{(\ell)}$, $B_\omega^{(\ell)}:=\sup_q\sum_{j\in\mathcal{N}^{(\ell)}(q)}|\omega^{(\ell)}_{qj}|$;
\item the structural parameters are bounded: $\|\lambda^{(\ell)}\|_{1}\le \Lambda_1$, $|\eta^{(\ell)}|\le H$, $|\alpha|\le A$, $|\tau^{(\ell)}|\le T$, and $B_\omega^{(\ell)}\le \Omega$;
\item there exist bounded sets $\mathcal{B}_\ell$ such that the exact and approximated forward passes both remain in $\mathcal{B}_\ell$ for all inputs in $[0,1]^n$ (bounded propagation).
\end{enumerate}
Construct $\hat\phi^{(\ell)}$ and $\hat\Phi^{(\ell)}$ as the piecewise-linear interpolants on uniform grids with $G_\phi^{(\ell)}\ge2$ and $G_\Phi^{(\ell)}\ge2$ knots on $I_\phi^{(\ell)}$ and $I_\Phi^{(\ell)}$, respectively, and write $h_\phi^{(\ell)}:=|I_\phi^{(\ell)}|/(G_\phi^{(\ell)}-1)$ and $h_\Phi^{(\ell)}:=|I_\Phi^{(\ell)}|/(G_\Phi^{(\ell)}-1)$. Then, with
$$\delta_\phi^{(\ell)}:=\frac{\|\phi^{(\ell)''}\|_{L^\infty(I_\phi^{(\ell)})}}{8}\,\bigl(h_\phi^{(\ell)}\bigr)^2, \qquad
\delta_\Phi^{(\ell)}:=\frac{\|\Phi^{(\ell)''}\|_{L^\infty(I_\Phi^{(\ell)})}}{8}\,\bigl(h_\Phi^{(\ell)}\bigr)^2,$$
the blockwise error satisfies
$$\varepsilon_\ell
:=\sup_{\mathbf{z}\in\mathcal{B}_{\ell-1}}\|T^{(\ell)}(\mathbf{z})-\hat T^{(\ell)}(\mathbf{z})\|_\infty
\ \le\ L_{\Phi}^{(\ell)}\,R_{\mathrm{mix}}^{(\ell)}\,\|\lambda^{(\ell)}\|_{1}\,\delta_\phi^{(\ell)}\;+\;\delta_\Phi^{(\ell)},$$
where $L_{\Phi}^{(\ell)}:=\|\Phi^{(\ell)'}\|_{L^\infty(I_\Phi^{(\ell)})}$. Consequently, by Lemma~\ref{lem:error_composition},
$$\sup_{\mathbf{x}\in[0,1]^n}|f(\mathbf{x})-\hat f(\mathbf{x})|
\;\le\;
\sum_{j=1}^{L}\left(\prod_{m=j+1}^{L}L_{T^{(m)}}\right)\Bigl(L_{\Phi}^{(j)}\,R_{\mathrm{mix}}^{(j)}\,\|\lambda^{(j)}\|_{1}\,\delta_\phi^{(j)}+\delta_\Phi^{(j)}\Bigr),$$
with $L_{T^{(m)}}$ any Lipschitz constant of the $m$-th block on $\mathcal{B}_{m-1}$ (e.g.\ from Lemma~\ref{lem:block_lip}, plus the residual-path constant when present). In particular, if $G_\phi^{(\ell)}=G_\Phi^{(\ell)}=G$ for all $\ell$, then with $h:=\max_\ell\{|I_\phi^{(\ell)}|,|I_\Phi^{(\ell)}|\}/(G-1)$ we obtain
$$\sup_{\mathbf{x}\in[0,1]^n}|f(\mathbf{x})-\hat f(\mathbf{x})|
\;=\; \mathcal{O}\!\left(h^2\right)
\;=\; \mathcal{O}\!\left(G^{-2}\right),$$
with constants depending on $\{\|\lambda^{(\ell)}\|_{1},L_{\phi^{(\ell)}},L_{\Phi}^{(\ell)},R_{\mathrm{mix}}^{(\ell)},\|\phi^{(\ell)''}\|_\infty,\|\Phi^{(\ell)''}\|_\infty,L_{T^{(\ell)}}\}_{\ell=1}^L$ and on the bounded-propagation sets $\{\mathcal{B}_\ell\}$.
\end{corollary}

\begin{remark}[Impact of lateral mixing and cyclic residuals]
Lateral mixing appears only inside the argument of $\Phi$ and inflates constants via $R_{\mathrm{mix}}^{(\ell)}=1+|\tau^{(\ell)}|\,B_\omega^{(\ell)}$ (with $B_\omega^{(\ell)}\le \|\omega^{(\ell)}\|_\infty N_{\max}^{(\ell)}$); it does not change the $G^{-2}$ rate. If a cyclic residual path $R$ is present and shared between $T^{(\ell)}$ and $\hat T^{(\ell)}$, it cancels in the blockwise difference and thus does not affect $\varepsilon_\ell$, but it \emph{does} contribute additively to $L_{T^{(\ell)}}$ in \eqref{eq:error_composition_result}. For example, when $d_{\mathrm{in}}=d_{\mathrm{out}}$ and $R(\mathbf{z})=w_{\mathrm{res}}\mathbf{z}$, we may take $L_{T^{(\ell)}}\le L_{\Phi}^{(\ell)}\,R_{\mathrm{mix}}^{(\ell)}\,\|\lambda^{(\ell)}\|_{1}\,L_{\phi^{(\ell)}} + |w_{\mathrm{res}}|$, with $L_{\phi^{(\ell)}}:=\|\phi^{(\ell)'}\|_{L^\infty(I_\phi^{(\ell)})}$. Analogous bounds hold for broadcast/pooling residuals using the operator norm of the corresponding linear map.
\end{remark}

\begin{remark}[Monotone spline parameterization]
Sprecher's construction requires the inner map to be nondecreasing. The piecewise-linear interpolant of a $C^2$ nondecreasing function on a uniform grid is itself nondecreasing (no oscillatory overshoot), and the $O(h^2)$ error bound used above holds unchanged. Hence the analysis is compatible with the monotone-spline parameterizations of $\phi$ used in this work.
\end{remark}

\begin{remark}[Depth dependence]
The $\mathcal{O}(G^{-2})$ \emph{rate} is dimension-free, while constants accumulate with depth through $\prod_{m=j+1}^{L}L_{T^{(m)}}$ as in \eqref{eq:error_composition_result}. This mirrors standard error-propagation phenomena in deep models and highlights the practical value of regularizing $\|\lambda^{(\ell)}\|_{1}$, $L_{\phi^{(\ell)}}$, and $L_{\Phi^{(\ell)}}$, and of controlling lateral mixing ($\tau^{(\ell)},\omega^{(\ell)}$).
\end{remark}

\begin{remark}[Smoothness assumptions and the original Sprecher construction]
The approximation analysis above assumes that the ideal constituent functions $\phi^{(\ell)}$ and $\Phi^{(\ell)}$ are $C^2$ (and hence Lipschitz). This assumption is natural for our \emph{spline-parameterized} networks, since piecewise-linear and PCHIP splines on bounded domains are automatically Lipschitz. However, it is worth noting that Sprecher's original 1965 existence proof constructs an outer function $\Phi$ that is only guaranteed to be \emph{continuous}, not necessarily Lipschitz or smooth. (The inner function $\phi$ in Sprecher's construction is monotonic and Lipschitz by design.) Thus, while our approximation rates apply to the class of SNs parameterized by smooth splines, they do not directly characterize the regularity of the functions guaranteed to exist by Sprecher's theorem. Extending the analysis to handle less regular $\Phi$ (e.g., merely continuous or H\"older continuous) remains an open theoretical question.
\end{remark}

\subsection{Vector-valued functions and deeper extensions}
For vector-valued functions $f: [0,1]^n \to \mathbb{R}^m$ with $m>1$, our construction appends an $(L+1)$-th block without final summation. While intuitively extending the representation, the universality of this specific construction is not directly covered by Sprecher's original theorem. The composition of multiple Sprecher blocks to create deep networks represents a natural but theoretically uncharted extension of Sprecher's construction. While single-layer universality is guaranteed, the expressive power of deep SNs remains an open question with several competing hypotheses. Depth might provide benefits analogous to those in standard neural networks: enabling more efficient representation of compositional functions, creating a more favorable optimization landscape despite the constrained parameter space, or allowing the network to gradually transform inputs into representations that are progressively easier to process. The addition of lateral mixing connections may further enhance these benefits by enabling richer transformations at each layer. Alternatively, the specific constraints of the SN architecture might interact with depth in unexpected ways, either amplifying the benefits of the structured representation or creating new challenges not present in single-layer networks.

\begin{conjecture}[Vector-valued Sprecher Representation]\label{conj:vector}
Let $n, m \in \mathbb{N}$ with $m > 1$, and let $f:[0,1]^n \to \mathbb{R}^m$ be any continuous function. Then for any $\epsilon > 0$, there exists a Sprecher Network with architecture $n \to [d_1] \to m$ (using $L=1$ hidden block of width $d_1 \ge 2n+1$ and one output block), with sufficiently flexible continuous splines $\phi^{(1)}, \Phi^{(1)}, \phi^{(2)}, \Phi^{(2)}$ ($\phi^{(1)}, \phi^{(2)}$ monotonic), appropriate parameters $\lambda^{(1)}, \eta^{(1)}, \lambda^{(2)}, \eta^{(2)}$, and optionally lateral mixing parameters, such that the network output $\hat{f}(\mathbf{x})$ satisfies $\sup_{\mathbf{x} \in [0,1]^n} \|f(\mathbf{x}) - \hat{f}(\mathbf{x})\|_{\mathbb{R}^m} < \epsilon$.
\end{conjecture}

Furthermore, stacking multiple Sprecher blocks ($L > 1$) creates deeper networks. It is natural to hypothesize that these deeper networks also possess universal approximation capabilities, potentially offering advantages in efficiency or learning dynamics for certain function classes, similar to depth advantages observed in MLPs. The role of lateral mixing in enhancing or modifying these universality properties remains unexplored.

\begin{conjecture}[Deep universality]\label{conj:deep_universal}
For any input dimension $n \ge 1$, any number of hidden blocks $L \ge 1$, and any continuous function $f: [0,1]^n \to \mathbb{R}$ (or $f: [0,1]^n \to \mathbb{R}^m$), and any $\epsilon > 0$, there exists a Sprecher Network with architecture $n \to [d_1, \dots, d_L] \to 1$ (or $\to m$), provided the hidden widths $d_1, \dots, d_L$ are sufficiently large (e.g., perhaps $d_\ell \ge 2 d_{\ell-1} + 1$ is sufficient, although likely not necessary), with sufficiently flexible continuous splines $\phi^{(\ell)}, \Phi^{(\ell)}$, appropriate parameters $\lambda^{(\ell)}, \eta^{(\ell)}$, and optionally lateral mixing parameters $\tau^{(\ell)}, \omega^{(\ell)}$, such that the network output $\hat{f}(\mathbf{x})$ satisfies $\sup_{\mathbf{x} \in [0,1]^n} |f(\mathbf{x}) - \hat{f}(\mathbf{x})| < \epsilon$ (or the vector norm equivalent).
\end{conjecture}

Proving Conjectures \ref{conj:vector} and \ref{conj:deep_universal} rigorously would require analyzing the compositional properties and ensuring that the range of intermediate representations covers the domain needed by subsequent blocks, potentially involving careful control over the spline ranges, the effect of the shifts $\eta^{(\ell)}$, and the impact of lateral mixing on the network's expressive power.

\section{Implementation considerations}\label{sec:implementation}

\subsection{Trainable splines}

For practical implementations, we support both piecewise-linear splines and $C^1$ cubic Hermite splines with \textsc{PCHIP} (Piecewise Cubic Hermite Interpolating Polynomial) slopes for both $\phi^{(\ell)}$ and $\Phi^{(\ell)}$. We use \textsc{PCHIP} (rather than, e.g., a B-spline basis) primarily for its shape-preservation: when the knot values are monotone it yields a monotone interpolant without oscillatory overshoot, which complements the monotonicity constraint used for $\phi^{(\ell)}$ in Sprecher Networks. Piecewise-linear splines are the default when only function values are required (high efficiency, simple evaluation); the cubic option is used when higher-order derivatives must flow through the spline---e.g., the Poisson PINN benchmark, which requires second derivatives (defined a.e.\ for $C^1$ Hermite splines). In both cases, a spline with $G$ uniformly spaced knot locations (thus $G-1$ intervals) is parameterized by the values at these $G$ knot locations (thus $O(G)$ parameters per spline); for the cubic case, the \textsc{PCHIP} knot slopes are computed deterministically from these knot values. This retains a simple ``values-at-knots'' parameterization, which we will also exploit for efficient range propagation (Lemma~\ref{lem:domain_prop}). For range propagation, piecewise-linear extrema on an interval occur at the interval endpoints and/or knot points. For cubic Hermite/\textsc{PCHIP} splines, extrema over an interval may occur at the interval endpoints, at interior knot points, or at interior critical points of the cubic segments (when present); we therefore evaluate all such candidates in the theoretical domain propagation of Section~\ref{sec:theoretical_domains} to obtain sound and tight range bounds. The knot count $G$ is a hyperparameter that trades off expressivity and computational cost.

For the inner spline $\phi^{(\ell)}$, we fix the codomain to $[0,1]$ and enforce \emph{strict} monotonicity of the knot values by parameterizing positive increments and normalizing. Concretely, let $(v_k)_k$ be unconstrained parameters and set $\Delta_k=\operatorname{softplus}(v_k)>0$; then define
$$u_k \;=\; \sum_{i\le k}\Delta_i,\qquad c_k \;=\; \frac{u_k}{u_{G-1}+\varepsilon},$$
with a tiny $\varepsilon>0$ (we use $\varepsilon=10^{-8}$) to prevent division by zero; this yields strictly increasing knot values with $0<c_0<c_1<\cdots<c_{G-1}<1$ (and $c_{G-1}\to 1$ as $\varepsilon\to 0$). Interpolating these knots with either piecewise--linear segments or a shape--preserving monotone $C^1$ cubic Hermite (\textsc{PCHIP}) scheme yields a monotone non-decreasing spline on its domain. Many KAN implementations~\cite{liu2024kan} parameterize edge activations using B-spline bases; we use piecewise-linear / \textsc{PCHIP} interpolation instead because $\phi$ is constrained to be monotone: with \textsc{PCHIP}, monotonicity can be enforced simply by requiring increasing knot values (as above), and the interpolation is shape-preserving (avoiding overshoot). Outside the domain we use constant extension, so $\phi^{(\ell)}(x)=0$ for inputs below the leftmost knot and $\phi^{(\ell)}(x)=1$ above the rightmost knot. The increments are initialized nearly uniform so that $\phi^{(\ell)}$ starts close to linear.

The outer spline $\Phi^{(\ell)}$ operates on a domain determined by the block’s bounds. Optionally, we parameterize its codomain via an interval $[\,c_c^{(\ell)}-c_r^{(\ell)},\,c_c^{(\ell)}+c_r^{(\ell)}\,]$ with trainable center $c_c^{(\ell)}$ and scale (initialized positive) $c_r^{(\ell)}$; when this codomain parameterization is disabled, $\Phi^{(\ell)}$ is used directly without an explicit codomain normalization, so its output range is determined by its knot values. When enabled, we initialize $(c_c^{(\ell)},c_r^{(\ell)})$ from the computed $\Phi^{(\ell)}$ input domain (i.e., the pre-activation range) so that $\Phi^{(\ell)}$ is near identity at initialization, and the parameters then adapt the output range during training. Monotonicity is not required for $\Phi^{(\ell)}$.

\subsection{Shifts, weights, and optimization}
Each block includes the learnable scalar shift $\eta^{(\ell)}$ and the learnable mixing weight vector $\lambda^{(\ell)} \in \mathbb{R}^{d_{\ell-1}}$. A practical default is to initialize $\lambda^{(\ell)}$ with a variance-preserving scheme (e.g.\ i.i.d.\ $\mathcal{N}(0,2/d_{\ell-1})$) so that the pre-activations $s_q^{(\ell)}$ start at $O(1)$ scale, and to initialize $\eta^{(\ell)}$ to a small positive value on the order of $1/d_\ell$, so that the initial shifts $\eta^{(\ell)}q$ span an $O(1)$ range across $q=0,\dots,d_\ell-1$. While Sprecher's original construction requires $\eta>0$, practical implementations can relax this constraint: the theoretical domain computation in Lemma~\ref{lem:domain_prop} supports both positive and negative values of $\eta^{(\ell)}$, so we treat $\eta^{(\ell)}$ as an unconstrained parameter during training.

When lateral mixing is enabled, each block additionally includes a lateral scale parameter $\tau^{(\ell)}$ (typically initialized to a small value like 0.1) and lateral mixing weights $\omega^{(\ell)}$. For cyclic mixing, this involves $d_\ell$ weights (one per output), while bidirectional mixing requires $2d_\ell$ weights (forward and backward for each output). These weights are typically initialized to small values to ensure training begins with minimal lateral communication, allowing the network to gradually learn the optimal mixing patterns.

All learnable parameters (spline coefficients, $\eta^{(\ell)}$, $\lambda^{(\ell)}$, lateral mixing parameters, and potentially range parameters for $\Phi^{(\ell)}$) are trained jointly using gradient-based optimization methods like Adam \cite{kingma2014adam} or LBFGS. The loss function is typically Mean Squared Error (MSE) for regression tasks. Due to the constrained parameter space and shared spline structure, SNs may require more training iterations than equivalent MLPs or KANs to converge; per-iteration wall-clock cost is workload- and implementation-dependent (fewer parameters reduce optimizer overhead, but spline evaluation and memory-saving recomputation can increase runtime).

\subsection{Memory-efficient forward computation}
\label{sec:memory_efficient}

A key advantage of SNs' vector-based weight structure extends beyond parameter count to memory footprint, in particular to peak forward-intermediate memory. This is particularly relevant when memory (not compute) is the primary constraint, e.g.\ on-device/edge inference, embedded accelerators, or training regimes where activation memory dominates. While dense MLP layers require $O(N^2)$ parameters (and thus memory) for their weight matrices alone, SN blocks have only $O(N)$ vector parameters (plus $O(G)$ spline parameters) and can be evaluated without materializing the full $d_{\mathrm{in}}\times d_{\mathrm{out}}$ outer-product intermediate.

The standard forward pass for a Sprecher block naively computes:
$$\text{shifted}_{b,i,q} = x_{b,i} + \eta q \quad \forall b \in [B], i \in [d_{\mathrm{in}}], q \in [d_{\mathrm{out}}]$$
storing the full tensor before applying $\phi$, requiring $O(B \cdot d_{\mathrm{in}} \cdot d_{\mathrm{out}})$ memory. However, we can compute the \emph{unmixed} pre-activations
$$s_{b,q} = \sum_{i=1}^{d_{\mathrm{in}}} \lambda_i \phi(x_{b,i} + \eta q) + \alpha q$$
sequentially for $q = 0, \ldots, d_{\mathrm{out}}-1$ without materializing $\text{shifted}_{b,i,q}$. When lateral mixing is enabled, we then apply
$$\tilde{s}_{b,q} = s_{b,q} + \tau \sum_{j\in \mathcal{N}(q)} \omega_{q,j}\, s_{b,j}$$
to the assembled vector $(s_{b,q})_{q=0}^{d_{\mathrm{out}}-1}$ before evaluating $\Phi(\tilde{s}_{b,q})$ (and adding the residual path).

This reformulation reduces peak \emph{forward-intermediate} memory from $O(B \cdot N^2)$ to $O\!\bigl(B \cdot \max(d_{\mathrm{in}}, d_{\mathrm{out}})\bigr)$ during computation, while producing \emph{mathematically identical} results. Combined with SNs' $O(N+G)$ parameter memory per block, the total (parameters + peak forward intermediates) scales as $O\!\bigl(N+G + B \cdot \max(d_{\mathrm{in}}, d_{\mathrm{out}})\bigr)$; for fixed $B$ and $G$, this is $O(N)$, compared to dense MLPs' $O(N^2)$ parameter memory per block. (In standard reverse-mode autodiff frameworks, training may require additional saved state for backpropagation; the bound above refers to forward intermediates/inference, or to training with recomputation/checkpointing.)

In practice, this enables training wider architectures under fixed memory budgets by avoiding materialization of the full $(B\times d_{\mathrm{in}}\times d_{\mathrm{out}})$ tensor of shifted inputs. The runtime impact is workload-dependent: sequential evaluation reduces intra-layer parallelism, but can also improve cache behavior and reduce allocator pressure for very large layers.

\begin{remark}[Preservation of mathematical structure]
The sequential computation is a pure implementation optimization that is mathematically exact: it produces identical outputs to the naive parallel formulation for any choice of parameters and inputs. It only reduces peak memory by avoiding materialization of the full $(B\times d_{\mathrm{in}}\times d_{\mathrm{out}})$ tensor of shifted inputs; it does not change the function being computed.
\end{remark}

\begin{remark}[Parallelism vs.\ memory trade-off]
Sequential (per-$q$) evaluation reduces peak activation memory from $O(B\,d_{\mathrm{in}} d_{\mathrm{out}})$ to $O\bigl(B\,\max\{d_{\mathrm{in}}, d_{\mathrm{out}}\}\bigr)$, at the cost of less intra-layer parallelism. This is often advantageous on accelerators with ample compute but constrained memory. 
\end{remark}

\subsection{Embedded device implementation}
\label{sec:embedded}

As a practical demonstration of SNs' memory efficiency, we implement an SN for digit classification on a severely resource-constrained embedded device: a handheld gaming console with a 67\,MHz ARM9 processor, no floating-point unit (FPU), and only 4\,MB of RAM. Inference is performed via an emulator using Q16.16 fixed-point arithmetic.

We train an SN with architecture $784\to[100,100]\to10$ for MNIST digit classification, requiring only 2{,}613 parameters. An analogous MLP with comparable layer structure would require approximately 89{,}610 parameters---over 34$\times$ more. Figure~\ref{fig:embedded-demo} shows the network successfully classifying hand-drawn digits in real time, with inference completing in approximately 2 seconds per digit. The top screen displays the architecture information and softmax output probabilities, while the bottom touchscreen serves as the input canvas.

We also train a deeper SN with architecture $784\to[1280,1280,1280]\to10$, amounting to 16{,}594 parameters. An equivalent-depth MLP at width 1280 would require approximately 4.3 million parameters, exceeding the device's available memory entirely. This demonstrates that the linear parameter scaling of SNs enables deployment of meaningfully expressive models on hardware where traditional dense architectures are simply infeasible.

\begin{figure}[ht]
\centering
\includegraphics[width=0.28\textwidth]{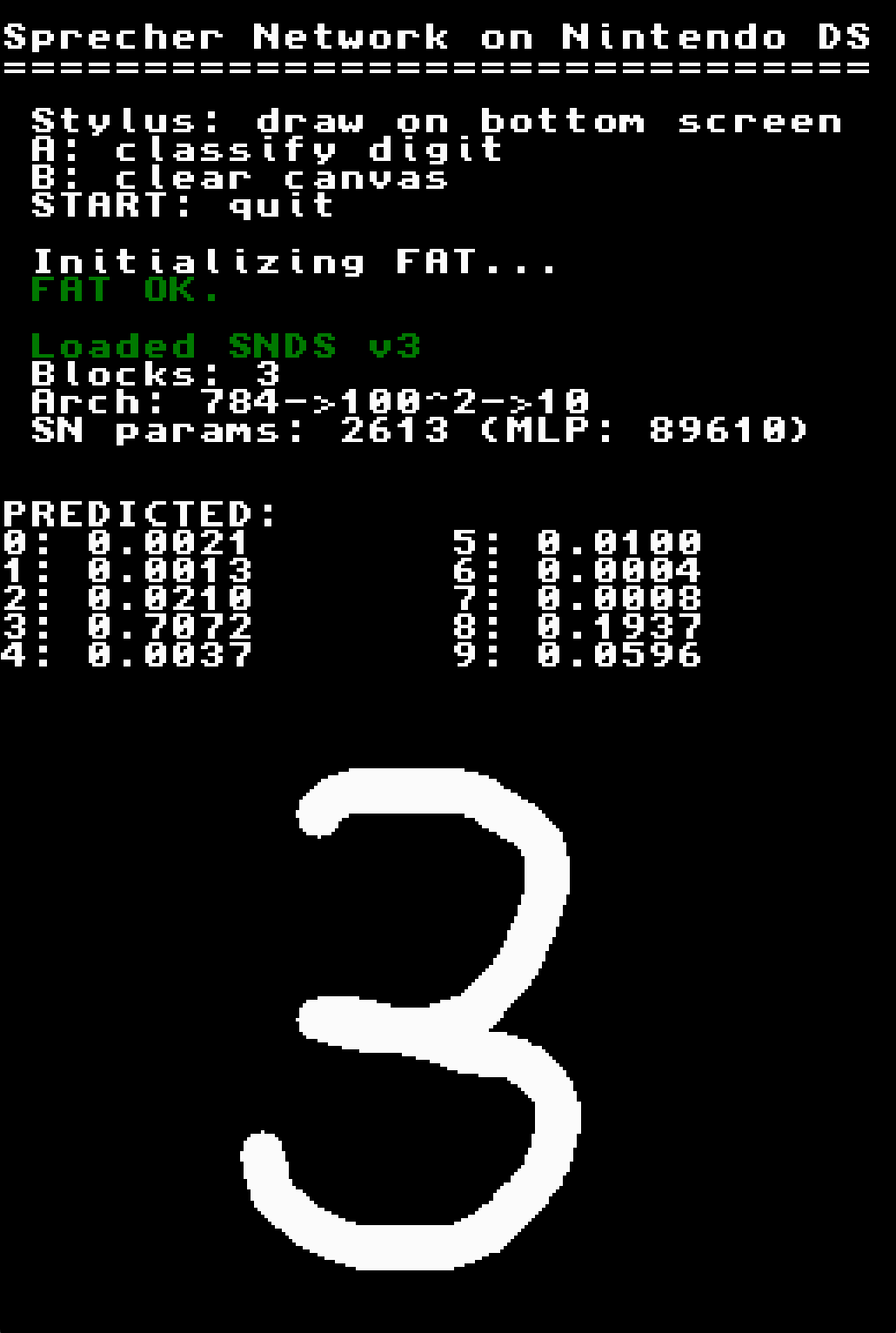}
\hfill
\includegraphics[width=0.28\textwidth]{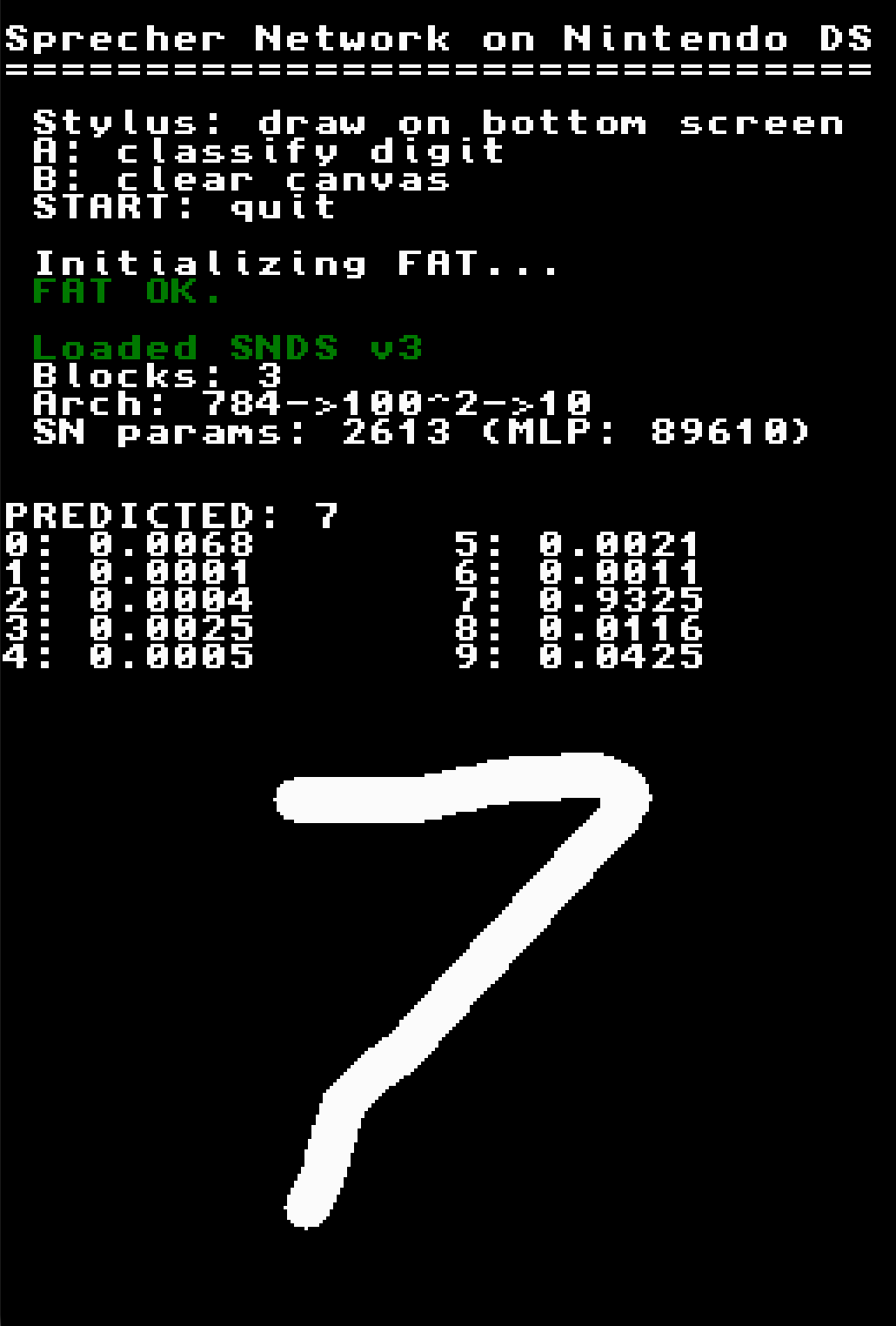}
\hfill
\includegraphics[width=0.28\textwidth]{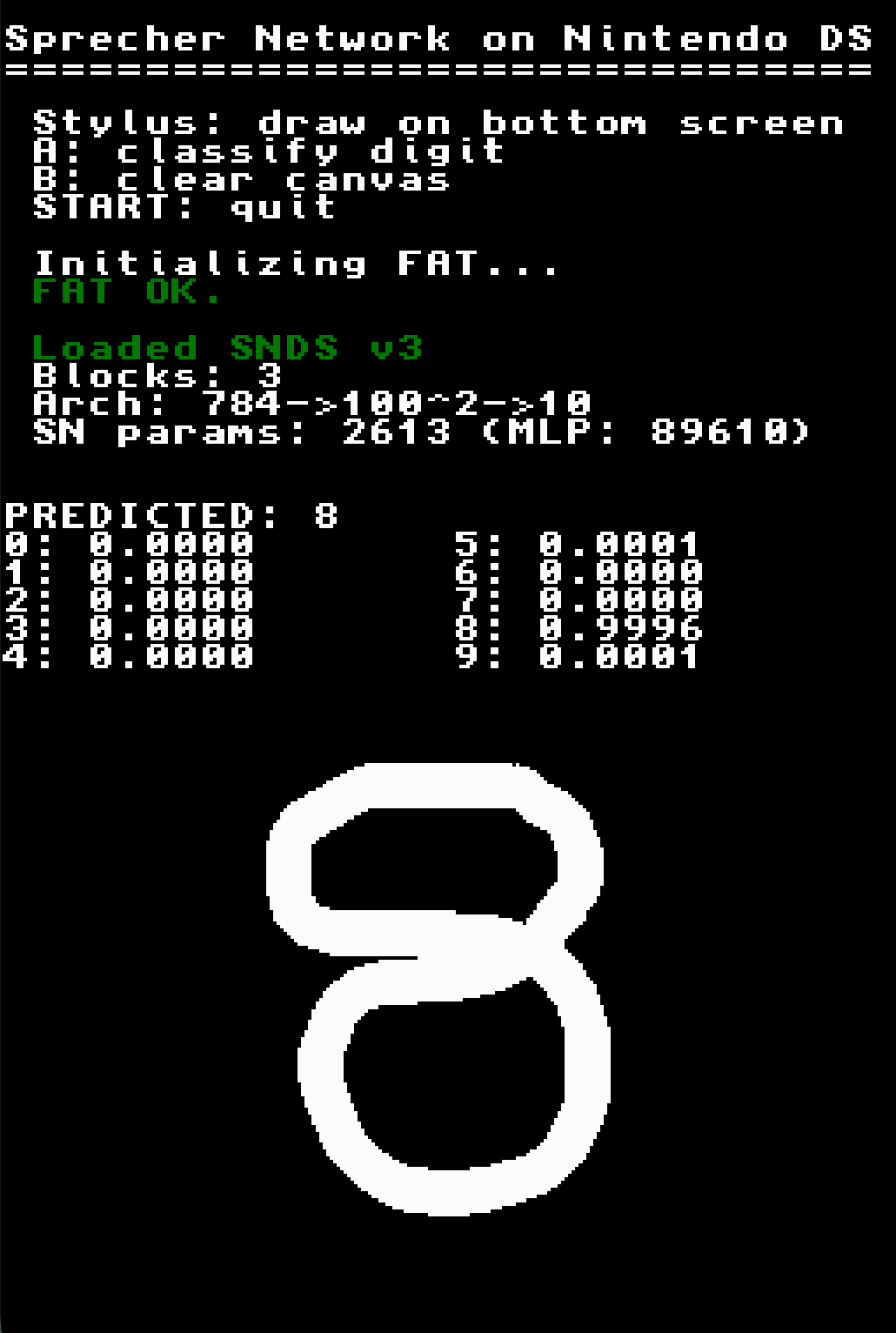}
\caption{Real-time inference on a resource-constrained embedded device (67\,MHz ARM9 processor, no FPU, 4\,MB RAM). An SN with 2{,}613 parameters classifies hand-drawn digits using Q16.16 fixed-point arithmetic. The top screen displays architecture information and softmax probabilities; the bottom touchscreen shows the user-drawn input digit.}
\label{fig:embedded-demo}
\end{figure}

\subsection{Theoretical domain computation}
\label{sec:theoretical_domains}
One advantage of the structured Sprecher architecture is the ability to compute \emph{sound} (typically conservative) interval bounds for spline inputs and outputs, which we use to adapt spline domains during training. When normalization layers are inserted between blocks, we propagate post-normalization bounds as described in Section~\ref{sec:normalization} (for BatchNorm in batch-statistics mode these are heuristic, whereas in running-statistics mode the bounds are exact since the transform is affine). For a fixed input interval to a spline, the exact range of a piecewise-linear spline on that interval is attained at the endpoints and any interior knots; for cubic Hermite/\textsc{PCHIP} splines, we additionally evaluate any interior critical points of the cubic segments (when present) to obtain sound range bounds. When these per-spline range computations are propagated through multiple blocks using interval arithmetic, the resulting network-level bounds are generally conservative (since correlations between coordinates are ignored), but are typically much tighter than naive global bounds. Throughout this analysis, we assume inputs lie in $[0, 1]^n$, though the methodology extends naturally to any bounded domain. This assumption enables principled initialization and can aid in debugging and analysis.

\begin{lemma}[Domain propagation through Sprecher blocks]\label{lem:domain_prop}
Consider a Sprecher block where inputs lie in a box $\prod_{i=1}^{d_{\mathrm{in}}}[a_i,b_i]$ (and define $a:=\min_i a_i$ and $b:=\max_i b_i$), with shift parameter $\eta$, weights $\lambda_i$, lateral mixing parameters $\tau, \omega$, and output dimension $d_{\mathrm{out}}$. Let us distinguish between a spline's \emph{domain} (the interval of valid inputs) and its \emph{codomain} (the interval of possible outputs). Then:
\begin{enumerate}
\item The domain required for $\phi$ to handle all possible inputs is:
$$\mathcal{D}_\phi = \begin{cases}
[a, b + \eta(d_{\mathrm{out}}-1)] & \text{if } \eta \geq 0 \\
[a + \eta(d_{\mathrm{out}}-1), b] & \text{if } \eta < 0
\end{cases}$$

\item Without lateral mixing, monotonicity of $\phi$ implies $\phi(x_i+\eta q)\in[\phi(a+\eta q),\,\phi(b+\eta q)]$ for all $i$ and $q$. Let
$\lambda^{+}:=\sum_{i:\lambda_i\ge 0}\lambda_i$ and $\lambda^{-}:=\sum_{i:\lambda_i<0}\lambda_i$.
Then for each head index $q$,
$$
s_{q}^{\min,\mathrm{unmixed}}=\lambda^{+}\,\phi(a+\eta q)+\lambda^{-}\,\phi(b+\eta q)+\alpha q,\qquad
s_{q}^{\max,\mathrm{unmixed}}=\lambda^{+}\,\phi(b+\eta q)+\lambda^{-}\,\phi(a+\eta q)+\alpha q.
$$
A valid domain for $\Phi$ is therefore $\mathcal{D}_{\Phi}=\bigl[\min_{q}s_{q}^{\min,\mathrm{unmixed}},\;\max_{q}s_{q}^{\max,\mathrm{unmixed}}\bigr]$.
(As a coarse bound, if $\phi$ is parameterized to satisfy $\phi(\cdot)\in[0,1]$, one may replace the unknown values $\phi(a+\eta q)$ and $\phi(b+\eta q)$ by $0$ and $1$, respectively.)

\item With lateral mixing, for cyclic mixing with scale $\tau$ and weights $\omega$, the domain for $\Phi$ must account for sign-aware neighbor contributions. Using the \emph{unmixed} bounds from the previous item, for each output $q$ the mixed pre-activation bounds are:
$$\begin{aligned}
s^{\min}_q &= s^{\min,\text{unmixed}}_q + (\tau\omega_q)^+ s^{\min,\text{unmixed}}_{(q+1) \bmod d_{\mathrm{out}}} + (\tau\omega_q)^- s^{\max,\text{unmixed}}_{(q+1) \bmod d_{\mathrm{out}}}\\
s^{\max}_q &= s^{\max,\text{unmixed}}_q + (\tau\omega_q)^+ s^{\max,\text{unmixed}}_{(q+1) \bmod d_{\mathrm{out}}} + (\tau\omega_q)^- s^{\min,\text{unmixed}}_{(q+1) \bmod d_{\mathrm{out}}}
\end{aligned}$$
where for any scalar $c$ we write $c^+ = \max(c,0)$ and $c^- = \min(c,0)$. For bidirectional mixing, apply the analogous sign-split bound to each neighbor contribution $(\tau\omega_{q,\text{fwd}})\,s_{(q+1)\bmod d_{\mathrm{out}}}$ and $(\tau\omega_{q,\text{bwd}})\,s_{(q-1)\bmod d_{\mathrm{out}}}$, and sum the contributions.

\item Tight output bounds require checking all candidate extrema of $\Phi$ over the interval. For piecewise linear $\Phi$, extrema occur at endpoints and knot points, so it suffices to evaluate $\Phi$ at interval endpoints and all knots within the interval. For cubic $\Phi$ (PCHIP), the interpolant is shape-preserving (it does not overshoot the knot values). However, since a cubic segment can attain its extrema at interior critical points, tight range bounds require evaluating $\Phi$ at the interval endpoints, all knots within the interval, and any interior critical points that lie within the interval. (If an affine codomain reparameterization $(c_c,c_r)$ is used, it is an affine map of the underlying spline values; it does not change the set of candidate points for extrema, only rescales and shifts the resulting values.) This yields correct range propagation even for oscillatory $\Phi$.

\item When residual connections are present, the output range must be adjusted by the residual contribution. For a scalar residual weight $w$ (the $d_{\mathrm{in}}=d_{\mathrm{out}}$ case), the residual term in coordinate
$q$ is $r_q=w x_{q+1}$ and lies in $[\min\{w a_{q+1},w b_{q+1}\},\max\{w a_{q+1},w b_{q+1}\}]$. For a linear residual projection matrix $W\in\mathbb{R}^{d_{\mathrm{in}}\times d_{\mathrm{out}}}$ (the $d_{\mathrm{in}}\neq d_{\mathrm{out}}$ case), the residual term in coordinate $q$ is $r_q=\sum_{i=1}^{d_{\mathrm{in}}}W_{i,q}x_i$ and lies in:
$$r_q^{\min}=\sum_{i=1}^{d_{\mathrm{in}}}\bigl(W^{+}_{i,q}a_i+W^{-}_{i,q}b_i\bigr),\qquad
r_q^{\max}=\sum_{i=1}^{d_{\mathrm{in}}}\bigl(W^{+}_{i,q}b_i+W^{-}_{i,q}a_i\bigr),$$
where $W^{+}=\max(W,0)$ and $W^{-}=\min(W,0)$ elementwise. These adjustments can be applied elementwise. For Cyclic residuals:

   \begin{itemize}

\item 
\textbf{Broadcast}: Each output $q$ receives $r_q = w_q^{\text{bcast}} \cdot x_{1+(q \bmod d_{\mathrm{in}})}$, hence
$$r_q \in \bigl[\min\{w_q^{\text{bcast}} a_{1+(q \bmod d_{\mathrm{in}})}, w_q^{\text{bcast}} b_{1+(q \bmod d_{\mathrm{in}})}\}, \max\{w_q^{\text{bcast}} a_{1+(q \bmod d_{\mathrm{in}})}, w_q^{\text{bcast}} b_{1+(q \bmod d_{\mathrm{in}})}\}\bigr].$$
   \item 
\textbf{Pooling}: With assignment $i \mapsto q(i)$ and weights $w_i$, the contribution to output $q$ is:
   $$
r^{\min}_q = \sum_{i: q(i)=q} \min\{w_i a_i, w_i b_i\}, \quad r^{\max}_q = \sum_{i: q(i)=q} \max\{w_i a_i, w_i b_i\}$$
   \end{itemize}

\item \textbf{Out-of-domain extension:} When inputs fall outside the spline's domain:
   \begin{itemize}
   \item $\phi$ (monotonic): Uses \emph{constant} extension (zero slope) outside its domain: for $x<\min\mathcal{D}_\phi$, set $\phi(x)=0$; for $x>\max\mathcal{D}_\phi$, set $\phi(x)=1$. This preserves the codomain $[0,1]$ and aligns with the monotone-increments parameterization.
   \item $\Phi$ (general): Uses \emph{linear} extrapolation based on the boundary \emph{derivative} of the spline interpolation (piecewise-linear: slope of the first/last segment; cubic \textsc{PCHIP}: endpoint derivative from the \textsc{PCHIP} slope rule). If $\Phi$'s codomain is parameterized by $(c_c,c_r)$ with $\operatorname{cod}(\Phi)=[c_c-c_r,\,c_c+c_r]$, we first evaluate the \emph{unnormalized} spline $\tilde{\Phi}$ (including its linear extrapolation) in coefficient space, then map its value $t=\tilde{\Phi}(x)$ by
   $$t \;\mapsto\; c_c - c_r \;+\; 2c_r \cdot \frac{t - t_{\min}}{t_{\max}-t_{\min}+\epsilon}\,,$$
   where $t_{\min}=\min_k \tilde{\Phi}(x_k)$ and $t_{\max}=\max_k \tilde{\Phi}(x_k)$ are the minimum/maximum unnormalized knot values (with $\{x_k\}$ the knot locations) and $\epsilon>0$ is a small constant. The boundary derivatives are scaled by the same factor $2c_r/(t_{\max}-t_{\min}+\epsilon)$.
   \end{itemize}
\end{enumerate}
\end{lemma}

\begin{corollary}[Per-dimension input intervals]
When per-dimension intervals $x_i \in [a_i, b_i]$ are available, tighter bounds can be computed:
$$\begin{aligned}
s_q^{\min} &= \sum_i \bigl(\lambda_i^+ \phi(a_i + \eta q) + \lambda_i^- \phi(b_i + \eta q)\bigr) + \alpha q\\
s_q^{\max} &= \sum_i \bigl(\lambda_i^+ \phi(b_i + \eta q) + \lambda_i^- \phi(a_i + \eta q)\bigr) + \alpha q
\end{aligned}$$
where $\lambda_i^+ = \max(\lambda_i, 0)$ and $\lambda_i^- = \min(\lambda_i, 0)$. These per-$q$ bounds then undergo sign-aware lateral mixing as in Lemma \ref{lem:domain_prop}.
\end{corollary}

This lemma enables a forward propagation algorithm for computing all spline domains throughout the network. The algorithm can optionally apply a small relative domain safety margin (e.g., 10\%) to reduce edge hits during training; setting the margin to 0 yields the tightest domains. Crucially, the algorithm's efficiency relies on the fact that spline extrema over compact intervals can be bounded by evaluating a small, explicit candidate set: for piecewise-linear splines, the interval endpoints and any knots inside the interval; for piecewise-cubic Hermite/\textsc{PCHIP} splines, the interval endpoints, knots inside the interval, and any interior critical points of the cubic pieces that lie inside the interval.

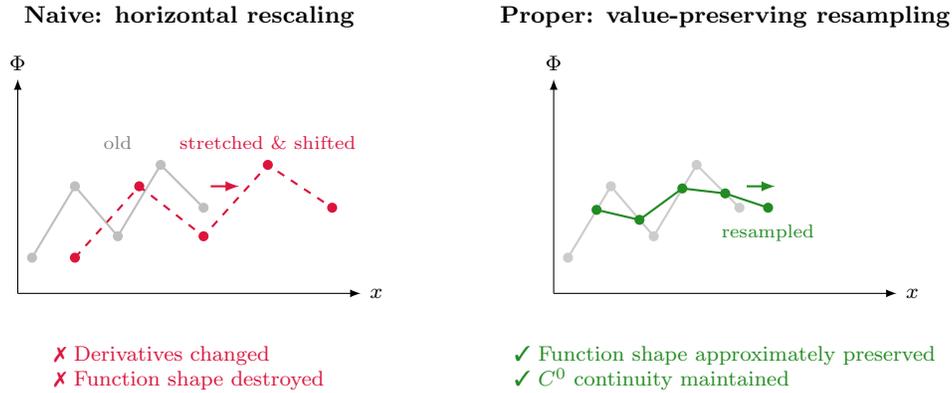
\begin{figure}[ht]
    \centering
    % Custom colors
    \definecolor{sprecherblue}{RGB}{70,130,180}
    \definecolor{sprecherred}{RGB}{220,20,60}
    \definecolor{sprechergreen}{RGB}{34,139,34}
    
    \begin{tikzpicture}[scale=0.95, >=latex]
        % --- Left Subfigure: Naive Approach (Stretching & Shifting) ---
        \begin{scope}[xshift=0cm]
            % Header
            \node[font=\small\bfseries, anchor=south] at (2.4, 3.6) {Naive: horizontal rescaling};
            
            % Axes
            \draw[->] (0,0) -- (4.8,0) node[right, font=\footnotesize] {$x$};
            \draw[->] (0,0) -- (0,3.0) node[above, font=\footnotesize] {$\Phi$};
            
            % Old function (Gray)
            \draw[thick, gray!50] (0.2,0.5) -- (0.8,1.5) -- (1.4,0.8) -- (2.0,1.8) -- (2.6,1.2);
            \foreach \p in {(0.2,0.5), (0.8,1.5), (1.4,0.8), (2.0,1.8), (2.6,1.2)}
                \fill[gray!50] \p circle (2pt);
            \node[font=\scriptsize, gray, anchor=south] at (1.4, 1.9) {old};
            
            % Transition Arrow
            \draw[->, thick, sprecherred] (2.7, 1.5) -- (3.1, 1.5);
            
            % New function (Red Dashed) - Stretched & Shifted
            \draw[thick, sprecherred, dashed] (0.8,0.5) -- (1.7,1.5) -- (2.6,0.8) -- (3.5,1.8) -- (4.4,1.2);
            \foreach \p in {(0.8,0.5), (1.7,1.5), (2.6,0.8), (3.5,1.8), (4.4,1.2)}
                \fill[sprecherred] \p circle (2pt);
            
            \node[font=\scriptsize, sprecherred, anchor=south] at (3.5, 1.9) {stretched \& shifted};
            
            % Failure Text: Centered
            \node[anchor=north, align=left, font=\footnotesize, text=sprecherred] at (2.4, -0.6) {
                \ding{55} Derivatives changed\\
                \ding{55} Function shape destroyed
            };
        \end{scope}
        
        % --- Right Subfigure: Proper Approach (Resampling) ---
        \begin{scope}[xshift=7.5cm]
            % Header
            \node[font=\small\bfseries, anchor=south] at (2.4, 3.6) {Proper: value-preserving resampling};
            
            % Axes
            \draw[->] (0,0) -- (4.8,0) node[right, font=\footnotesize] {$x$};
            \draw[->] (0,0) -- (0,3.0) node[above, font=\footnotesize] {$\Phi$};
            
            % Old function (Gray)
            \draw[thick, gray!40] (0.2,0.5) -- (0.8,1.5) -- (1.4,0.8) -- (2.0,1.8) -- (2.6,1.2);
            \foreach \p in {(0.2,0.5), (0.8,1.5), (1.4,0.8), (2.0,1.8), (2.6,1.2)}
                \fill[gray!40] \p circle (2pt);
            
            % Transition Arrow
            \draw[->, thick, sprechergreen] (2.7, 1.5) -- (3.1, 1.5);
            
            \draw[thick, sprechergreen] (0.6,1.17) -- (1.2,1.03) -- (1.8,1.47) -- (2.4,1.40) -- (3.0,1.20);
            
            % New Knots (5 count)
            \foreach \p in {(0.6,1.17), (1.2,1.03), (1.8,1.47), (2.4,1.40), (3.0,1.20)}
                \fill[sprechergreen] \p circle (2pt);
                
            % Label moved safely below the tail
            \node[font=\scriptsize, sprechergreen, anchor=north] at (3.0, 1.1) {resampled};
            
            % Success Text: Centered
            \node[anchor=north, align=left, font=\footnotesize, text=sprechergreen] at (2.4, -0.6) {
                \ding{51} Function shape approximately preserved\\
                \ding{51} $C^0$ continuity maintained
            };
        \end{scope}
    \end{tikzpicture}
    \caption{Spline domain resampling during training (applied to the outer spline $\Phi$). As parameters $\lambda$ and $\eta$ evolve, induced spline input ranges shift dynamically. \textbf{Left:} Naive horizontal rescaling distorts the learned function by stretching and shifting the shape to fit the new domain. \textbf{Right:} Our resampling evaluates the previous spline at the new knot positions (for piecewise-linear $\Phi$, query points are clamped to the previous domain so any new knot outside the old domain uses the nearest boundary value; for cubic/\textsc{PCHIP} $\Phi$, out-of-domain query points are evaluated via the spline's linear extrapolation determined by the endpoint derivative), approximately preserving the function values at the new knots that lie within the previous domain; the updated $\Phi$ still uses linear extrapolation outside its \emph{new} domain.}
    \label{fig:resampling}
\end{figure}

\begin{algorithm}
\caption{Forward domain propagation with lateral mixing and per-dimension tracking}
\label{alg:domain_prop}
\begin{algorithmic}[1]
    \State \textbf{Input:} Network parameters $\{\lambda^{(\ell)}, \eta^{(\ell)}, \phi^{(\ell)}, \Phi^{(\ell)}, \tau^{(\ell)}, \omega^{(\ell)}, R^{(\ell)}\}_{\ell=1}^{L_{\text{blocks}}}$, input domain $[0,1]^n$
    \State \textbf{(Here $L_{\text{blocks}}$ is the number of Sprecher blocks: $L_{\text{blocks}}=L$ for the summed scalar-output form, and $L_{\text{blocks}}=L{+}1$ when an explicit non-summed output block is used (always for vector outputs, and optionally also for scalar outputs).)}
    \State \textbf{Output:} Domains and ranges for all splines
    \State Initialize current range $\mathcal{R}_0 \leftarrow [0,1]^n$ (per-dimension when available)
    \For{each block $\ell = 1, \ldots, L_{\text{blocks}}$}
        \If{normalization is applied \emph{before} block $\ell$}
            \State Apply normalization bounds to $\mathcal{R}_{\ell-1}$ (batch-statistics: conservative shared bound derived from standardized $[-4,4]$; running-statistics: use running stats and affine)
        \EndIf
        \State $\mathcal{D}_\phi^{(\ell)} \leftarrow$ apply Lemma \ref{lem:domain_prop} part (1) to the current input range
        \If{per-dimension intervals available}
            \State Compute per-$q$ bounds using the corollary (sign-splitting via $\lambda_i^\pm$ and $\phi(a_i+\eta q),\phi(b_i+\eta q)$)
        \Else
            \State Compute coarse per-$q$ bounds $s^{\min,\mathrm{unmixed}}_q,s^{\max,\mathrm{unmixed}}_q$ via Lemma~\ref{lem:domain_prop} part (2)
        \EndIf
        \State Apply sign-aware lateral mixing (part 3) to get final $s^{\min}_q, s^{\max}_q$
        \State $\mathcal{D}_\Phi^{(\ell)} \leftarrow [\min_q s^{\min}_q, \max_q s^{\max}_q]$
        \State Evaluate $\Phi^{(\ell)}$ on each interval $[s^{\min}_q, s^{\max}_q]$ at (i) the endpoints, (ii) all $\Phi$-knots lying inside the interval, and (iii) (for cubic $\Phi$) any interior critical points of the relevant cubic pieces lying inside the interval
        \State Let $[y_q^{\min},y_q^{\max}]$ denote the range of $\Phi^{(\ell)}$ on $[s_q^{\min},s_q^{\max}]$
        \If{block $\ell$ has residual connections}
            \State Adjust each $[y_q^{\min},y_q^{\max}]$ according to Lemma \ref{lem:domain_prop} part (5)
        \EndIf
        \If{$\ell=L_{\text{blocks}}$ and the network output is scalar (sum over $q$)}
            \State $\mathcal{R}_\ell \leftarrow \Bigl[\sum_{q=0}^{d_\ell-1} y_q^{\min},\,\sum_{q=0}^{d_\ell-1} y_q^{\max}\Bigr]$
        \Else
            \State $\mathcal{R}_\ell \leftarrow \prod_{q=0}^{d_\ell-1} [y_q^{\min},y_q^{\max}]$
        \EndIf
        \If{normalization is applied \emph{after} block $\ell$}
            \State Apply normalization bounds (batch-statistics: heuristic shared bound based on standardized $[-4,4]$; running-statistics: use running stats and affine)
        \EndIf
    \EndFor
\end{algorithmic}
\end{algorithm}

\begin{remark}[Dynamic spline updates and value-preserving resampling]
A critical challenge in training Sprecher Networks is that the domains of the splines $\phi^{(\ell)}$ and $\Phi^{(\ell)}$ evolve as the parameters $\eta^{(\ell)}$, $\lambda^{(\ell)}$, and lateral mixing parameters are updated. To maintain theoretical fidelity to Sprecher's formula while adapting to evolving domains, we employ a selective value-preserving resampling strategy:

\paragraph{General Splines ($\Phi^{(\ell)}$):} When the domain of a $\Phi^{(\ell)}$ spline changes, new knot locations are established (typically uniformly spaced across the new computed domain). The original spline is treated as a continuous function and evaluated at these new knot locations to yield updated knot coefficients (for piecewise-linear $\Phi^{(\ell)}$, query points are clamped to the previous domain so any new knot outside the old domain uses the nearest boundary value; for cubic/\textsc{PCHIP} $\Phi^{(\ell)}$, out-of-domain query points are evaluated via the spline's linear extrapolation determined by the endpoint derivative). This effectively ``resamples'' the learned shape onto the new domain: it preserves the \emph{values} of the old spline at the new knot locations that lie within the previous domain (and therefore preserves the function exactly on the overlap whenever the old spline is linear on each new knot interval, e.g.\ when the new knot set contains the old knots). For piecewise-linear $\Phi^{(\ell)}$, this procedure is exact whenever the new knot set contains the old knots; for cubic (PCHIP) $\Phi^{(\ell)}$, it matches the old spline at the resampled knots but can (slightly) change derivatives between knots, while still avoiding abrupt resets. Optionally, one may additionally apply a cheap post-update refresh that adjusts knot \emph{locations} without resampling coefficients; this is not value-preserving (it amounts to a mild horizontal rescaling), so it is best reserved for small domain changes where the primary goal is to keep bound propagation synchronized after parameter updates.

\paragraph{Monotonic Splines ($\phi^{(\ell)}$):} For the monotonic splines, whose purpose is to provide an increasing map to $[0,1]$, complex resampling is not required. Their learnable parameters define relative increments between knots, not an arbitrary shape. Therefore, a straightforward update of the knot positions to the new theoretical domain is sufficient and computationally efficient.

This targeted approach avoids the most direct sources of information loss and instability from domain changes. In engineering ablations on Toy-2D (Table~\ref{tab:component_ablations}), disabling $\Phi$-resampling typically led to oscillatory or unstable training as domains shifted, while disabling dynamic domain updates caused spline inputs to drift far outside their learned ranges (forcing $\phi$ into saturation and $\Phi$ into extrapolation). We therefore treat dynamic domains and $\Phi$-resampling as default components unless explicitly stated otherwise (e.g., the PINN setting in Sec.~\ref{sec:pinn_poisson}). The fundamental challenge of optimizing splines within dynamically shifting domains remains, but this mitigation strategy has proven effective in practice.
\end{remark}

One particularly useful application of domain computation is the initialization of $\Phi^{(\ell)}$ splines. When $\Phi$ codomain scaling $(c_c^{(\ell)},c_r^{(\ell)})$ is enabled, we compute each $\Phi^{(\ell)}$'s theoretical domain before training and (i) set its knot domain to this interval, (ii) initialize its knot coefficients to a linear ramp over that same interval so that the underlying spline represents the identity map on the interval (for both linear and cubic/\textsc{PCHIP} interpolation), and (iii) initialize $(c_c^{(\ell)},c_r^{(\ell)})$ to the domain center and radius so that the affine codomain reparameterization preserves this identity map. This provides a principled initialization strategy that ensures the initial network performs a series of near-identity transformations.

\begin{remark}[Practical benefits]
The ability to compute \emph{sound} (typically conservative) domain and range bounds through interval arithmetic (provably sound in the absence of normalization; with BatchNorm in training mode we use the heuristic bounds from Section~\ref{sec:normalization}) provides several practical benefits: (i) it enables theoretically grounded initialization without arbitrary hyperparameters, (ii) it can help diagnose training issues by detecting when values fall outside expected ranges, (iii) it promotes numerical stability by encouraging spline evaluations to remain within their well-resolved knot regions (with explicit extension/extrapolation behavior outside the knot domain), and (iv) it allows for adaptive domain adjustments that account for lateral mixing dynamics. These benefits distinguish Sprecher Networks from architectures with less structured internal representations.
\end{remark}

\begin{table}[t]
\centering
\small
\begin{tabular}{l p{0.68\linewidth}}
\toprule
Component ablated & Typical observed effect \\
\midrule
Disable domain resampling & Training can become unstable when pre-activations drift beyond initial spline domains, increasing extrapolation error. \\
Disable domain updates & Intermediate pre-activations may drift outside predicted bounds, reducing effective capacity via saturation/clipping effects. \\
Remove normalization & Pre-activations become poorly scaled, making spline-domain management harder and slowing optimization. \\
Remove $\Phi$ codomain scaling $(c_c,c_r)$ & Forces $\Phi$ to learn both scale and shape, often slowing convergence and degrading final accuracy. \\
Remove cyclic residuals & Deep networks become harder to train, with degraded accuracy and occasional divergence. \\
Remove lateral mixing & Wide shallow networks can plateau due to shared-weight symmetries (cf.\ Remark above). \\
\bottomrule
\end{tabular}
\caption{Qualitative summary of component ablations observed across our experiments. These effects align with each component's motivation: domain management for numerical stability, normalization and $\Phi$ codomain scaling for conditioning, residual connections for depth, and lateral mixing for breaking shared-weight symmetries.}
\label{tab:component_ablations}
\end{table}

\section{Empirical demonstrations and case studies}\label{sec:experiments}
We report a set of empirical studies designed to probe both performance and scaling behavior of SNs, with controlled comparisons to MLP and KAN baselines.

\paragraph{Benchmarks used in this work.}

We evaluate Sprecher Networks on synthetic supervised regression tasks, on a physics-informed Poisson problem, and on the Fashion-MNIST image classification benchmark.

\paragraph{Fashion-MNIST classification benchmark.}
To assess SNs on a standard real-world benchmark, we train on Fashion-MNIST \cite{xiao2017fashionmnist}, a 10-class image classification dataset of $28\times28$ grayscale images with 60,000 training and 10,000 test examples (784 input dimensions after flattening; we use the standard split). We use cross-entropy loss and report test accuracy (mean $\pm$ std over 10 seeds). Inputs are scaled to $[0,1]$ (raw pixel intensities rescaled by $1/255$) before being fed to both SN and MLP models. Unless stated otherwise, we train with AdamW (learning rate $10^{-3}$, batch size 64, weight decay $10^{-6}$) and clip gradients to max-norm 1. For the MLP baselines we use ReLU activations and no BatchNorm layers. For SNs we use cubic (\textsc{PCHIP}) splines with 60 knots for the inner maps $\phi$; for the outer maps $\Phi$ we use 60 knots for the shallow models and 180 knots for the 25-layer model (Table~\ref{tab:fashion-mnist}). We apply BatchNorm after each block (except the first) and use residual connections in all SN Fashion-MNIST runs (linear-projection residuals for the $[12,11,12]$ models; cyclic residuals for the constant-width 25-layer model). We enable cyclic lateral mixing in all SN Fashion-MNIST runs, and we update (theoretical) spline domains before each minibatch during both training and evaluation (resampling the knot values of $\Phi$ accordingly). When reporting test accuracy, we evaluate BatchNorm using per-batch statistics while freezing its running-statistic buffers. Table~\ref{tab:fashion-mnist} compares SNs against an MLP baseline under matched parameter budgets.

\begin{table}[ht]
\centering
\caption{\textbf{Fashion-MNIST classification results.} Test accuracy (mean $\pm$ std over 10 seeds). The shallow SN uses linear residual projections; the 25-layer SN uses cyclic residuals and demonstrates stable training at 25 hidden layers when equipped with residual connections and BatchNorm.}
\label{tab:fashion-mnist}
\small
\begin{tabular}{llrrrr}
\toprule
Model & Architecture & Params & Epochs & Best Test Accuracy (\%) \\
\midrule
SN & $784\to[12,11,12]\to10$ & 11{,}244 & 20 & $85.09 \pm 0.18$ \\
SN & $784\to[12,11,12]\to10$ & 11{,}244 & 200 & $86.10 \pm 0.28$ \\
SN (25 layers) & $784\to[40]^{25}\to10$ & 12{,}008 & 20 & $85.90 \pm 0.24$ \\
\midrule
MLP & $784\to[15,20]\to10$ & 12{,}305 & 20 & $86.27 \pm 0.09$ \\
MLP & $784\to[15,20]\to10$ & 12{,}305 & 200 & $86.76 \pm 0.15$ \\
\bottomrule
\end{tabular}
\end{table}

At this parameter budget, the shallow SN is competitive with the MLP baseline. More importantly for trainability, the constant-width 25-layer SN optimizes stably across seeds once equipped with cyclic residual connections and normalization, reaching $85.90\pm 0.24\%$ test accuracy. This suggests that deep stacks of Sprecher blocks can be trained reliably, although closing the remaining gap to dense MLPs (and understanding when depth helps) remains an open direction.

\paragraph{Poisson PINN benchmark (manufactured solution).}\label{sec:pinn_poisson}
We solve the Dirichlet problem
$$\Delta u(\mathbf{x}) = g(\mathbf{x}) \quad \text{for } \mathbf{x}\in \Omega,\qquad
u(\mathbf{x})=0 \quad \text{for } \mathbf{x}\in \partial\Omega,$$
with $\Omega=[-1,1]^2$, using the manufactured solution
$$u(x,y) = \sin(\pi x)\,\sin(\pi y^2),$$
so that $g=\Delta u$ is known analytically.
We train a network $u_\theta$ with collocation sets $S_{\mathrm{int}}\subset\Omega$ and $S_{\partial}\subset\partial\Omega$ by minimizing

$$\mathcal{L} \;=\;
\frac{1}{|S_{\mathrm{int}}|}\sum_{\mathbf{x}\in S_{\mathrm{int}}}
\frac{\bigl(\Delta u_\theta(\mathbf{x})-g(\mathbf{x})\bigr)^2}{\mathbb{E}_{\mathbf{x}\in S_{\mathrm{int}}}[g(\mathbf{x})^2]}
\;+\;
\frac{1}{|S_{\partial}|}\sum_{\mathbf{x}\in S_{\partial}} u_\theta(\mathbf{x})^2.$$
Unless stated otherwise, we draw a fixed set of $|S_{\mathrm{int}}|=2048$ interior points i.i.d.\ uniformly at random in $\Omega$ (sampled once at the start of training), and
$|S_{\partial}|=4\cdot 257$ boundary points are taken on a uniform grid along each edge (including the corners on each edge, hence corners are duplicated).

\emph{Model variants for the PINN benchmark.}
To isolate the spline parameterization itself, we compare a \textbf{barebones Sprecher Network} (no lateral mixing, no residuals, no normalization, and \textbf{fixed} spline domains; i.e.\ no domain updates) against a \textbf{barebones KAN} baseline with spline-only edge activations (no SiLU/base activation and no grid updates). Because the Poisson residual involves second derivatives, both models use shape-preserving cubic (PCHIP) splines in this benchmark (piecewise-linear splines yield $\Delta u_\theta=0$ almost everywhere).
For the Sprecher Network, inputs are linearly rescaled from $[-1,1]^2$ to $[0,1]^2$ before the first spline evaluation via the affine map $\rho:\Omega\to[0,1]^2$ defined component-wise by $\rho(\mathbf{x})=(\mathbf{x}+\mathbf{1})/2$, matching the $[0,1]$ spline-domain convention used for SNs in this work; the KAN baseline operates directly on $[-1,1]^2$. In both cases, derivatives are taken with respect to the original coordinates $\mathbf{x}$ via automatic differentiation (for SN this includes the Jacobian of $\rho$ by the chain rule).

\emph{Optimization and evaluation.}
We train in double precision (float64) using Adam with $10^{-3}$ learning rate, cosine learning rate decay to $10^{-5}$, and gradient-norm clipping 1.0. We report two parameter budgets: an $\approx$1200-parameter model trained for 5{,}000 epochs (mean $\pm$ std over 5 seeds), and an $\approx$3000-parameter model trained for 20{,}000 epochs (1 seed). We evaluate both models on a uniform $101\times101$ grid over $\Omega=[-1,1]^2$ and report: (i) PDE residual MSE (interior grid points), (ii) boundary MSE (boundary grid points), and (iii) $L^2$ MSE against the manufactured solution (full grid).

\begin{table}[t]
  \centering
  \caption{\textbf{Poisson PINN results (manufactured solution).} Metrics are evaluated on a uniform $101\times101$ grid over $\Omega=[-1,1]^2$: PDE residual MSE is computed on the interior grid points, boundary MSE on the boundary grid points, and $L^2$ MSE against the manufactured solution on the full grid (all lower is better). For the $\approx$1200-parameter setting we report mean $\pm$ std over 5 seeds; the $\approx$3000-parameter setting is reported for 1 seed due to cost.}
  \label{tab:poisson-pinn}
  \begin{tabular}{lrrr}
    \toprule
    Model & PDE MSE & Bndry MSE & $L^2$ MSE \\
    \midrule
    \multicolumn{4}{l}{\textbf{$\approx$1200 params, 5{,}000 epochs (5 seeds)}} \\
    SN  & $\mathbf{1.36\times 10^{2} \pm 3.69}$ & $7.85\times 10^{-2} \pm 7.22\times 10^{-3}$ & $\mathbf{6.55\times 10^{-2} \pm 1.75\times 10^{-3}}$ \\
    KAN & $2.41\times 10^{3} \pm 1.78\times 10^3$ & $\mathbf{2.06\times 10^{-4} \pm 4.53\times 10^{-5}}$ & $1.70\times 10^{-1} \pm 4.33\times 10^{-3}$ \\
    \midrule
    \multicolumn{4}{l}{\textbf{$\approx$3000 params, 20{,}000 epochs (1 seed)}} \\
    SN  & $\mathbf{1.38\times 10^2}$ & $4.76\times 10^{-2}$ & $\mathbf{1.05\times 10^{-1}}$ \\
    KAN & $1.09\times 10^4$ & $\mathbf{8.42\times 10^{-4}}$ & $1.32\times 10^{-1}$ \\
    \bottomrule
  \end{tabular}
\end{table}

\subsection{Basic function approximation}
We train SNs on datasets sampled from known target functions $f$. The network learns the parameters ($\eta^{(\ell)}$, $\lambda^{(\ell)}$), spline coefficients ($\phi^{(\ell)}$, $\Phi^{(\ell)}$), and when enabled, lateral mixing parameters ($\tau^{(\ell)}$, $\omega^{(\ell)}$) via gradient descent (Adam optimizer) using an MSE objective.

For 1D functions $f(x)$ on $[0,1]$, an SN like $1 \to [W] \to 1$ (one block) learns $\phi^{(1)}$ and $\Phi^{(1)}$ that accurately approximate $f$, effectively acting as a learnable spline interpolant structured according to Sprecher's formula. While the network achieves a very accurate approximation of the overall function $f(x)$, the learned components (splines $\hat{\phi}, \hat{\Phi}$, weights $\hat{\lambda}$, shift $\hat{\eta}$) only partially resemble the ground truth functions and parameters used to generate the data. This is expected: the internal decomposition is generally not identifiable, and in particular the learned outer spline $\hat{\Phi}$ can become strongly oscillatory/jagged even when the ground-truth outer function $\Phi$ is smooth.

Moving to multivariate functions, consider the 2D scalar case $f(x,y) = \exp(\sin(11x)) + 3y + 4\sin(8y)$. A network like $2\to[5,8,5]\to1$ (3 blocks) can achieve high accuracy. Figure
\ref{fig:twovarsprecher_revised} shows the interpretable layerwise spline plots and the final fit quality. When lateral mixing is enabled for this architecture, we often observe a different allocation of complexity across channels and layers; however, the learned $\Phi^{(\ell)}$ splines can still be highly oscillatory, so we treat the spline plots as qualitative rather than expecting smoothness.

For 2D vector-valued functions $f(x,y)=(f_1(x,y), f_2(x,y))$, we use an SN of the form $2\to[20,20]\to2$ (2 hidden layers plus a final block for the vector-valued output; 3 blocks total). Figure \ref{fig:twovarsprechervector_revised} illustrates the learned splines and the approximation of both output surfaces. Lateral mixing can be beneficial here by breaking output symmetry across channels, so we enable it in this example.

\begin{figure}[!ht]
\centering
\includegraphics[width=0.8\textwidth]{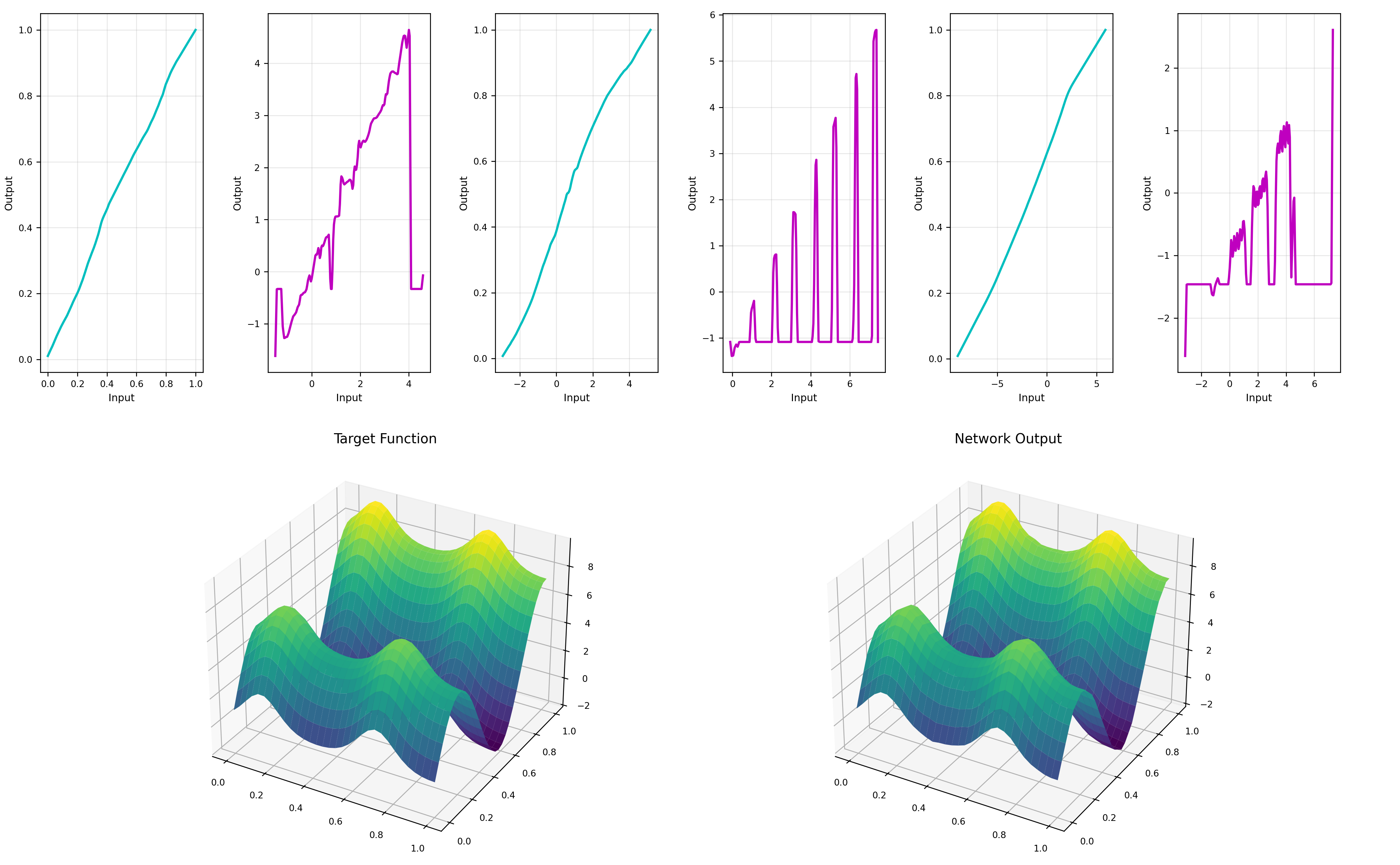}
\caption{Example of SN approximating the scalar 2D target function $z = f(x,y) = \exp(\sin(11x)) + 3y + 4\sin(8y)$ on $(x,y)\in[0,1]^2$ using architecture \texorpdfstring{$2\to[5,8,5]\to1$}{2 -> [5,8,5] -> 1} (3 blocks). Top row: Learned spline functions for each block --- monotonic splines $\phi^{(\ell)}$ (cyan) and general splines $\Phi^{(\ell)}$ (magenta). Bottom row: Comparison between the target function surface (left) and the network approximation (right). This run uses cyclic lateral mixing; the learned outer splines $\Phi^{(\ell)}$ can be highly oscillatory/jagged even when the overall fit is accurate, so the per-layer spline plots should be interpreted qualitatively.}
\label{fig:twovarsprecher_revised}
\end{figure}

\begin{figure}[!ht]
\centering
\includegraphics[width=0.8\textwidth]{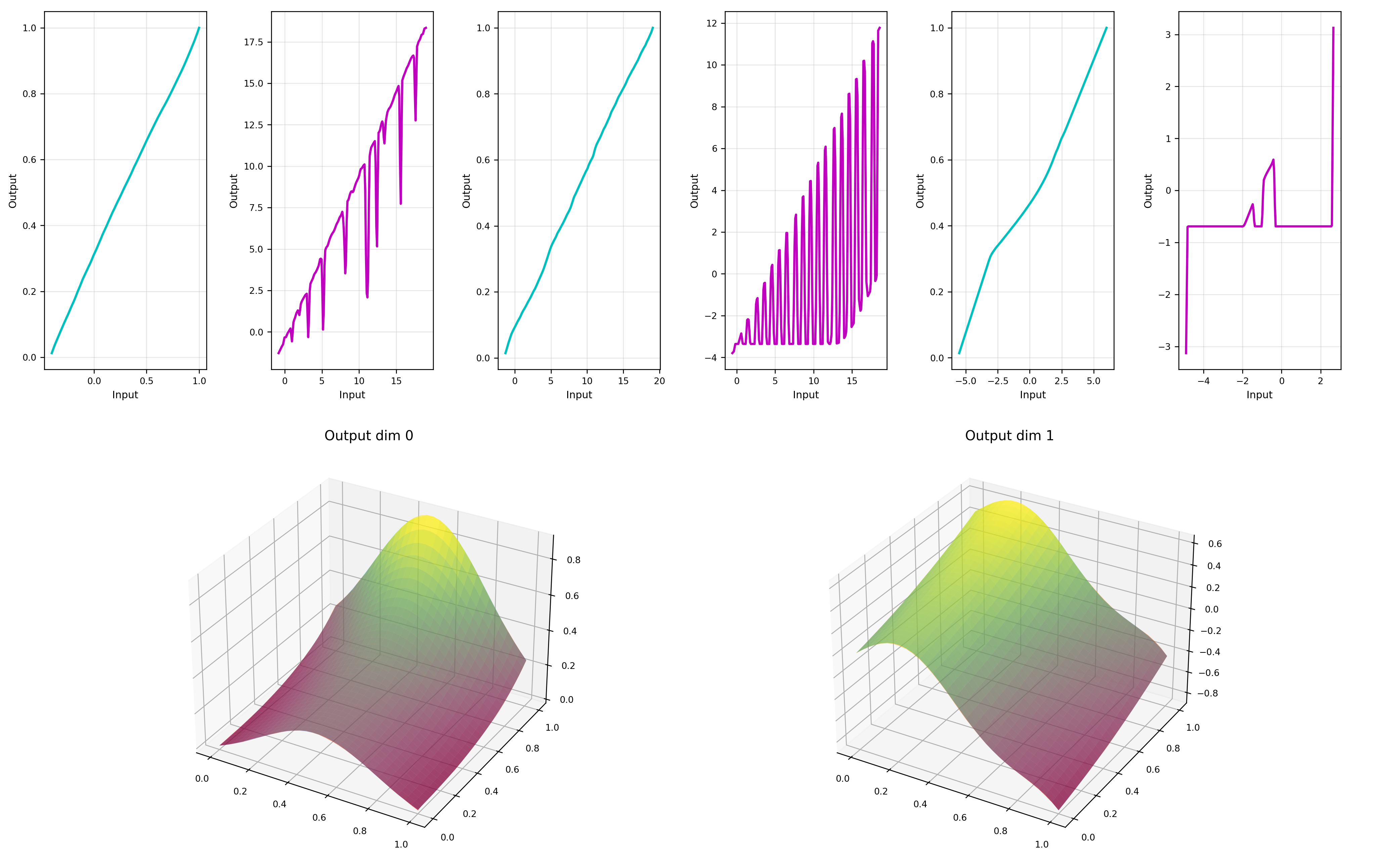}
\caption{Example of SN approximating a 2D vector-valued function $f:[0,1]^2\to\mathbb{R}^2$ with components $f_1(x,y) = \frac{\exp(\sin(\pi x) + y^2)-1}{7}$ and $f_2(x,y) = \frac{1}{4}y + \frac{1}{5}y^2 - x^3 + \frac{1}{5}\sin(7x)$. The network architecture is \texorpdfstring{$2\to[20,20]\to2$}{2 -> [20,20] -> 2} (2 hidden layers plus an additional output block; 3 blocks total). The first row shows the learned monotonic splines $\phi^{(\ell)}$ (cyan) and general splines $\Phi^{(\ell)}$ (magenta) for each block. The second row shows, for each output component (left: $f_1$, right: $f_2$), the target surface overlaid with the corresponding network output surface. Lateral mixing was enabled (cyclic in this run); as in other examples, the learned $\Phi^{(\ell)}$ splines can be highly oscillatory/jagged even when the overall fit is accurate.}
\label{fig:twovarsprechervector_revised}
\end{figure}

These examples demonstrate the feasibility of training SNs and the potential interpretability offered by visualizing the learned shared splines $\phi^{(\ell)}$ and $\Phi^{(\ell)}$ for each block, as well as the impact of lateral mixing on spline smoothness and approximation quality.

\subsection{Impact of lateral mixing}
To evaluate the contribution of lateral mixing, we conducted ablation studies comparing networks with and without this mechanism across various tasks:

\paragraph{Scalar outputs:} For many scalar-output regression tasks with moderate widths, lateral mixing yields only modest gains. However, as discussed in Remark~3.6, scalar-output SNs can benefit substantially in regimes where the shared-weight constraint induces strong symmetry across channels (e.g., very wide shallow networks), where lateral mixing provides lightweight cross-channel interaction that breaks optimization plateaus. Note that the final summation is a fixed linear readout \emph{after} applying $\Phi$ channelwise; it does not provide pre-$\Phi$ cross-channel communication, so lateral mixing remains a distinct (and sometimes crucial) mechanism even when the network output is scalar.

\paragraph{Vector outputs:} On vector-valued regression tasks, networks with lateral mixing (cyclic variant) often achieved lower RMSE with only a marginal increase in parameters (one additional per-channel weight plus a shared scale). The improvement was most pronounced for functions where output dimensions exhibit strong correlations, suggesting that lateral mixing can help the network capture cross-output dependencies despite the constraint of shared splines.

\paragraph{Convergence speed:} Networks with lateral mixing often converged faster during training in our ablations, reaching the same loss threshold in fewer iterations. This suggests that lateral connections can provide beneficial gradient pathways that accelerate optimization, although the effect is task-dependent.

\paragraph{Spline behavior:} Visual inspection of learned splines suggests that lateral mixing can change how complexity is distributed across output channels. In particular, the learned outer splines $\Phi^{(\ell)}$ can remain highly oscillatory even when the overall fit is accurate, which is consistent with Sprecher's original construction. KANs are often motivated by interpretability because each edge carries an explicit learned univariate function \cite{liu2024kan}; however, as Figure~\ref{fig:sn_vs_kan_splines} illustrates, these learned edge functions can also be jagged/noisy under parameter parity, so visual interpretability does not automatically imply smoothness.

\subsection{Baseline apples-to-apples comparisons: SNs vs.\ KANs}\label{sec:sn-kan-baselines}

\begin{figure}[!ht]
\centering
\includegraphics[width=0.95\textwidth]{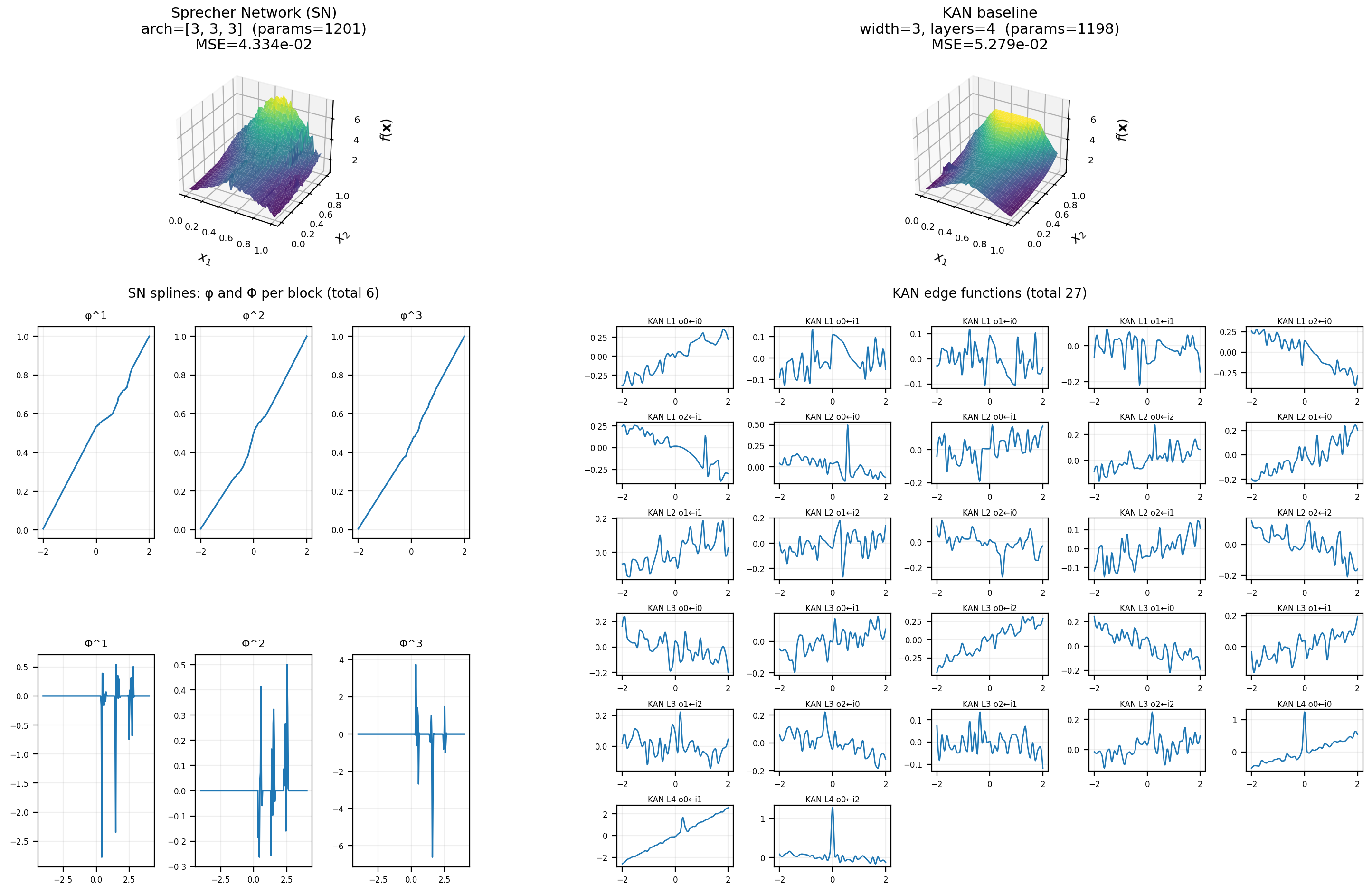}
\caption{Side-by-side comparison of learned functions and internal univariate splines for an SN and a KAN on the Toy-2D benchmark ($f(x,y)=\exp(\sin(\pi x)+y^2)$) under parameter parity ($\approx 1200$ parameters). For this figure we use a spline-isolation setting (no residual connections, no normalization, no lateral mixing, fixed spline domains) to make the learned univariate functions directly comparable. In the representative run shown, the SN (left; $\text{arch}=[3,3,3]$, $1201$ params) attains MSE $4.334\times 10^{-2}$ and the KAN (right; width=3, layers=4, $1198$ params) attains MSE $5.279\times 10^{-2}$. KANs are often motivated by interpretability because each edge carries an explicit learned univariate function \cite{liu2024kan}, yet the learned edge functions (27 total here, one per edge) can still be jagged/noisy in practice. In contrast, the SN achieves a comparable fit using one shared pair $(\phi^{(\ell)},\Phi^{(\ell)})$ per block ($6$ shared splines total across three blocks) with similarly oscillatory outer splines $\Phi^{(\ell)}$, consistent with the role of $\Phi$ in Sprecher-style superpositions.}

\label{fig:sn_vs_kan_splines}
\end{figure}

\paragraph{Rationale.}
The goal of this subsection is to provide simple, controlled, apples-to-apples tests where Sprecher Networks (SNs) and Kolmogorov–Arnold Networks (KANs) are trained under the same budget and with matched modeling capacity, so that we can isolate architectural inductive biases rather than hyperparameter tuning. We do \emph{not} claim that SNs dominate KANs on every task; rather, we highlight representative settings where SNs perform as well or better, sometimes substantially so, while keeping the comparison fair.

\paragraph{Protocol (fairness constraints).}
Unless otherwise stated (e.g., Figure~\ref{fig:sn_vs_kan_splines}), each head-to-head uses exactly $4000$ training epochs for both architectures. All benchmark results in this subsection use the following comparison setting:

\begin{itemize}[leftmargin=*,itemsep=2pt]
\item \textbf{Barebones comparisons} (\S\ref{sssec:barebones}): Both SN and KAN use \emph{no} residual connections, \emph{no} normalization layers, and (for SN) \emph{no} lateral mixing. SNs use piecewise-linear (PWL) splines with dynamic spline-domain updates during the first 10\% of training (400/4000 epochs), then frozen. KANs use cubic Hermite (PCHIP) splines with fixed uniform knot grids---no grid updates or adaptive knot relocation. This isolates the core architectural differences.
\end{itemize}

\noindent We match parameter counts by choosing the KAN spline basis size $K$ (the number of learnable knot values per edge spline) to be as close as possible to the SN parameter count (allowing small over/under-shoot when exact parity is not attainable). Seeds, datasets, and train/test splits are identical across models. Training uses either full-batch or mini-batch updates depending on the benchmark; when mini-batching is used, SN and KAN are trained on identical mini-batches in the same order. Primary metric is test-set RMSE (reported per task below; for vector-valued outputs, RMSE is computed over all test samples and output coordinates). For Benchmark~3 we report the best (minimum) test RMSE attained during training (a non-standard choice since it uses the test set for checkpoint selection); for Benchmark~8 (which includes a validation split) we report the test RMSE at the checkpoint with best validation RMSE.

\paragraph{Reporting conventions (and what we \emph{do not} report).}
We report test RMSE, parameter counts, and (where available) wall-clock training time. We do not report train RMSE and we do not treat wall-clock time as a primary metric (it is sensitive to implementation details and hardware).

\vspace{0.5ex}
\noindent\textbf{Structure of this subsection.} We present eight barebones comparisons (\S\ref{sssec:barebones}) across diverse regression tasks designed to isolate architectural differences between SNs and KANs.

%% =============================================================================
%% Benchmarks
%% =============================================================================
\subsubsection{Barebones comparisons (no residuals, no normalization)}\label{sssec:barebones}

In these experiments, we strip away all engineering enhancements and focus on the fundamental architectural differences between SNs and KANs:
\begin{itemize}[leftmargin=*,nosep]
\item \textbf{SN:} Piecewise-linear (PWL) splines; dynamic spline-domain updates during the first 10\% of training (400/4000 epochs), then frozen; no residuals; no normalization; no lateral mixing; no learned $\Phi$-codomain parameters.
\item \textbf{KAN:} Cubic Hermite (PCHIP) splines with fixed uniform knot grids; no grid updates; no residuals; no normalization; linear extrapolation outside the knot interval.
\end{itemize}

\noindent The target functions are chosen to highlight settings where the SN's structural inductive biases, particularly \emph{shared inner splines} across input dimensions and \emph{systematic head shifts}, provide an architectural advantage. For most benchmarks, these advantages stem from the Sprecher structure itself (parameter sharing, shifted-input reuse); we also include one benchmark (PWL-vs-PCHIP) that tests spline-type flexibility on non-smooth targets. Table~\ref{tab:barebones-summary} summarizes the results; we describe each benchmark below.

\begin{table}[ht]
\centering
\setlength{\tabcolsep}{3pt}
\renewcommand{\arraystretch}{1.15}
\footnotesize
\begin{tabularx}{\linewidth}{@{}>{\raggedright\arraybackslash}p{3.2cm} c c c c c c@{}}
\toprule
\textbf{Benchmark} & \textbf{\shortstack{Params (SN/KAN)}} & \textbf{SN RMSE} & \textbf{KAN RMSE} & \textbf{Ratio} & \textbf{\shortstack{SN Time (s)}} & \textbf{\shortstack{KAN Time (s)}} \\
\midrule
Softstair-Wavepacket (10D) & 2761/2761 & $\mathbf{0.352 \pm 0.003}$ & $0.551 \pm 0.009$ & $1.57\times$ & $38.4 \pm 5.9$ & $56.3 \pm 7.8$ \\
InputShift-Bump (12D, 17h) & 1599/1597 & $\mathbf{0.253 \pm 0.001}$ & $0.339 \pm 0.105$ & $1.34\times$ & --- & --- \\
Shared-Warped-Ridge (16D) & 2115/2113 & $\mathbf{0.701 \pm 0.002}$ & $0.805 \pm 0.100$ & $1.15\times$ & $266.7 \pm 64.2$ & $250.2 \pm 97.6$ \\
Shared-Warp-Chirp (10D) & 2447/2449 & $\mathbf{0.896 \pm 0.025}$ & $1.470 \pm 0.060$ & $1.64\times$ & $549.5 \pm 107.0$ & $432.7 \pm 124.9$ \\
Motif-Chirp (10D) & 2106/2107 & $\mathbf{0.123 \pm 0.006}$ & $0.182 \pm 0.015$ & $1.48\times$ & $305.3 \pm 100.4$ & $259.0 \pm 89.5$ \\
Oscillatory-HeadShift (12D, 64h) & 2763/2768 & $\mathbf{0.669 \pm 0.020}$ & $0.752 \pm 0.016$ & $1.12\times$ & $366.3 \pm 39.0$ & $225.1 \pm 28.3$ \\
PWL-vs-PCHIP (10D) & 1701/1701 & $\mathbf{0.732 \pm 0.012}$ & $0.998 \pm 0.035$ & $1.36\times$ & $76.5 \pm 11.4$ & $71.6 \pm 6.4$ \\
Quantile Harmonics (2D) & 2026/2025 & $\mathbf{0.395 \pm 0.146}$ & $0.619 \pm 0.035$ & $1.57\times$ & $1236.1 \pm 82.4$ & $3790.8 \pm 489.3$ \\
\bottomrule
\end{tabularx}
\caption{Barebones SN vs.\ KAN comparisons (20 seeds each; no residuals, no normalization, no lateral mixing; and no learned $\Phi$-codomain parameters in SN). RMSE values are mean $\pm$ std across seeds; \textbf{bold} indicates the winner. \emph{Ratio} $=$ KAN/SN (higher means larger SN advantage). Unless otherwise noted, RMSE is the final-epoch test RMSE; Benchmark~3 reports the best test RMSE observed during training, and Benchmark~8 reports test RMSE at the best validation checkpoint. Wall-clock times (in seconds) are mean $\pm$ std where per-model timings are available. SN uses PWL splines with dynamic spline-domain updates during the first 10\% of training (400/4000 epochs), then frozen; KAN uses cubic PCHIP splines with fixed grids.}
\label{tab:barebones-summary}
\end{table}

\paragraph{Benchmark 1: Softstair-Wavepacket (10D).}
The first benchmark tests a smooth but highly non-unimodal, permutation-symmetric target built from a shared monotone coordinate warp and a localized oscillatory wavepacket. For $t\in[0,1]$, define
$$\mathrm{softstair}(t)=t
+0.25\,\operatorname{sigmoid}(25(t-0.20))
+0.35\,\operatorname{sigmoid}(35(t-0.55))
+0.20\,\operatorname{sigmoid}(60(t-0.85)).$$
For $\mathbf{x}\in[0,1]^{10}$ let $s(\mathbf{x})=\frac{1}{10}\sum_{i=1}^{10}\mathrm{softstair}(x_i)$ and
$$y(\mathbf{x})=0.70\,\exp\!\bigl(-30(s(\mathbf{x})-0.90)^2\bigr)\Bigl[\sin(14\pi s(\mathbf{x})) + 0.30\sin(42\pi s(\mathbf{x}))\Bigr] + 0.05\,(s(\mathbf{x})-0.90).$$
With $n_{\text{train}}=2048$ and $n_{\text{test}}=8192$ i.i.d.\ uniform samples, SN achieves $0.352 \pm 0.003$ RMSE while KAN obtains $0.551 \pm 0.009$ (ratio $1.57\times$). The target applies the same 1D warp to each coordinate before pooling, matching SN's shared inner spline; a KAN must allocate separate edge splines per coordinate to learn this shared warp under the same parameter budget.

\paragraph{Benchmark 2: InputShift-Bump (12D, 17 heads).}
This multi-output regression defines $y(\mathbf{x})\in\mathbb{R}^{17}$ with components indexed by $q\in\{0,\dots,16\}$ through a head-dependent input shift. Let
$$\phi_\star(t)=\operatorname{sigmoid}(8(t-0.5)),\qquad
\Phi_\star(s)=\exp\!\left(-\frac12\left(\frac{s-8.0}{2.6}\right)^2\right)+0.25\sin(0.85\,s),$$
fix a mixing vector $w\in\mathbb{R}^{12}$ (sampled once per seed and then held fixed). Concretely, sample $\tilde{w}\sim\mathcal{N}(0,I_{12})$, set $\tilde{w}\leftarrow \tilde{w}/\|\tilde{w}\|$, let $s\in\mathbb{R}^{12}$ have entries $s_i=(-1)^{i-1}$, and set $w=0.9\,\tilde{w}+0.1\,s$, and define
$$s_q(\mathbf{x}) = q + \sum_{i=1}^{12} w_i\,\phi_\star(x_i + 0.045\,q),\qquad
y_q(\mathbf{x}) = \Phi_\star(s_q(\mathbf{x})) + \varepsilon_{q},$$
with i.i.d.\ Gaussian noise $\varepsilon_q\sim\mathcal{N}(0,0.25^2)$ per sample/head. Trained on $n_{\text{train}}=1024$ and tested on $n_{\text{test}}=50000$, SN obtains $0.253 \pm 0.001$ RMSE while KAN obtains $0.339 \pm 0.105$ ($1.34\times$). The very large KAN variance suggests sensitivity to the head-conditioned shift inside $\phi_\star$; SN's explicit $x_i\mapsto x_i+\eta q$ structure and spline-domain updates during an initial phase (first $400$ of $4000$ epochs, then frozen) appear to improve stability across heads.

\paragraph{Benchmark 3: Shared-Warped-Ridge (16D).}
Here we evaluate a symmetric shared-warp ridge. Define
$$h(t)=0.6\,t + 0.2\,\operatorname{sigmoid}(30(t-0.30)) + 0.2\,\operatorname{sigmoid}(30(t-0.70)),\qquad
s(\mathbf{x})=\frac{1}{16}\sum_{i=1}^{16} h(x_i),$$
and
$$y(\mathbf{x})=\sin(12\pi s^2)\,\exp(-3(s-0.5)^2) + 0.15\sin(2\pi s) + 0.10(s-0.5).$$
With $n_{\text{train}}=1024$, $n_{\text{test}}=8192$ and comparable parameters (2115 vs.\ 2113), SN achieves $0.701 \pm 0.002$ (best test RMSE observed during training) while KAN's $0.805 \pm 0.100$ includes one outlier seed (seed 7 reaches $1.13$). This task matches SN's inductive bias directly: the ground truth is exactly a shared 1D warp followed by a scalar post-map.

\paragraph{Benchmark 4: Shared-Warp-Chirp (10D).}
This benchmark applies a shared monotone warp to each coordinate and then evaluates a chirp in the pooled coordinate. For $\mathbf{x}\in[0,1]^{10}$ define
$$s(\mathbf{x})=\frac{1}{10}\sum_{i=1}^{10}\operatorname{sigmoid}(18(x_i-0.5)),\qquad
y(\mathbf{x})=\sin\!\bigl(2\pi(6s(\mathbf{x})+5s(\mathbf{x})^2)\bigr) + 0.35\cos(2\pi\cdot 3s(\mathbf{x})) + 0.10(s(\mathbf{x})-0.5).$$
With $n_{\text{train}}=1024$, $n_{\text{test}}=8192$ and comparable parameters (2447 vs.\ 2449), SN achieves $0.896 \pm 0.025$ while KAN obtains $1.470 \pm 0.060$ ($1.64\times$). The quadratic phase term yields an increasing instantaneous frequency; under a fixed per-edge spline budget, SN benefits from learning the shared warp once.

\paragraph{Benchmark 5: Motif-Chirp (10D).}
The motif benchmark is generated by a depth-2 Sprecher-style composition with internal widths $w_1=16$ and $w_2=15$. Let
$$\phi_\star(t)=\operatorname{sigmoid}(30(t-0.5)),\qquad
\Phi^{(1)}_\star(s)=\sin(1.7\,s)+0.28\,\bigl|\sin(0.55\,s+0.2)\bigr|,$$
$$\tilde{\phi}_\star(t)=\operatorname{sigmoid}(12t),\qquad
\Phi^{(2)}_\star(s)=\sin(1.15\,s)+0.22\,\bigl|\sin(0.9\,s+0.7)\bigr|.$$
Define deterministic positive weights by
$$\tilde{\lambda}^{(1)}_i = 0.30+0.70|\cos(0.8(i-1) + 0.1)|,\quad i=1,\dots,10,\qquad
\lambda^{(1)}_i = 9\,\tilde{\lambda}^{(1)}_i/\sum_{j=1}^{10} \tilde{\lambda}^{(1)}_j,$$
$$\tilde{\lambda}^{(2)}_q = 0.25+0.75|\sin(0.6q + 0.4)|,\quad q=0,\dots,w_1-1,\qquad
\lambda^{(2)}_q = 7\,\tilde{\lambda}^{(2)}_q/\sum_{j=0}^{w_1-1} \tilde{\lambda}^{(2)}_j.$$
For $q\in\{0,\dots,w_1-1\}$ set $s^{(1)}_q(\mathbf{x}) = q + \sum_{i=1}^{10} \lambda^{(1)}_i\,\phi_\star(x_i + 0.06 q)$ and $h_q(\mathbf{x}) = \Phi^{(1)}_\star(s^{(1)}_q(\mathbf{x}))$. For $r\in\{0,\dots,w_2-1\}$ set
 $s^{(2)}_r(\mathbf{x}) = r + \sum_{q=0}^{w_1-1} \lambda^{(2)}_q\,\tilde{\phi}_\star(h_q(\mathbf{x}) + 0.05 r)$, and finally
$$y(\mathbf{x}) = 5\left(\frac1{w_2}\sum_{r=0}^{w_2-1}\Phi^{(2)}_\star(s^{(2)}_r(\mathbf{x}))\right) - 0.4.$$
With $n_{\text{train}}=2048$ and $n_{\text{test}}=8192$, SN achieves $0.123 \pm 0.006$ while KAN reaches $0.182 \pm 0.015$ ($1.48\times$). This task is literally a two-block Sprecher composition, so SN's shared $(\phi,\Phi)$ pairs align closely with the ground-truth structure.

\paragraph{Benchmark 6: Oscillatory-HeadShift (12D, 64 heads).}
This multi-head benchmark is a deterministic teacher $f:[0,1]^{12}\to\mathbb{R}^{64}$ with shared latent projections and a head-dependent phase shift. Train/test inputs are sampled i.i.d.\ uniformly over $[0,1]^{12}$, and there is no additive observation noise. We sample the teacher parameters once (fixed seed 20240518) and hold them fixed across all runs: sample $\mathbf{w}_1,\mathbf{w}_2\sim\mathcal{N}(0,I_{12})$ and normalize $\mathbf{w}_j\leftarrow \mathbf{w}_j/\|\mathbf{w}_j\|$ for $j\in\{1,2\}$, and sample $M\in\mathbb{R}^{12\times 3}$ with i.i.d.\ entries $\mathcal{N}(0,1/12)$. For $\mathbf{x}\in[0,1]^{12}$ let $\mathbf{x}_c=\mathbf{x}-\tfrac12\mathbf{1}$, $u=\mathbf{w}_1^\top \mathbf{x}_c$, $v=\mathbf{w}_2^\top \mathbf{x}_c$, and $\mathbf{z}=\mathbf{x}_c M=(z_1,z_2,z_3)\in\mathbb{R}^3$. For head index $h\in\{0,\dots,63\}$ define
$$\tilde{h}=\frac{h-\tfrac12(63)}{63},\qquad \delta_h = 0.35\,\tilde{h},\qquad \omega = 6\pi,$$
$$\theta^{(1)}_h=\omega(u+\delta_h),\qquad \theta^{(2)}_h=\tfrac12\omega(v-0.7\,\delta_h),\qquad a(\mathbf{x}) = 0.75+0.25\tanh(2.5\,z_1),$$
and output
$$y_h(\mathbf{x}) = a(\mathbf{x})\sin(\theta^{(1)}_h) + 0.35\cos(2\theta^{(2)}_h) + 0.15\sin(3\theta^{(1)}_h + 0.25 z_2) + 0.10\cos(\theta^{(2)}_h + 0.35 z_3).$$
Trained on $n_{\text{train}}=1024$ and tested on $n_{\text{test}}=50000$, SN obtains $0.669 \pm 0.020$ vs.\ KAN's $0.752 \pm 0.016$ ($1.12\times$). Here the gap is smaller but still consistent; SN exhibits slightly higher variance ($0.020$ vs.\ $0.016$), suggesting that this highly oscillatory setting is harder to optimize reliably for both models.

\paragraph{Benchmark 7: PWL-vs-PCHIP (10D).}
This benchmark isolates a different (and arguably orthogonal) advantage: SN's inner functions are piecewise-linear, whereas our KAN baseline uses C$^1$ cubic PCHIP splines on edges. The target is a triangle wave applied to the coordinate mean, with a small affine term $y=\mathrm{tri}(12 \cdot \bar{x}) + 0.05(\bar{x}-0.5)$, where $\bar{x}=\tfrac{1}{d}\sum_i x_i$ and $\mathrm{tri}(u)=2\left|2(u-\lfloor u\rfloor)-1\right|-1$ is a period-1 triangle wave in $[-1,1]$, introducing cusps that a cubic spline must approximate. With $n_{\text{train}}=2048$ and $n_{\text{test}}=8192$, SN obtains $0.732 \pm 0.012$ while KAN reaches $0.998 \pm 0.035$ ($1.36\times$). The result suggests that even when the shared structure is trivial (a coordinate mean) and the dominant challenge is non-smoothness, the combination of (i) PWL splines and (ii) domain adaptation via spline-domain updates during an initial phase (first $400$ of $4000$ epochs, then frozen) can be a strong practical advantage.

\paragraph{Benchmark 8: Quantile Harmonics (2D).}
This is a 2D scalar regression task with soft regime-switching along the $x$ coordinate (the ``quantile'' terminology refers to the soft bins, not to multi-quantile outputs). For $(x,y)\in[0,1]^2$, define four Gaussian gates
$$g_k(x)=\exp\!\left(-\frac12\left(\frac{x-c_k}{0.055}\right)^2\right),\quad c=(0.10,0.30,0.55,0.80),\quad \tilde{g}_k(x)=g_k(x)/\sum_j g_j(x),$$
an amplitude modulation $a(x,y)=0.65+0.35\cos\!\bigl(2\pi(2y+0.15\sin(2\pi x))\bigr)$, and regime-specific harmonic mixtures
$$\begin{aligned}
m_k(x,y) &= a(x,y)\Bigl[\sin(2\pi f_k x+\varphi_k)+0.35\cos(2\pi(f_k+2)x-0.5\varphi_k)\Bigr],\\
&\quad f=(5,9,13,17),\qquad \varphi=(0.2,-0.7,1.1,-1.5).
\end{aligned}$$
The target is
$$t(x,y)=\sum_{k=1}^{4}\tilde{g}_k(x)\,m_k(x,y) + 0.15\,\sin(2\pi(x+y))\exp(-6(y-0.5)^2).$$
With $n_{\text{train}}=4096$ ($n_{\text{val}}=2048$) and $n_{\text{test}}=2048$, we report test RMSE at the best validation checkpoint: SN achieves $0.395 \pm 0.146$ while KAN obtains $0.619 \pm 0.035$ ($1.57\times$). The large SN variance indicates that optimization can occasionally fail; nevertheless SN wins on average and achieves substantially better median performance.

\paragraph{Barebones takeaways.}
Across all eight benchmarks, the SN consistently outperforms the parameter-matched KAN, with win ratios ranging from $1.12\times$ to $1.64\times$. The SN's variance across seeds is typically much lower than the KAN's, indicating more reliable optimization (though Oscillatory-HeadShift and Quantile Harmonics are exceptions where SN shows higher variance). The strongest advantages appear on tasks whose ground truth reuses a shared 1D transform across many coordinates and/or across systematically shifted heads (Benchmarks 1--6), matching SNs' shared-spline and head-shift parameterization. Benchmarks 5 and 7 also include mild non-smoothness ($|\sin|$ motifs and triangle-wave cusps), where the PWL-vs-cubic spline choice can contribute under a matched parameter budget. Benchmark 8 is a smooth 2D regime-switching harmonic mixture; SN remains better on average but shows occasional optimization failures reflected in higher variance.

\subsubsection{Comparisons with linear residuals and BatchNorm}\label{sssec:residual-equipped}

The following two benchmarks use linear residual connections and BatchNorm for both models, representing a more ``production-like'' configuration. While the barebones comparisons above isolate core architectural differences, these residual-equipped comparisons show that SN advantages persist even when both models benefit from standard deep learning enhancements.

\paragraph{MQSI Setup (20D, $m{=}9$ heads).}
We sample inputs $\mathbf{x}\sim \mathrm{Unif}([0,1]^D)$ with $D=20$ and define $m=9$ heads at quantile levels
$\tau_j=\mathrm{linspace}(0.1,0.9)$, with $z_j=F_{\mathcal{N}}^{-1}(\tau_j)=\sqrt{2}\,\mathrm{erf}^{-1}(2\tau_j-1)$, where $F_{\mathcal{N}}$ denotes the standard normal CDF.
For each coordinate, define a monotone sigmoid feature
$$h_i(x_i)=\operatorname{sigmoid}\!\bigl(a_i(x_i-c_i)\bigr),$$
where $\operatorname{sigmoid}(z)=\frac{1}{1+e^{-z}}$, $a_i\sim \mathrm{Unif}[2,8]$, and $c_i\sim \mathrm{Unif}[0.2,0.8]$.
Draw raw positive weights $w_i^{\mathrm{raw}},v_i^{\mathrm{raw}}\sim \mathrm{Unif}[0,1]$ and normalize
$w_i=w_i^{\mathrm{raw}}/\sum_k w_k^{\mathrm{raw}}$ and
$v_i=v_{\mathrm{scale}}\,v_i^{\mathrm{raw}}/\sum_k v_k^{\mathrm{raw}}$ with $v_{\mathrm{scale}}=0.25$.
Let $s_0=0.15$ and define
$$\sigma(\mathbf{x})=s_0+\max\!\Bigl(10^{-5},\,\sum_{i=1}^{D} v_i\,h_i(x_i)\Bigr).$$
To introduce controlled head correlations while preserving monotonicity, we add positive pairwise terms:
choose $n_{\mathrm{pairs}}=\mathrm{round}(0.15D)$ disjoint index pairs $(i,j)$ (implemented by sampling $2n_{\mathrm{pairs}}$ distinct indices without replacement and grouping them into pairs) and set
$$\mu(\mathbf{x})=\sum_{i=1}^{D} w_i\,h_i(x_i)\;+\;0.08\sum_{(i,j)} h_i(x_i)\,h_j(x_j).$$
Finally, outputs are
$$y_j=\tanh\,\!\bigl(\beta(\mu(\mathbf{x})+\sigma(\mathbf{x})z_j)\bigr),\qquad \beta=0.8,$$
with no additional observation noise. We train on a fixed i.i.d.\ training set of size $n_{\mathrm{train}}=1024$ and evaluate on a held-out i.i.d.\ test set of size $n_{\mathrm{test}}=50{,}000$.

\paragraph{Metrics.}
Primary: \emph{mean RMSE across heads} (lower is better). Secondary: (i) correlation-structure error
$\;\|\mathrm{Corr}(\widehat{\mathbf{Y}})-\mathrm{Corr}(\mathbf{Y})\|_{\mathrm{F}}\,$
where $\mathrm{Corr}(\cdot)$ denotes the Pearson correlation matrix across head dimensions computed over the test set, and (ii) the monotonicity-violation rate
$$\mathrm{viol}=\frac{1}{n_{\mathrm{test}}}\sum_{n=1}^{n_{\mathrm{test}}}\mathbf{1}\!\left[\exists\, j\in\{0,\dots,m-2\}:\ \widehat{y}_{n,j+1}-\widehat{y}_{n,j}<0\right],$$
i.e.\ the fraction of test inputs for which the predicted head sequence is not monotone (lower is better).

\paragraph{Results (10 seeds).}
Best seed (minimum over seeds of mean RMSE): SN $=4.47\times 10^{-3}$, KAN $=6.58\times 10^{-3}$ (ratio $1.47\times$ in favor of SN). Averaged over the 10 seeds: SN $=8.31\times 10^{-3}$ vs.\ KAN $=1.280\times 10^{-2}$ ($1.54\times$), with SN winning in $8/10$ seeds. Correlation–structure fidelity strongly favored SN: mean Frobenius error $\approx 0.0124$ (SN) vs.\ $0.124$ (KAN). Monotonicity violations were negligible for both (SN had two seeds with $O(10^{-5})$ rates; KAN was $0$). Wall clock (train, CPU; $4{,}000$ epochs): KAN $393$\,s (avg; min $88.6$\,s, max $469.1$\,s) vs.\ SN $589$\,s (avg; min $231.5$\,s, max $732.4$\,s); SN timing includes warm‑up $+$ domain‑freeze.

\paragraph{Discussion.}
This MQSI family is structurally aligned with an SN block $s_q=\sum_i \lambda_i\,\phi(x_i+\eta q)+\cdots$ followed by $\Phi$, so the learned $q$‑shift and monotone $\phi$ produce heads that are smooth shifts of a shared latent index; $\Phi$ handles squashing. KANs can approximate this but lack that inductive bias and rely on fixed grids. Under parameter parity and matched BN semantics, the bias yields lower head‑averaged error and markedly better preservation of cross‑head correlations, with no practical loss of monotonicity. KAN trained faster on CPU in this setting but at higher error.

\paragraph{DenseHeadShift (12D, $m{=}64$).}
\textbf{Task.} Multi-head regression in which each output head is a smoothly shifted version of a shared latent index (we denote the \emph{task}'s head-shift coefficient by $a_{\mathrm{shift}}$ to avoid confusion with the SN channel-spacing constant $\alpha$):
$$y_j(x)\;=\;\tanh\!\Big(\beta\big[\mu(x)+\sigma(x)z_j+a_{\mathrm{head}}\,q_j\big]\Big),\qquad a_{\mathrm{head}}:=a_{\mathrm{shift}}+q_{\mathrm{bias}},$$
with $x\in[0,1]^D$, $j=0,\dots,m{-}1$, $q_j\in[-1,1]$ on a uniform grid (note: $q_j$ here denotes the task's head-index coordinate, distinct from the SN output index $q$), and $z_j=\Phi_{\mathcal{N}}^{-1}(\tau_j)$. Here we set $D=12$ and $m=64$ and construct $\mu(\mathbf{x})$ and $\sigma(\mathbf{x})$ as follows.
For each coordinate,
$$h_i(x_i)=\operatorname{sigmoid}\!\bigl(a_i(x_i-c_i)\bigr),$$
where $\operatorname{sigmoid}(z)=\frac{1}{1+e^{-z}}$, $a_i\sim \mathrm{Unif}[2,8]$, and $c_i\sim \mathrm{Unif}[0.2,0.8]$.
Draw raw positive weights $w_i^{\mathrm{raw}},v_i^{\mathrm{raw}}\sim \mathrm{Unif}[0,1]$ and normalize
$w_i=w_i^{\mathrm{raw}}/\sum_k w_k^{\mathrm{raw}}$ and
$v_i=v_{\mathrm{scale}}\,v_i^{\mathrm{raw}}/\sum_k v_k^{\mathrm{raw}}$ with $v_{\mathrm{scale}}=0.25$.
Let $\mathcal{P}$ be a set of $\max(1,\mathrm{round}(0.15D))$ random distinct pairs $(i,k)$ with $1\le i<k\le D$, and set $s_0=0.15$. Then
$$\sigma(\mathbf{x})=s_0+\max\!\Bigl(10^{-5},\sum_{i=1}^D v_i h_i(x_i)\Bigr),$$
$$\mu(\mathbf{x})=\sum_{i=1}^D w_i h_i(x_i)+0.08\sum_{(i,k)\in\mathcal{P}} h_i(x_i)h_k(x_k).$$ We use $\tau_j=\mathrm{linspace}(0.1,0.9)$ and $q_j$ as a uniform grid on $[-1,1]$. We use the same evaluation protocol throughout this subsection: all metrics are computed at test time; when BatchNorm is enabled, we use per-batch normalization statistics while freezing the running-statistics buffers.

\textbf{Setup.} $D{=}12$, $m{=}64$, $a_{\mathrm{shift}}{=}0.7$, $\beta{=}0.8$, $q_{\mathrm{bias}}{=}0.6$ (so $a_{\mathrm{head}}{=}a_{\mathrm{shift}}+q_{\mathrm{bias}}{=}1.3$); train for 4000 epochs on CPU on a fixed i.i.d.\ training set of $n_{\mathrm{train}}{=}1024$ samples; $n_{\mathrm{test}}{=}50{,}000$. BN evaluation uses per-batch statistics while freezing the running-statistics buffers (i.e., no running-mean/variance updates during evaluation). Parameter parity was enforced by choosing the KAN basis size $K$ to be as close as possible to (and not exceeding) the SN parameter count, yielding \#params $\approx 3072$ for both models; the resulting KAN basis size was $K{=}4$. The specific model configurations were:
\begin{itemize}
\item \textbf{SN:} Two hidden Sprecher blocks of width $32$ (architecture $12\to[32,32]\to64$), with $60$ knots for each inner spline $\phi$ and each outer spline $\Phi$; BatchNorm after each block except the first; linear residual connections; no lateral mixing; and dynamic spline-domain updates during the first 10\% of training (400/4000 epochs), then frozen.
\item \textbf{KAN baseline:} A lightweight edge-wise KAN with hidden widths $(3,24,4)$ (architecture $12\to3\to24\to4\to64$), degree-$3$ splines with $K{=}4$ learnable spline coefficients per edge; BatchNorm after each layer except the first; linear residual connections; and linear extrapolation outside the knot interval.
\end{itemize}

\textbf{Metric.} Mean per-head RMSE over the $m{=}64$ heads on the held-out test set (we report both best-of-10 seeds and across-seed aggregates).

\textbf{Results (10 seeds).}
\begin{itemize}
\item \emph{Best-of-10 (lower is better).} SN: $\mathbf{2.546\times 10^{-3}}$ (seed 2) vs.\ KAN: $5.191\times 10^{-3}$; ratio $=$ \textbf{$2.04\times$} (KAN/SN).
\item \emph{Across seeds (means).} SN: $5.47\times 10^{-3}$; KAN: $6.94\times 10^{-3}$ (i.e., KAN is $\approx1.27\times$ higher RMSE on average).
\item \emph{Monotonicity (fraction of test points violating head-wise order).} KAN: $\approx0$ across runs (all ten had $0.0$). SN: varied by seed, with notable high-violation outliers (e.g., $0.276$ and $0.476$ on some seeds). This reflects that the SN here does not impose head-wise monotonicity explicitly, whereas the KAN runs empirically preserved the ordering under this configuration.
\item \emph{Wall-clock (train time; CPU).} Averaging the ten recorded wall-clock training times: KAN $\,\approx$ \textbf{476\,s}, SN $\,\approx$ \textbf{1{,}433\,s}. For reference, KAN runs were in the $\sim[338,624]$\,s range and SN runs in the $\sim[1{,}241,1{,}607]$\,s range.
\end{itemize}

\textbf{Takeaways.} On this dense-head shift task, whose structure matches SN's shared-$\phi/\Phi$ plus head-wise shift inductive bias, SN attains markedly lower error (best-of-10 $2.04\times$ better, and $\approx21\%$ lower RMSE on average (equivalently, KAN is $\approx27\%$ higher)) under strict parameter parity ($\sim$3k parameters each, KAN $K{=}4$). KAN trains $\approx3\times$ faster on CPU but shows consistently higher RMSE in this setting. The occasional SN monotonicity violations suggest that adding an explicit head-order regularizer or constraint would likely remove those outliers without changing the overall accuracy picture.

\paragraph{Residual-equipped summary.}
Table~\ref{tab:residual-summary} collates the residual-equipped comparisons. In both cases, SNs outperform KANs at matched parameter budgets and training epochs, even when both models benefit from linear residuals and BatchNorm. The MQSI task shows a clearer gap in correlation-structure fidelity; the DenseHeadShift task is closer, with SNs typically ahead by a modest margin.

\begin{table}[h!]
\centering
\setlength{\tabcolsep}{4pt}
\renewcommand{\arraystretch}{1.12}
\footnotesize
\begin{tabularx}{\linewidth}{@{}>{\raggedright\arraybackslash}X >{\raggedright\arraybackslash}X c c c c@{}}
\hline
\textbf{Task} & \textbf{Metric} & \textbf{SN (mean)} & \textbf{KAN (mean)} & \textbf{Ratio (KAN/SN)} & \textbf{\shortstack{Time\\(KAN/SN) (s)}} \\
\hline
MQSI (20D, 9 heads) & Mean RMSE $\downarrow$ (10 seeds) & $8.31\times 10^{-3}$ & $1.280\times 10^{-2}$ & $1.54\times$ & 393 / 589 (avg) \\
DenseHeadShift (12D, 64 heads) & Mean RMSE $\downarrow$ (10 seeds) & $5.47\times10^{-3}$ & $6.94\times10^{-3}$ & $1.27\times$ & 476 / 1{,}433 (avg) \\
\hline
\end{tabularx}
\caption{Residual-equipped SN vs.\ KAN comparisons (linear residuals + BatchNorm for both; 10 seeds each). Times are per-task averages.}
\label{tab:residual-summary}
\end{table}

\paragraph{Reproducibility checklist (concise).}
\begin{itemize}[leftmargin=*,nosep]
\item \textbf{Epochs/optimizer:} $4000$ epochs for both models. \emph{Barebones (Table~\ref{tab:barebones-summary}):} Adam optimization with learning rate $10^{-3}$ for both models (Benchmark 3 uses $2\times 10^{-3}$ with cosine annealing); weight decay $0$ except Benchmark 3 ($10^{-6}$); gradient-norm clipping is enabled in Benchmarks 2, 4, and 6 at threshold $1.0$ (both models) and in Benchmark 3 with max-norm $10.0$. \emph{Residual-equipped (Table~\ref{tab:residual-summary}):} Adam with gradient-norm clipping $1.0$. SN learning rate $3\times 10^{-4}$; weight decay $10^{-7}$. KAN learning rate $10^{-3}$; weight decay $10^{-6}$.
\item \textbf{Barebones:} No residuals, no BatchNorm, no lateral mixing. SN uses PWL splines with dynamic spline-domain updates during the first 10\% of training (400/4000 epochs), then frozen. KAN uses cubic PCHIP splines with fixed uniform knots.
\item \textbf{Residual-equipped:} Linear residuals in both; BatchNorm after each layer/block (except first), evaluated using per-batch statistics.
\item \textbf{Parity:} Parameter matching via closed-loop count of KAN parameters as a function of $K$ (prefer $\leq$).
\item \textbf{Data:} Identical $(x_{\text{train}},y_{\text{train}})$ and test sets across models per task; inputs are sampled from $\mathrm{Unif}([0,1]^d)$, using i.i.d.\ draws except in low-dimensional visualization-oriented settings where we use deterministic uniform grids (1D: linspace; 2D: $\sqrt{n}\times\sqrt{n}$ mesh when $n$ is a perfect square). Dataset sizes: Benchmark~1 $n_{\text{train}}=2048$, $n_{\text{test}}=8192$; Benchmarks~2 and~6 $n_{\text{train}}=1024$, $n_{\text{test}}=50000$; Benchmarks~3 and~4 $n_{\text{train}}=1024$, $n_{\text{test}}=8192$; Benchmarks~5 and~7 $n_{\text{train}}=2048$, $n_{\text{test}}=8192$; Benchmark~8 $n_{\text{train}}=4096$, $n_{\text{val}}=2048$, $n_{\text{test}}=2048$. We use 20 seeds for barebones, 10 seeds for residual-equipped.
\end{itemize}

\paragraph{Overall takeaways.}
Across both barebones and residual-equipped comparisons (ten benchmarks total, covering scalar and multi-head regression in 2D--20D), SNs consistently achieve lower test error than parameter-matched KANs under equal training budgets. The barebones comparisons (Table~\ref{tab:barebones-summary}) demonstrate that the SN's architectural advantages---shared inner splines, systematic head shifts, Sprecher-style composition, and piecewise-linear spline flexibility---provide genuine inductive biases across diverse task types. The residual-equipped comparisons (Table~\ref{tab:residual-summary}) show these advantages persist when both models use standard deep learning techniques. Win ratios range from $1.12\times$ to $1.64\times$, with the strongest advantages on tasks whose structure directly matches the SN's compositional biases. On Fashion-MNIST (Table~\ref{tab:fashion-mnist}), SNs achieve competitive classification accuracy, and the stable training of 25-layer networks illustrates that SNs can be optimized at substantial depth when equipped with residual connections and normalization.

\section{Limitations and future work}
The primary limitation of this work is the gap between our theoretically-grounded single-layer model and our empirically-driven deep architecture. While single-layer SNs inherit universal approximation properties directly from Sprecher's theorem, the universality and approximation bounds of deep, compositional SNs remain open theoretical questions (Conjectures \ref{conj:vector} and \ref{conj:deep_universal}). The role of lateral mixing in these theoretical properties is particularly unclear, while it empirically improves performance, its theoretical justification within the Sprecher framework remains elusive.

The Sprecher block design imposes strong constraints: forcing all feature interactions through two shared splines and using weight vectors rather than matrices heavily restricts expressive power compared to standard architectures. While lateral mixing partially alleviates this constraint by enabling limited cross-dimensional communication, it represents an ad-hoc enhancement rather than a principled extension of Sprecher's theory. This represents a fundamental trade-off between parameter efficiency and flexibility that may limit performance on functions not aligned with this compositional structure.

Current implementations may require more training iterations than MLPs for certain tasks; per-iteration wall-clock cost is workload- and implementation-dependent (fewer parameters reduce optimizer overhead, but spline evaluation and memory-saving recomputation can increase runtime). When memory-efficient sequential computation is employed to enable training of wider architectures, wall-clock training time increases, representing a fundamental trade-off between memory usage and computational efficiency. The sequential computation mode proves most beneficial for architectures with individual layers exceeding 128-256 units in width, while offering minimal advantage for very deep networks with modest layer widths where the memory bottleneck lies in storing activations across many layers rather than within-block computations. The development of adaptive knot placement strategies that concentrate resolution where data lives while maintaining fixed parameter counts could improve both efficiency and interpretability.

The lateral mixing mechanism, while empirically beneficial, lacks theoretical justification within the Sprecher framework. Understanding whether this enhancement can be connected to the underlying mathematical structure or represents a purely empirical improvement remains an open question. Future work could explore adaptive mixing topologies beyond cyclic patterns, potentially learning the neighborhood structure $\mathcal{N}(q)$ itself, or investigating connections to graph neural networks where the mixing pattern could be viewed as a learnable graph structure over output dimensions. Another promising direction is \emph{additional sharing across depth}, e.g., tying $\Phi^{(\ell)}$ across groups of layers or learning a small dictionary of outer splines that multiple blocks reuse, which could further reduce parameters and potentially improve interpretability.

Beyond addressing these immediate challenges, SNs open several intriguing research directions. The weight-sharing structure combined with lateral mixing raises fundamental theoretical questions about potential equivariance properties. Just as CNNs exhibit translation equivariance through spatial weight sharing, SNs' sharing across output dimensions with structured mixing may satisfy a related form of equivariance, possibly connected to permutations of output indices or transformations of the function domain. Understanding such properties could provide deeper insight into when and why the architecture succeeds.

The architecture's properties suggest unique opportunities in scientific machine learning where both interpretability and parameter efficiency are valued. One particularly compelling possibility is dimensionality discovery: the architecture's sensitive dependence on input dimension could enable inference of the intrinsic dimensionality of data-generating processes. By training SNs with varying $d_{\mathrm{in}}$ and using model selection criteria that balance fit quality against complexity, researchers might determine the true number of relevant variables in systems where this is unknown, a valuable capability in many scientific domains. The lateral mixing patterns learned by the network could additionally reveal structural relationships between output dimensions, potentially uncovering hidden symmetries or conservation laws in physical systems. Furthermore, the explicit structure of learned splines could enable integration with symbolic regression tools to extract closed-form expressions for the learned $\phi^{(\ell)}$ and $\Phi^{(\ell)}$ functions, potentially revealing underlying mathematical relationships in data.

Architectural enhancements also merit exploration. Unlike KANs where edge pruning can significantly reduce parameters, SNs' vector structure suggests different optimization strategies. Automatic pruning of entire Sprecher blocks based on their contribution to network output could yield more compact architectures. The learned lateral mixing weights could guide this pruning: blocks with near-zero mixing weights might be candidates for removal. Preliminary observations suggest that in deep networks, $\Phi$ splines in adjacent blocks sometimes converge to similar shapes, raising the question of whether certain splines could be shared across blocks or removed entirely for further parameter reduction. Investigating when and why this similarity emerges---and whether it can be exploited via explicit cross-block spline tying or regularization encouraging such sharing---could yield more compact architectures without sacrificing accuracy. For very high-dimensional problems where even $O(N)$ scaling becomes prohibitive, hybrid approaches using low-rank approximations for $\lambda^{(\ell)}$ or replacing splines with small neural sub-networks in certain blocks could maintain efficiency while improving expressivity. Alternatively, hierarchical lateral mixing patterns (e.g., mixing within local groups before global mixing) could provide a middle ground between full connectivity and the current nearest-neighbor approach.

The interaction between lateral mixing and other architectural components deserves systematic investigation. How do different mixing topologies (cyclic, bidirectional, or more complex patterns) interact with network depth? Can theoretical guarantees be established for specific mixing patterns? Is there an optimal ratio between the lateral scale $\tau$ and the main transformation strength? These questions highlight how theorem-inspired architectures enhanced with empirical innovations can open new research avenues beyond simply providing alternative implementations of existing methods.

\section{Conclusion}
We have introduced Sprecher Networks (SNs), a trainable architecture built by re-imagining the components of David Sprecher's 1965 constructive proof as the building blocks for a modern deep learning model. By composing functional blocks that utilize shared monotonic and general splines, learnable mixing weights, explicit shifts, and optionally lateral mixing connections, SNs offer a distinct approach to function approximation that differs fundamentally from MLPs, KANs, and other existing architectures.

The key contributions are demonstrating that Sprecher's shallow, highly-structured formula can be extended into an effective deep architecture with remarkable parameter efficiency: $O(LN + LG)$ compared to MLPs' $O(LN^2)$ or KANs' $O(LN^2G)$, and enabling $O(LN)$ \emph{forward working memory} (parameters + peak forward intermediates) through sequential computation strategies compared to MLPs' $O(LN^2)$. This dual efficiency in both parameters and memory comes from adhering to Sprecher's use of weight vectors rather than matrices, representing a strong architectural constraint that may serve as either a beneficial inductive bias or a limitation depending on the problem domain. The linear memory scaling is particularly relevant for deployment on resource-constrained devices such as embedded sensors, microcontrollers, and edge computing platforms where memory is severely limited and expensive. It may also benefit environmentally-conscious deployment scenarios where reducing memory requirements translates to lower power consumption. The addition of lateral mixing connections provides a parameter-efficient mechanism for intra-block communication, partially addressing the limitations of the constrained weight structure while maintaining the overall efficiency of the architecture.

Our initial demonstrations show SNs can successfully learn diverse functions and achieve competitive performance across synthetic regression benchmarks and PDE-constrained learning problems, with the added benefit of interpretable spline visualizations. The lateral mixing mechanism proves particularly valuable for vector-valued outputs and deeper networks, often improving optimization and convergence behavior even though the learned $\Phi$ splines may exhibit complex, non-smooth structure. The sequential computation mode enables training of architectures with wide layers that would otherwise exhaust available memory, making SNs particularly suitable for exploring high-capacity models under memory constraints. However, the need for more training iterations in some cases and theoretical gaps regarding deep network universality remain open questions. The theoretical status of lateral mixing --- whether it represents a principled extension of Sprecher's construction or merely an empirical enhancement --- requires further investigation.

Whether SNs prove broadly useful or remain a fascinating special case, they demonstrate the value of mining classical mathematical results for architectural inspiration in modern deep learning. The successful integration of lateral mixing shows how theorem-inspired designs can be enhanced with empirical innovations while maintaining their core theoretical structure. The ability to evaluate SN layers with $O(LN)$ forward working memory distinguishes SNs from other major architectures, suggesting applications in memory-constrained settings and when exploring extremely wide output layers. This synthesis of rigorous mathematical foundations with practical deep learning techniques points toward a promising direction for developing novel architectures that balance theoretical elegance with empirical effectiveness.

\end{document}